\pdfoutput=1
\newif\iflipics
\lipicsfalse
\iflipics
  \csname @@input\endcsname lipics/lipics-setup.tex
\else
  \documentclass[11pt]{article}
  \usepackage[english]{babel}
  \usepackage[letterpaper,top=2.2cm,bottom=2.2cm,left=2.7cm,right=2.7cm]{geometry}
\fi
\usepackage[T1]{fontenc}
\usepackage{lmodern,microtype}
\usepackage{amsmath,amssymb,amsthm,amsfonts,mathtools}
\usepackage{mleftright}\mleftright
\usepackage{bbm}
\usepackage{booktabs}
\usepackage{array}
\usepackage{graphicx}
\usepackage{xcolor}
\iflipics
  \lipicsbibsetup
\else
\usepackage[backend=bibtex,style=alphabetic,natbib=true,sortcites=true,maxcitenames=2,maxbibnames=99,maxalphanames=6,minalphanames=3,isbn=false]{biblatex}
\DeclareFieldFormat[misc]{title}{\iffieldundef{eprint}{\mkbibemph{#1}}{\mkbibquote{#1\isdot}}} % arXiv @misc titles in quotes, like @unpublished
\AtBeginBibliography{\sloppy}
\makeatletter
\def\blx@foreignfile#1#2{\blx@warning@noline{#1 (skipped; rerun BibTeX)}}
\patchcmd{\blx@ifsigned}{\blx@error}{\blx@foreignfile}{}{\PackageWarningNoLine{main}{Could not patch biblatex's .bbl check}}
\makeatother
\fi
\usepackage{aliascnt}
\iflipics\else % the LIPIcs class loads hyperref and cleveref, and defines the theorem environments
\usepackage{hyperref}
\usepackage[capitalize,noabbrev]{cleveref}
\fi
\allowdisplaybreaks
\iflipics\else
\theoremstyle{plain}
\newtheorem{theorem}{Theorem}[section]
\newaliascnt{proposition}{theorem}

\aliascntresetthe{proposition}
\newaliascnt{lemma}{theorem}
\newtheorem{lemma}[lemma]{Lemma}
\aliascntresetthe{lemma}
\newaliascnt{corollary}{theorem}
\newtheorem{corollary}[corollary]{Corollary}
\aliascntresetthe{corollary}
\newaliascnt{claim}{theorem}
\newtheorem{claim}[claim]{Claim}
\aliascntresetthe{claim}
\theoremstyle{definition}
\newaliascnt{definition}{theorem}
\newtheorem{definition}[definition]{Definition}
\aliascntresetthe{definition}
\theoremstyle{remark}
\newaliascnt{remark}{theorem}

\aliascntresetthe{remark}
\fi
\crefname{proposition}{Proposition}{Propositions}
\crefname{lemma}{Lemma}{Lemmas}
\crefname{corollary}{Corollary}{Corollaries}
\crefname{claim}{Claim}{Claims}
\crefname{definition}{Definition}{Definitions}
\crefname{remark}{Remark}{Remarks}
\makeatletter
\newenvironment{proofof}[1]{\par\pushQED{\qed}\normalfont\topsep6\p@\@plus6\p@\relax
  \trivlist\item[\hskip\labelsep]\textbf{#1.}\ \ignorespaces}{\popQED\endtrivlist\@endpefalse}
\makeatother

\crefformat{equation}{(#2#1#3)}
\Crefformat{equation}{Equation~(#2#1#3)}
\crefrangeformat{equation}{(#3#1#4) to~(#5#2#6)}
\crefmultiformat{equation}{(#2#1#3)}{ and~(#2#1#3)}{, (#2#1#3)}{, and~(#2#1#3)}
\crefname{appendix}{Appendix}{Appendices}
\Crefname{appendix}{Appendix}{Appendices}

\newcommand{\reals}{\mathbb{R}}
\newcommand{\naturals}{\mathbb{N}}
\iflipics % unnumbered run-in headers: LIPIcs numbers \paragraph (secnumdepth 4) and sets it as a display heading
\newcommand{\para}[1]{\subparagraph*{#1}}
\else
\newcommand{\para}[1]{\paragraph*{#1}}
\fi
\newcommand{\eps}{\varepsilon}
\newcommand{\out}{y}
\newcommand{\outDom}{\mathcal{Y}}
\newcommand{\alg}{M}
\newcommand{\domain}{\mathcal{X}}
\newcommand{\dataset}{\boldsymbol{s}}
\newcommand{\prob}[2]{\underset{#1}{\mathbb{P}} \left(#2 \right)}
\newcommand{\expect}[2]{\underset{#1}{\mathbb{E}} \left[#2 \right]}
\newcommand{\loss}{\ell}
\newcommand{\lossFunc}[3]{\loss\left(#1; #2, #3 \right)}
\newcommand{\lossRV}{L}
\newcommand{\lossRVfunc}[2]{\lossRV_{#1, #2}}
\newcommand{\HockeyStick}[1]{\boldsymbol{H}_{#1}}
\newcommand{\HockeyStickDiv}[3]{\HockeyStick{#1}\left(#2 ~\left\Vert~ #3 \right. \right)}
\newcommand{\Var}{\operatorname{Var}}
\newcommand{\chisq}[2]{\chi^{2}_{#1,#2}}
\newcommand{\chiVar}{\tilde V} % chi-square divergence of the one-step pair (P,Q)
\newcommand{\Phibar}{\overline{\Phi}}
\newcommand{\ind}{\mathbbm{1}}
\newcommand{\Pbar}[1]{\bar{P}_{#1}}
\newcommand{\Nrm}{\mathcal{N}}
\newcommand{\rangeBnd}{{\Delta}} % one-sided for Bernstein's inequality, two-sided for the Gaussian coupling
\newcommand{\gridSpacing}{h} % spacing of a connect-the-dots grid
\newcommand{\truncSet}{E} % set on which a privacy loss is bounded (truncation set or conditioning event)
\newcommand{\lossShift}{\alpha} % loss shift, also the shift of shifted GDP
\newcommand{\tailShift}{\beta} % tail shift, i.e. additive error in delta
\newcommand{\lossMean}[2]{\mu_{#1,#2}} % mean of the privacy loss of the pair (#1,#2)
\newcommand{\lossVar}[2]{V_{#1,#2}} % variance of the privacy loss of the pair (#1,#2)
\newcommand{\truncMean}[2]{\mu'_{#1,#2}} % mean of the truncated privacy loss of (#1,#2)
\newcommand{\truncVar}[2]{V'_{#1,#2}} % variance of the truncated privacy loss of (#1,#2)
\newcommand{\subsampVar}{W} % variance parameter of a subsampled composition
\newcommand{\ratio}{r}
\newcommand{\ratioFunc}[3]{\ratio\left(#1; #2, #3 \right)} % likelihood ratio #2(#1)/#3(#1)
\newcommand{\ratioRV}{R}
\newcommand{\ratioRVfunc}[2]{\ratioRV_{#1, #2}} % likelihood ratio random variable, r(omega; #1, #2) for omega ~ #1
\newcommand{\renyiOrder}{q} % order of a R\'enyi divergence
\newcommand{\tcdpOrder}{\bar q} % order bound of tCDP
\newif\ifshownotes
\shownotesfalse
\ifshownotes
\providecommand{\vf}[1]{\textcolor{red}{[VF: #1]}}
\providecommand{\ms}[1]{\textcolor{blue}{[MS: #1]}}
\else
\providecommand{\vf}[1]{}
\providecommand{\ms}[1]{}
\fi

\title{Gaussian Approximation for Composition and Subsampling of Differentially Private Algorithms}
\iflipics
  \lipicsmeta
\else
\author{
    Vitaly Feldman\\
    Apple
    \and
    Moshe Shenfeld\footnote{Work partially done while author was an intern at Apple.}\\
    The Hebrew University of Jerusalem
}
\date{}
\fi

\begin{document}
\maketitle

\begin{abstract}

The privacy analysis of differentially private stochastic gradient descent (DP-SGD) relies on bounds for the composition of subsampled mechanisms. The privacy of such a composition is bounded by that of a pair of product distributions, whose privacy loss is a sum of independent terms and has been analyzed via a variety of notions and concentration results. Yet, existing closed-form bounds either lose a logarithmic factor or require restrictive conditions on the noise scale and delicate analysis of subsampling. Further, while it is well known that sums of independent random variables are approximately Gaussian, existing CLT-based techniques give an approximation with a large error in $\delta$.

We show that the privacy loss of any composition is close to a Gaussian random variable with the same mean and variance as long as each step has a small privacy loss except with small probability. Specifically, after the rare large losses are truncated, the privacy profile of the composition is bounded from above and from below by that of the Gaussian, up to an additive error $\tailShift$ in $\delta$ and a shift of $\eps$ that grows only as $\ln(1/\tailShift)$. This shift does not depend on the number of steps, and thus in typical applications the composition satisfies Gaussian differential privacy up to lower-order shifts. For $t$ steps of a mechanism with Poisson subsampling at rate $\lambda$, the variance of the Gaussian is essentially $t\lambda^2$ times the $\chi^2$ divergence of a dominating pair of the mechanism. This gives a new and particularly simple analysis of subsampled composition. We extend the approximation and this analysis to random allocation (balls-and-bins sampling). For the Gaussian mechanism with noise scale $\sigma$ and either sampling method at rate $1/t$, we obtain the rate predicted by the central limit theorem for $\sigma\gtrsim\max\left\{1,\sqrt{\ln(1/\delta)}/\ln(t)\right\}$, whereas the standard R\'enyi-divergence analyses of Poisson subsampling need $\sigma\gtrsim\sqrt{t\ln(1/\delta)}/\ln(t)$. As another application, we show that for bounded privacy losses the approximation explains the empirically observed high accuracy of the connect-the-dots discretization of privacy loss distributions.

\end{abstract}

\section{Introduction}\label{sec:intro}
The privacy analysis of multi-step differentially private (DP) algorithms relies on composition of privacy losses of individual steps. For efficiency purposes each step is often used on a random subsample of the data with the analysis relying on privacy amplification by subsampling. In differentially private stochastic gradient descent (DP-SGD) \citep{SCS13,BST14,ACGMMT16}, for example, each step releases the sum of the clipped gradients of a batch with Gaussian noise of scale $\sigma$. It is typically analyzed with \emph{Poisson subsampling} in which each element joins each batch independently with probability $\lambda$. The privacy of $t$ steps is then upper bounded by that of a pair of $t$-fold product distributions \citep{DRS22, ZDW22} and the privacy loss of such a pair is a sum of $t$ independent random variables. Such sums have been analyzed through concentration inequalities \citep{DRV10,CR21} and through several privacy notions, such as R\'enyi differential privacy \citep{Mironov17}, concentrated differential privacy \citep{DR16} and its variants \citep{BS16,BDRS18}, and Gaussian differential privacy (GDP) \citep{DRS22}. 

For $\sigma\ge1$ the variance of this sum predicts $(\eps,\delta)$-DP with $\eps$ of order $\lambda\sqrt{t\ln(1/\delta)}/\sigma$, which we call the \emph{central limit rate}. Several closed-form bounds at this rate are known and used in theoretical analyses of private learning \citep{BFTT19,BFGT20,Steinke25}.
However, the known closed-form bounds give the central limit rate only in restricted regimes, and the black-box ones lose additional factors. If each step is shown to be $(\eps_0,\delta_0)$-DP, amplified by subsampling and then combined by advanced composition \citep{KLNRS08,DRV10}, the resulting bound exceeds the central limit rate by a factor $\sqrt{\ln(1/\delta)}$ at $\lambda=1/t$. The moments accountant \citep{ACGMMT16} and its refinement for the sampled Gaussian mechanism \citep{MTZ19}, both based on R\'enyi differential privacy \citep{Mironov17}, remove this factor through a per-step bound $O(\lambda^2\renyiOrder/\sigma^2)$ on the R\'enyi divergence of order $\renyiOrder$. But this bound holds only for orders up to about $\sigma^2\ln(1/\lambda)$ \citep[Lemma~3]{ACGMMT16}, \citep{MTZ19}. Since the conversion to $(\eps,\delta)$-DP at the central limit rate uses orders of about $\ln(1/\delta)/\eps$, for $\lambda=1/t$ the two requirements are compatible only if $\sigma\gtrsim\sqrt{t\ln(1/\delta)}/\ln(t)$. Practical applications typically rely on $\sigma$ being a small constant, and thus these closed-form bounds do not give the central limit rate for such $\sigma$.
Further, central limit theorems for the privacy loss \citep{SMM19,DRS22} predict that the composition is approximately Gaussian differentially private. Yet, their quantitative (Berry--Esseen) form \citep[Theorem~3.4]{DRS22} has an additive error in $\delta$, of order $1/\sqrt t$ at $\lambda=1/t$. Thus it implies the central limit rate only for relatively large $\delta$.

This leads to the questions of whether the Gaussian approximation predicted by the central limit theorems holds beyond the large error in $\delta$ and whether the closed-form bounds at the central limit rate hold under mild and practically realistic conditions on $\sigma$ and $\delta$.

\subsection{Our contributions}\label{sec:contrib}
We give new and relatively simple bounds on the concentration of the privacy loss which imply that the answer to both questions is yes under mild conditions on the privacy loss of each step. If the privacy loss of each step is small except with small probability then the privacy loss of the product pair (in either direction) is close to a Gaussian random variable with the same mean and variance, after truncating the rare large privacy losses. Specifically, if the privacy loss of every step deviates from its mean by at most $\rangeBnd$, then for every $\tailShift$ the cumulative distribution function of the privacy loss of the product pair is bounded from above and from below by that of the Gaussian random variable with the same mean $\mu$ and variance $V$, up to a \emph{loss shift} of $O(\rangeBnd\ln(1/\tailShift))$ in its argument and a \emph{tail shift} $\tailShift$ in its value (\cref{thm:gauss-approx}). We say that such a privacy loss is \emph{Gaussian-concentrated}. For a pair with a Gaussian-concentrated privacy loss, the privacy profile is within the same shifts (in $\eps$ and in $\delta$) of the privacy profile of $\sqrt V$-GDP (\cref{def:shifted-gdp}) shifted in $\eps$ by the excess $\mu-V/2$ of the mean over half the variance (\cref{lem:gc-profile}). Importantly, the loss shift does not depend on the number of steps $t$ and depends only logarithmically on the tail shift. Therefore, in contrast to the Berry--Esseen bounds, the approximation covers values of $\delta$ below every fixed power of $1/t$ (as in the second part of \cref{cor:gauss-leading}). For privacy losses that are bounded only with high probability we truncate the privacy loss of each step. The truncation uses part of the tail shift $\tailShift$ for the probability of the truncated events and replaces the mean and the variance by those of the truncated privacy loss.
Shifted GDP is preserved by post-processing and therefore transfers from the product of the dominating pairs of the steps to their adaptive composition (although the resulting upper bounds need not be tight).
In \cref{app:notions} we compare shifted GDP with zero-concentrated differential privacy \citep{BS16}, truncated concentrated differential privacy \citep{BDRS18} and concentrated differential privacy \citep{DR16}.

\para{From Bernstein's inequality to a Gaussian coupling.}
Our results rely on the fact that the privacy profile of a composition of many steps is determined by the mean and the variance of its privacy loss, up to an error that depends on the size of the privacy loss of a single step but not on the number of steps $t$. As a warm-up, in \cref{sec:bernstein} we show that the standard Bernstein concentration inequality can be used to derive an upper bound on the privacy loss probability tail (up to a second-order factor). This bounds the privacy profile by $\delta$ at $\eps=\mu+\sqrt{2V\ln(1/\delta)}+O(\rangeBnd\ln(1/\delta))$, where $\mu$ is the mean of the total privacy loss and $\rangeBnd$ bounds the deviation of each privacy loss above its mean (\cref{thm:bernstein-bounded}). The first two terms depend only on the mean and the variance, and the last term, which depends on $\rangeBnd$, does not grow with $t$. This improves over advanced composition in the same way that Bernstein's inequality improves over Hoeffding's inequality, which underlies advanced composition \citep{DRV10}. After truncation, this is essentially all we need for upper bounds at the central limit rate (\cref{thm:subsampled}). A tail inequality cannot, however, lower bound a privacy profile. To obtain the two-sided Gaussian concentration, we appeal to Gaussian coupling results for sums of independent random variables by \citet{Sakhanenko96} (\cref{thm:sakhanenko}). These results couple the sum with a Gaussian random variable of the same mean and variance, with an error that depends on $\rangeBnd$ but not on $t$ (\cref{lem:coupling}).

\para{Composition and subsampling through the $\chi^2$ divergence.}
In \cref{sec:chi2-subsampling} we show a way to simplify the resulting concentration bounds by using a tail condition and $\chiVar\coloneqq\chisq{P}{Q}$, the $\chi^2$ divergence of the pair $(P,Q)$ that dominates one step. This result is based on the observation that the $\chi^2$ divergence bounds the mean and the variance of the privacy loss (\cref{lem:chi2-moments}). In addition, subsampling at rate $\lambda$ multiplies the $\chi^2$ divergence by exactly $\lambda^2$ (in the add direction, by at most $\lambda^2/(1-\lambda)$).
Together with the basic Bernstein concentration inequality, this yields $\eps=O\left(t\lambda^2\chiVar+\lambda\sqrt{t\cdot\chiVar\ln(1/\delta)}\right)$ for $t$ subsampled steps whenever large single-step privacy losses are sufficiently rare (\cref{thm:subsampled}). Using the Gaussian concentration, we also show that $t$ subsampled steps are shifted $\sqrt{\subsampVar}$-GDP for $\subsampVar=t\lambda^2\chiVar/(1-\lambda)$ (\cref{cor:subsampled-approx}). Further, under mild moment and tail conditions, $\subsampVar$ is nearly tight: for small $\lambda$, the truncated privacy loss in both directions has variance close to $\subsampVar$ and mean close to $\subsampVar/2$ up to a factor $(1\pm O(\lambda))$ (\cref{cor:subsampled-lower}). This makes $\chiVar$ a particularly useful single proxy for characterizing the privacy profile of (subsampled) composition.

\para{The subsampled Gaussian mechanism.}
In \cref{sec:gauss-two-sided} we apply these results to DP-SGD. Each step is dominated by the Gaussian pair, whose $\chi^2$ divergence is $\chiVar=e^{1/\sigma^2}-1$ and whose privacy losses are Gaussian (\cref{clm:gauss-pair}). Thus DP-SGD is shifted $\sqrt{\subsampVar}$-GDP (\cref{cor:gauss-poisson}). Converting this guarantee to $(\eps,\delta)$-DP gives explicit and mild conditions \cref{eq:gauss-conditions} on $\sigma$, $t$ and $\delta$ under which $\eps=O\left(t\lambda^2/\sigma^2+\lambda\sqrt{t\ln(1/\delta)}/\sigma\right)$, the central limit rate at $\lambda=1/t$. For example, for all sufficiently large $t$ the conditions hold when $\sigma\gtrsim\max\left\{1,\sqrt{\ln(1/\delta)}/\ln(t)\right\}$ and $t>\ln^{3}(1/\delta)$, and for a fixed $\sigma\ge1$ they hold for all $\delta\ge\exp(-\Omega(\sigma^2\ln^2(t)))$. For comparison, at $\lambda=1/t$ these conditions allow $\sigma$ smaller by a factor of at least order $\sqrt t/\ln(t)$ than the closed-form analyses based on R\'enyi divergence need. Indeed, these analyses bound exponential moments of the privacy loss. Those of low order scale as $\lambda^2$, but the conversion to $(\eps,\delta)$-DP for small $\delta$ uses high orders, at which the exponential moments are dominated by rare large privacy losses and scale as $\lambda^2$ only for large $\sigma$. We need only the first two moments and a tail bound (see \cref{sec:comparison} for a more detailed comparison).

\para{Random allocation.}
In $1$-out-of-$t$ \emph{random allocation} (also known as balls-and-bins sampling), every element is used in exactly one step, chosen independently and uniformly at random \citep{CGHLKKMSZ25,CCGHST25}. The tightest asymptotic bounds on its privacy loss are obtained through a reduction to Poisson subsampling \citep{FS25ra}. The resulting bound for the Gaussian mechanism therefore inherits the assumption on $\sigma$ of the analysis of Poisson subsampling, and the reduction needs the additional assumption $\sigma\gtrsim\sqrt{\ln(t/\delta)}$.

Our Gaussian approximation in terms of $\chiVar=\chisq{P}{Q}$ extends to random allocation even more directly than to Poisson subsampling. The likelihood ratio of the dominating pair of random allocation \citep{FS25ra} is an average of $t$ independent one-step likelihood ratios, each with mean $1$ and variance $\chiVar$. In \cref{sec:alloc} we couple this average (as opposed to its logarithm) with a Gaussian random variable with the same mean and variance. We show that if each one-step likelihood ratio is within $\rangeBnd$ of $1$ and $(\rangeBnd/t)\ln(1/\tailShift)$ is sufficiently small, then the privacy losses of the pair are Gaussian-concentrated with mean $\chiVar/(2t)$ and variance $\chiVar/t$ (as for $\sqrt{\chiVar/t}$-GDP). The tail shift is $\tailShift$ and the loss shift is of order $(\rangeBnd/t)\ln(1/\tailShift)$ (\cref{thm:ra-gauss}). For the Gaussian mechanism, random allocation is shifted $\sqrt{\chiVar/t}$-GDP, with a parameter slightly smaller than that of Poisson subsampling at rate $1/t$. Thus we obtain the central limit rate under a condition of the same form as for Poisson subsampling, with no additional assumption. We also obtain the same asymptotic upper bounds (\cref{cor:gauss-ra,cor:gauss-ra-closed-form,cor:gauss-leading-ra}). 

\para{Connect-the-dots discretization.}
As an additional application of our techniques we consider the connect-the-dots (CtD) discretization \citep{DGKKM22}, which Google's \texttt{dp\_accounting} library \citep{GoogleDPAccounting} uses by default for the privacy loss distributions (PLDs) of continuous mechanisms. Its authors found that it ``can work with much larger discretization intervals while keeping a similar error bound compared to previous approaches'' such as privacy buckets \citep{MM18}, and asked ``where these improvements come from'' \citep[Section~7]{DGKKM22} (see \cref{fig:ctd-gauss} for a numerical comparison). A single discretized step can have $\eps$ larger by almost the grid spacing $\gridSpacing$ (\cref{clm:ctd-rr}). A simple analysis thus bounds the increase of $\eps$ for the $t$-step composition only by $t\gridSpacing$. The error analysis in \citep[Lemma~5.3]{GLW21}, based on a coupling argument, gives an upper bound of $\tilde O\left(\sqrt t\,\gridSpacing\right)$, and it applies to connect-the-dots discretization.
Our tight bounds imply that under composition the privacy profile is determined by the mean and the variance of the privacy loss, up to the Gaussian-approximation error. We show that the discretization changes the mean and the variance by only $O(\gridSpacing^2)$ per step (\cref{lem:ctd-facts}). Consequently, for bounded privacy losses (when the composed privacy loss has at least constant variance), discretization adds only $\tilde O(t\gridSpacing^2)$ to $\eps$, up to the accuracy of the Gaussian concentration (\cref{cor:ctd-error}). Whenever discretization is the main source of error, this improves quadratically on the $\tilde O\left(\sqrt t\,\gridSpacing\right)$ bound in \citep{GLW21}.

\subsection{Related work}\label{sec:related}
Composition is the most basic and the most extensively studied property of differential privacy \citep{DR14}. Advanced composition \citep{DRV10} shows that the privacy parameter of $t$ steps grows as $\sqrt t$ rather than as $t$, and the optimal composition of $(\eps_0,\delta_0)$-DP steps was characterized in \citep{KOV15}. Concentrated differential privacy \citep{DR16} and its variants \citep{BS16,BDRS18}, R\'enyi differential privacy \citep{Mironov17} and Gaussian differential privacy \citep{DRS22} describe privacy by quantities that compose simply. We refer to \citep{Steinke25} for a detailed survey.

The analysis of DP-SGD \citep{SCS13,BST14,ACGMMT16}, which combines privacy amplification by subsampling \citep{KLNRS08,BBG18} with composition over many steps, has driven the development of both the theory and the numerical methods. The moments accountant \citep{ACGMMT16} and the analyses that followed it \citep{MTZ19,ZW19} bound the R\'enyi divergences of the subsampled Gaussian mechanism. Numerical accountants compute the privacy profile of the composition from the PLDs of the steps \citep{MM18,KJH20,KJPH21,KH21,GLW21,DGKKM22}, and dominating pairs \citep{ZDW22} reduce adaptive compositions to products of distributions. Central limit theorems for the privacy loss go back to \citep{SMM19}. \citet[Theorem~3.5]{DRS22} prove that the composition of many mechanisms that are each close to perfect privacy is asymptotically GDP (see also \citep{BDLS20}). However, \citet{GLW21} point out that the GDP accountant can significantly underreport the true $\eps$ of DP-SGD. \citet{GKKCHBH26} observe that the privacy profile computed by a numerical accountant for DP-SGD with large noise and number of steps is close to that of GDP up to a small error. Their approximated GDP does not, however, upper bound the full privacy profile, whereas shifted GDP gives such upper bounds. Refined approximations use Edgeworth expansions \citep{ZDLS20,WGZSSW26} or the saddle-point method \citep{AGACKS23}. We compare these approaches with ours in \cref{sec:comparison}.

Our approach is closer in spirit to concentration-based analyses of composition, such as the proof of advanced composition by Azuma's inequality \citep{DRV10}, concentrated differential privacy \citep{DR16,BS16}, and martingale inequalities for privacy losses with subgaussian increments \citep{CR21}. Such subgaussian bounds, however, do not capture amplification by subsampling (\cref{app:notions}).

Practical implementations of DP-SGD usually process the data in passes in which every element is used once, whereas the privacy analysis assumes Poisson subsampling, in which batch sizes vary and an element can be used in any number of steps. In random allocation or \emph{balls-and-bins} sampling every element is used in exactly one of the $t$ steps, chosen uniformly at random. Random allocation was proposed for DP-SGD \citep{CGHLKKMSZ25}, for matrix mechanisms \citep{CCGHST25} and for private low-communication distributed mean estimation \citep{AFKRT26}. \citet{FS25ra} gave its dominating pair, from which our analysis starts. Several theoretical analyses, including that of its R\'enyi divergence, were given in \citep{FS25ra,DCO25,SK26}. \citet{FS26pld} showed that its PLD can be computed efficiently and accurately, and \citet{DO26} compared the sampling schemes asymptotically.

\section{Preliminaries}\label{sec:prelim}
We denote the domain of elements by $\domain$ and the set of outputs by $\outDom$. An algorithm $\alg$ maps every dataset $\dataset \in \domain^{*}$ to a distribution over $\outDom$. For a distribution $P$ over $\outDom$ we write $P(\omega)$ for the probability of $\omega$ or, in the continuous case, for its density with respect to a fixed dominating measure. All spaces are standard Borel and all randomized mappings are measurable.
We denote by $\Phi$ the distribution function of the standard normal distribution and by $\Phibar\coloneqq1-\Phi$ its tail, and we use that $\Phibar(w)\le e^{-w^2/2}/2$ for every $w\ge0$.

\subsection{Privacy losses and dominating pairs}\label{sec:prelim-pairs}
We augment the domain to $\domain_\bot\coloneqq\domain\cup\{\bot\}$ with a null element $\bot$ such that $\alg(\dataset, \bot) = \alg(\dataset)$ for all $\dataset \in (\domain_\bot)^*$. We call two datasets in $(\domain_\bot)^*$ \emph{adjacent} if one is obtained from the other by replacing a single element by $\bot$ (zero-out adjacency). We write their induced output distributions as $\alg(\dataset,x)$ and $\alg(\dataset,\bot)$, where $x\in\domain$ is the \emph{distinguished} element. 

\begin{definition}[Privacy loss and hockey-stick divergence \citep{BKOZB12,DR16}]\label{def:hs}
Let $P$ and $Q$ be distributions over $\outDom$. The \emph{privacy loss} of $\omega\in\outDom$ is $\lossFunc{\omega}{P}{Q}\coloneqq\ln(P(\omega)/Q(\omega))$, which is $+\infty$ when $P(\omega)>0$ and $Q(\omega)=0$. The \emph{privacy loss random variable} $\lossRVfunc{P}{Q}$ is $\lossFunc{\omega}{P}{Q}$ for $\omega\sim P$, and its distribution is the \emph{privacy loss distribution (PLD)}. Denoting $[\,\cdot\,]_+\coloneqq\max\{\cdot\,,0\}$, for any $\kappa\ge0$, the \emph{$\kappa$-hockey-stick divergence} is
\[
 \HockeyStickDiv{\kappa}{P}{Q}\coloneqq\int_{\outDom}\left[P(\omega)-\kappa Q(\omega)\right]_+d\omega
 =\expect{}{\left[1-\kappa e^{-\lossRVfunc{P}{Q}}\right]_+}.
\]
\end{definition}
The hockey-stick divergence of a pair depends only on the distribution of its privacy loss. Following \citep[Definition~2.2]{FS26pld}, for a random variable $X$ with values in $[-\infty,\infty]$ and $\kappa\ge0$ we write $\HockeyStick{\kappa}(X)\coloneqq\expect{}{[1-\kappa e^{-X}]_+}$, so that $\HockeyStickDiv{\kappa}{P}{Q}=\HockeyStick{\kappa}(\lossRVfunc{P}{Q})$ by \cref{def:hs}. For every $l \in\reals$ we have that $\HockeyStick{e^{\eps}}(X+l)=\HockeyStick{e^{\eps-l}}(X)$.

We denote the mean and the variance of $\lossRVfunc{P}{Q}$ by $\lossMean{P}{Q}$ and $\lossVar{P}{Q}$.
For a set $\truncSet$, we denote by $\truncMean{P}{Q}$ and $\truncVar{P}{Q}$ the mean and the variance of the \emph{truncated privacy loss} $\lossFunc{\omega}{P}{Q}\,\ind\{\omega\in\truncSet\}$ for $\omega\sim P$, which is set to zero outside $\truncSet$; the set is specified wherever these moments appear.

We say that distributions $P$ and $Q$ are \emph{$(\eps,\delta)$-close} if $\HockeyStickDiv{e^{\eps}}{P}{Q}\le\delta$ and $\HockeyStickDiv{e^{\eps}}{Q}{P}\le\delta$. An algorithm is \emph{$(\eps,\delta)$-differentially private (DP)} \citep{DKMMN06} if for all adjacent datasets their induced output distributions are $(\eps,\delta)$-close. We refer to the case where the first dataset has $\bot$ in place of the distinguished element as the \emph{add} direction and to the reverse as the \emph{remove} direction \citep{FS25ra}. This definition is closely related to differential privacy with respect to adding or removing an element (with $\bot$ corresponding to removal). The \emph{privacy profile} of a pair $(P,Q)$ is the function $\eps\mapsto\HockeyStickDiv{e^{\eps}}{P}{Q}$ \citep{BBG18}.

For $\kappa=e^{\eps}$, the random variable $[1-e^{\eps-\lossRVfunc{P}{Q}}]_+$ in \cref{def:hs} is at most the indicator of the event $\{\lossRVfunc{P}{Q}>\eps\}$. Therefore we have the \emph{tail bound} $\HockeyStickDiv{e^{\eps}}{P}{Q}\le\prob{}{\lossRVfunc{P}{Q}>\eps}$.
For $P$ absolutely continuous with respect to $Q$ we denote the likelihood ratio by $\ratioFunc{\omega}{P}{Q}\coloneqq e^{\lossFunc{\omega}{P}{Q}} = P(\omega)/Q(\omega)$, and by $\ratioRVfunc{P}{Q}\coloneqq e^{\lossRVfunc{P}{Q}}$ its value for $\omega\sim P$. For $\omega\sim Q$ the likelihood ratio $\ratioFunc{\omega}{P}{Q}=e^{-\lossFunc{\omega}{Q}{P}}$ is distributed as $1/\ratioRVfunc{Q}{P}$, it has mean $1$, and its variance is the \emph{$\chi^2$ divergence}
\[
 \chisq{P}{Q}\coloneqq\expect{\omega\sim Q}{(\ratioFunc{\omega}{P}{Q}-1)^2}=\expect{\omega\sim Q}{\ratioFunc{\omega}{P}{Q}^2}-1.
\]

\begin{definition}[Dominating pair \citep{ZDW22}]\label{def:dom}
A pair of distributions $(P,Q)$ \emph{dominates} a pair $(P',Q')$ if $\HockeyStickDiv{\kappa}{P'}{Q'}\le\HockeyStickDiv{\kappa}{P}{Q}$ for every $\kappa\ge0$. It \emph{dominates} an algorithm $\alg$ if it dominates $(\alg(\dataset,x),\alg(\dataset,\bot))$ for all adjacent datasets $(\dataset,x)$ and $(\dataset,\bot)$.
\end{definition}
We use the following characterization of domination by post-processing.
\begin{lemma}[{\citep[Theorem~2.5]{KOV15}}]\label{lem:kov}
The pair $(P,Q)$ dominates $(P',Q')$ if and only if there is a randomized mapping $\varphi$ with $\varphi(P)=P'$ and $\varphi(Q)=Q'$.
\end{lemma}

In the \emph{$t$-step adaptive composition} of algorithms $\alg_1,\ldots,\alg_t$, the $i$-th algorithm also receives the outputs $\out_1,\ldots,\out_{i-1}$ of the previous steps, and the composition outputs $(\out_1,\ldots,\out_t)$, where $\out_i$ is drawn from the output distribution of $\alg_i(\cdot\,;\out_1,\ldots,\out_{i-1})$ on the dataset. For $\lambda\in[0,1]$, the \emph{$\lambda$-subsampled} algorithm $\alg_\lambda$ runs $\alg$ on the dataset in which every element is independently replaced by $\bot$ with probability $1-\lambda$ (Poisson subsampling); in a composition of subsampled algorithms, the subsampling is independent at every step. For distributions $P$ and $Q$ we similarly write $P_\lambda\coloneqq(1-\lambda)Q+\lambda P$, and we denote by $Q^t$ the product of $t$ copies of $Q$ (and use the same notation for other distributions). We use the following standard forms of composition and subsampling for dominating pairs.
\begin{lemma}[Composition and subsampling {\citep[Theorems~10 and~11(1)]{ZDW22}}]\label{lem:dom-comp-sub}
Let $t\in\naturals$, and let $\alg_1,\ldots,\alg_t$ be algorithms, where $\alg_i$ also receives the outputs of the previous steps.
\begin{enumerate}
 \item (Composition) If for every $i\in[t]$ and all outputs $\out_1,\ldots,\out_{i-1}$ the pair $(P_i,Q_i)$ dominates $\alg_i(\cdot\,;\out_1,\ldots,\out_{i-1})$, then the pair $(P_1\times\cdots\times P_t,\,Q_1\times\cdots\times Q_t)$ dominates the $t$-step adaptive composition of $\alg_1,\ldots,\alg_t$.
 \item (Subsampling) If $(P,Q)$ dominates an algorithm $\alg$, then for every $\lambda\in[0,1]$ the pair $(P_\lambda,Q)$ dominates $\alg_\lambda$.
\end{enumerate}
\end{lemma}
In particular, \cref{lem:dom-comp-sub} implies that if $(P,Q)$ dominates $\alg_i(\cdot\,;\out_1,\ldots,\out_{i-1})$ for every $i\in[t]$ and all outputs $\out_1,\ldots,\out_{i-1}$, then $(P_\lambda^t,Q^t)$ dominates the $t$-step adaptive composition of the $\lambda$-subsampled algorithms.
The pair $(P,Q)$ dominates $(P',Q')$ if and only if $(Q,P)$ dominates $(Q',P')$~\citep[Lemma~28]{ZDW22}, and thus $(Q^t,P_\lambda^t)$ dominates the pairs of output distributions in the add direction.

The \emph{Gaussian mechanism} with noise scale $\sigma$ and unit norm bound, for a function $f\colon\domain_\bot\to\reals^d$, is the algorithm
\begin{equation}\label{eq:gauss-mech}
 \alg(\dataset)=\Nrm\left(\sum_{z\in\dataset}f(z),\,\sigma^2I_d\right),\qquad \|f(z)\|_2\le1,\qquad f(\bot)=0,
\end{equation}
where in an adaptive composition the function $f$ may depend on the outputs of the previous steps (for example, $f(z)$ is the clipped gradient of the loss on $z$ at the current model). We refer to $(\Nrm(1,\sigma^2),\Nrm(0,\sigma^2))$ as the \emph{Gaussian pair} and use the following standard fact.
\begin{claim}[{\citep[Lemma~3]{BW18}}]\label{clm:gauss-dom}
The Gaussian pair dominates the Gaussian mechanism.
\end{claim}

\subsection{Concentration and shifted GDP}\label{sec:prelim-gauss}
We use two notions of approximation by a Gaussian, both up to a shift and an additive error.
Shifted GDP bounds the privacy profiles of an algorithm from above by that of Gaussian differential privacy (GDP), and Gaussian concentration bounds the distribution function of a privacy loss from above and from below by that of a Gaussian random variable.
In \cref{sec:gaussapprox} we prove Gaussian concentration of the privacy losses of product pairs, which implies shifted GDP for the algorithms that these pairs dominate (\cref{lem:gc-profile}).
Bernstein's inequality, which we use in \cref{sec:bernstein}, gives only upper bounds.

\para{Shifted GDP.}
For $V>0$ we call $\mathcal G_V\coloneqq(\Nrm(1,1/V),\Nrm(0,1/V))$ the \emph{Gaussian reference} with parameter $V$ and denote its privacy profile by $g_V$; for $V=0$ we let $\mathcal G_0$ be a pair of identical distributions, with $g_0(\eps)=[1-e^{\eps}]_+$, and take $\Nrm(\mu,0)$ to be the point mass at $\mu$. Since adding independent Gaussian noise (or a constant map) sends $\mathcal G_V$ to $\mathcal G_{V'}$ for $0\le V'<V$, $g_V(\eps)$ is nondecreasing in $V$ (\cref{lem:kov}).
An algorithm is \emph{$\sqrt V$-Gaussian differentially private ($\sqrt V$-GDP)} if and only if it is $(\eps,g_V(\eps))$-DP for every $\eps\ge0$ \citep[Corollary~2.13]{DRS22}.
We relax this notion by allowing the privacy profiles to exceed $g_V$ up to a shift of $\eps$ and an additive error. To do so, we use the following relaxation of domination (\cref{def:dom}).
\begin{definition}[Approximate domination]\label{def:approx-dom}
Let $\lossShift\in\reals$ and $\tailShift\ge0$. A pair of distributions $(P,Q)$ \emph{$(\lossShift,\tailShift)$-approximately dominates} a pair $(P',Q')$ if $\HockeyStickDiv{e^{\eps}}{P'}{Q'}\le\HockeyStickDiv{e^{\eps-\lossShift}}{P}{Q}+\tailShift$ for every $\eps\in\reals$.
\end{definition}
Since $\HockeyStickDiv{0}{P}{Q}=1$ for every pair, $(0,0)$-approximate domination is domination.
Shifted GDP is approximate domination by the Gaussian reference in both directions.
\begin{definition}[Shifted GDP]\label{def:shifted-gdp}
Let $V,\tailShift\ge0$ and $\lossShift\in\reals$. An algorithm $\alg$ is \emph{$(\lossShift,\tailShift)$-shifted $\sqrt V$-GDP} if for any $\dataset \in (\domain_\bot)^*, x\in\domain$ and every $\eps\in\reals$,
\[
 \max\left\{\HockeyStickDiv{e^{\eps}}{\alg(\dataset,x)}{\alg(\dataset,\bot)}, \HockeyStickDiv{e^{\eps}}{\alg(\dataset,\bot)}{\alg(\dataset,x)} \right\}\le g_V(\eps-\lossShift)+\tailShift.
\]
\end{definition}
For convenience, we also say that a pair $(P,Q)$ is $(\lossShift,\tailShift)$-shifted $\sqrt V$-GDP if the same holds for $\HockeyStickDiv{e^{\eps}}{P}{Q}$ and $\HockeyStickDiv{e^{\eps}}{Q}{P}$.
Note that the shift $\lossShift$ can be negative when $\tailShift>0$.
The privacy loss of $\mathcal G_V$ is distributed as $\Nrm(V/2,V)$.
We denote by $\eps_{\mathrm G}(V,\delta)\coloneqq\inf\{\eps\ge0:g_V(\eps)\le\delta\}$ the smallest $\eps$ for which $\sqrt V$-GDP implies $(\eps,\delta)$-DP.
By the tail bound and the bound on $\Phibar$, we have that $\eps_{\mathrm G}(V,\delta)\le V/2+\sqrt{2V\ln(1/\delta)}$ for $\delta \in (0,1)$.

Shifted GDP is preserved by post-processing, that is, by domination (\cref{lem:kov}). If $(P,Q)$ is $(\lossShift,\tailShift)$-shifted $\sqrt V$-GDP, then so is every pair and every algorithm that it dominates.
Approximate domination composes, with the parameters adding up (\cref{lem:approx-dom-comp}).
Since the product of $\mathcal G_{V_1}$ and $\mathcal G_{V_2}$ has the PLD of $\mathcal G_{V_1+V_2}$, the adaptive composition of an $(\lossShift_1,\tailShift_1)$-shifted $\sqrt{V_1}$-GDP algorithm with an algorithm that is $(\lossShift_2,\tailShift_2)$-shifted $\sqrt{V_2}$-GDP for all outputs of the first is $(\lossShift_1+\lossShift_2,\tailShift_1+\tailShift_2)$-shifted $\sqrt{V_1+V_2}$-GDP.

\para{Gaussian concentration.}
We compare a privacy loss with a Gaussian random variable using approximate stochastic domination.
\begin{definition}[Gaussian concentration and approximation]\label{def:gauss-conc}
Let $\lossShift,\tailShift\ge0$. A random variable $Y$ \emph{$(\lossShift,\tailShift)$-approximately stochastically dominates} a random variable $X$, denoted by $X\preceq_{(\lossShift,\tailShift)}Y$, if
\[
 \prob{}{X>\loss}\le\prob{}{Y>\loss-\lossShift}+\tailShift\qquad\text{for every }\loss\in[-\infty,\infty] .
\]
For $\mu\in\reals$ and $V\ge0$, a random variable $X$ is \emph{$(\lossShift,\tailShift)$-approximately $(\mu,V)$-Gaussian-concentrated} if $G\preceq_{(\lossShift,\tailShift)}X\preceq_{(\lossShift,\tailShift)}G$ for $G\sim\Nrm(\mu,V)$. We refer to $\lossShift$ as the \emph{loss shift} and to $\tailShift$ as the \emph{tail shift}.
\end{definition}
Note that both relations follow from a coupling of $X$ and $G$ for which $|X-G|\le\lossShift$ except with probability $\tailShift$ (a coupling approximation in the terminology of \citep[Definition~5.1]{GLW21}).
Indeed, $\{X>\loss\}\subseteq\{G>\loss-\lossShift\}\cup\{|X-G|>\lossShift\}$.
Also, directly from the definition, if $X\preceq_{(\lossShift_1,\tailShift_1)}Y$ and $Y\preceq_{(\lossShift_2,\tailShift_2)}Y'$, then $X\preceq_{(\lossShift_1+\lossShift_2,\tailShift_1+\tailShift_2)}Y'$.
This relation is the approximate stochastic domination of \citep[Definition~4.1]{FS26pld}, which for $\lossShift=\tailShift=0$ is first-order stochastic domination.
For privacy losses, it implies approximate domination with the same parameters. That is, if $\lossRVfunc{P'}{Q'}\preceq_{(\lossShift,\tailShift)}\lossRVfunc{P}{Q}$, then $(P,Q)$ $(\lossShift,\tailShift)$-approximately dominates $(P',Q')$ \citep[Claim~4.2]{FS26pld}.

\para{From Gaussian concentration to shifted GDP.}
We use the following form of this implication for general random variables.
\begin{claim}[Hockey-stick divergence under approximate stochastic domination {\citep[Claim~4.2]{FS26pld}}]\label{clm:stoch-hs}
Given $\lossShift,\tailShift\ge0$ and random variables $X$ and $Y$ with values in $[-\infty,\infty]$, if $X\preceq_{(\lossShift,\tailShift)}Y$, then $\HockeyStick{e^{\eps}}(X)\le\HockeyStick{e^{\eps-\lossShift}}(Y)+\tailShift$ for every $\eps\in\reals$.
\end{claim}
For $G\sim\Nrm(\mu,V)$ the random variable $G-\mu+V/2$ is distributed as the privacy loss of $\mathcal G_V$, and therefore
\begin{equation}\label{eq:shift-identity}
 \HockeyStick{e^{\eps}}(G)=g_V\left(\eps-\mu+\frac V2\right)\qquad\text{for every }\eps\in\reals .
\end{equation}
Thus the excess $\mu-V/2$ of the mean of a Gaussian random variable over half its variance acts as a bias of $\eps$ in the privacy profile of $\sqrt V$-GDP.
Applying \cref{clm:stoch-hs} in both directions to a privacy loss and a Gaussian random variable and combining it with \cref{eq:shift-identity}, we obtain the following two-sided bound.
\begin{lemma}[Privacy profiles of Gaussian-concentrated privacy losses]\label{lem:gc-profile}
Let $(P,Q)$ be a pair of distributions whose privacy loss $\lossRVfunc{P}{Q}$ is $(\lossShift,\tailShift)$-approximately $(\mu,V)$-Gaussian-concentrated. Then for every $\eps\in\reals$,
\[
 g_V\left(\eps-\mu+\frac V2+\lossShift\right)-\tailShift\le\HockeyStickDiv{e^{\eps}}{P}{Q}\le g_V\left(\eps-\mu+\frac V2-\lossShift\right)+\tailShift .
\]
\end{lemma}
By the upper bound in \cref{lem:gc-profile}, if the privacy losses of both $(P,Q)$ and $(Q,P)$ are $(\lossShift,\tailShift)$-approximately $(\mu,V)$-Gaussian-concentrated, then every algorithm dominated by $(P,Q)$ is $(\mu-V/2+\lossShift,\tailShift)$-shifted $\sqrt V$-GDP.\footnote{In fact, this holds even if only the upper bound part of the approximation holds.}
These upper bounds need not be tight, since the Gaussian concentration bounds the privacy profiles of $(P,Q)$ from both sides but those of the algorithm only from above.
We compare shifted GDP with concentrated differential privacy and its variants \citep{DR16,BS16,BDRS18} in \cref{app:notions}.

\section{Gaussian approximation of compositions}\label{sec:gaussapprox}
In this section we show that the privacy loss of a composition of steps with small privacy losses is Gaussian-concentrated, with the mean and the variance of its truncated privacy loss. We first derive an upper bound from Bernstein's inequality and then prove the two-sided concentration for privacy losses that are bounded except with small probability. We defer the proofs to \cref{app:gaussapprox-proofs}.

Throughout this section, $(P,Q)$ is a pair of mutually absolutely continuous distributions. 
By \cref{lem:dom-comp-sub}, an upper bound on the privacy profile of $(P^t,Q^t)$ also holds for the $t$-step adaptive composition of algorithms whose steps are dominated by $(P,Q)$ for all outputs of the previous steps. The privacy loss of a product pair is the sum of the privacy losses of its coordinates: $\lossFunc{\omega}{P^t}{Q^t}=\sum_{i=1}^{t}\lossFunc{\omega_i}{P}{Q}$ for $\omega=(\omega_1,\ldots,\omega_t)$.
Thus $\lossRVfunc{P^t}{Q^t}$ is a sum of $t$ independent copies of $\lossRVfunc{P}{Q}$, and therefore its mean and variance are $t\cdot\lossMean{P}{Q}$ and $t\cdot\lossVar{P}{Q}$.
The proofs of this section do not use that the steps are identical. By the linearity of the expectation and of the variance of independent random variables, the results extend to the products of different pairs $(P_1,Q_1),\ldots,(P_t,Q_t)$, with the sums of the means and of the variances of their privacy losses in place of $t\cdot\lossMean{P}{Q}$ and $t\cdot\lossVar{P}{Q}$.
We state the results of this section for the pair $(P^t,Q^t)$. Applied to the reversed pair $(Q,P)$, the results give bounds of the same form in the add direction, with the assumptions and the moments of its privacy loss $-\lossFunc{\omega}{P}{Q}$ for $\omega\sim Q$.

\subsection{Warm-up: Bernstein's inequality}\label{sec:bernstein}
Our first bound relies on the one-sided form of Bernstein's inequality \citep[Theorem~2.10]{BLM13}, which bounds the sum of $t$ independent centered random variables that are at most $\rangeBnd$ almost surely by $\sqrt{2V\ln(1/\delta)}+\rangeBnd \ln(1/\delta)/3$ with probability at least $1-\delta$ for every $\delta \in (0,1]$, where $V$ is an upper bound on the sum of their variances.
We apply it to $t$ independent copies of the centered privacy loss $\lossRVfunc{P}{Q}-\lossMean{P}{Q}$, and the tail bound then gives the following upper bound.
\begin{theorem}[Composition of bounded privacy losses]\label{thm:bernstein-bounded}
Let $t\in\naturals$ and $\rangeBnd\ge0$, and assume that $\lossRVfunc{P}{Q}-\lossMean{P}{Q}\le \rangeBnd$ almost surely. Then for every $\delta \in (0,1)$,
\[
 \HockeyStickDiv{e^{\eps}}{P^t}{Q^t}\le\delta
 \qquad\text{for}\qquad
 \eps=t\cdot\lossMean{P}{Q}+\sqrt{2t\cdot\lossVar{P}{Q}\ln(1/\delta)}+\frac \rangeBnd3\ln(1/\delta).
\]
\end{theorem}
To compare \cref{thm:bernstein-bounded} with the Gaussian approximation, we write its bound as
\[
 \eps=\left(\frac{t\cdot\lossVar{P}{Q}}2+\sqrt{2t\cdot\lossVar{P}{Q}\ln(1/\delta)}\right)+t\cdot\left(\lossMean{P}{Q}-\frac{\lossVar{P}{Q}}2\right)+\frac\rangeBnd3\ln(1/\delta) .
\]
The first term is the upper bound on $\eps_{\mathrm G}(t\cdot\lossVar{P}{Q},\delta)$ given by the tail bound of the Gaussian.
The second term is the excess of the mean over half the variance, which is a bias of the privacy profile itself rather than an error of the bound.
By \cref{eq:shift-identity}, it is the bias in $\eps$ between the privacy profile of $\sqrt{t\cdot\lossVar{P}{Q}}$-GDP and that of a Gaussian random variable with the same mean and variance as the privacy loss.
The third term is a loss shift error. It does not depend on either the variance or the number of steps. 

%It is of lower order than the first term when $\rangeBnd$ is well below $\sqrt{t\cdot\lossVar{P}{Q}/\ln(1/\delta)}$ and of the same order when $\rangeBnd$ is of this order. 
We call $\sqrt{t\cdot\lossVar{P}{Q}/\ln(1/\delta)}$ the \emph{balanced cutoff}.
Put simply, up to the unavoidable bias of the privacy profile, the bound exceeds $\eps_{\mathrm G}(t\cdot\lossVar{P}{Q},\delta)$ only by two errors (the third term and the error of the tail bound).
Thus Bernstein's inequality essentially suffices for the upper bounds at the central limit rate in our applications.

If the privacy losses are bounded only with high probability, we first truncate them from above at the balanced cutoff. This truncation does not increase their means or second moments and keeps the third term of the same order as the first. It changes the sum of the privacy losses only if the privacy loss of some step exceeds the cutoff, and we add the probability of this event to $\delta$.
In \cref{sec:chi2-subsampling} we do so for subsampled mechanisms, whose privacy losses may be unbounded (\cref{thm:subsampled}).

\subsection{Gaussian concentration}\label{sec:approx-composition}
To bound the privacy profiles also from below, we couple the privacy loss with a Gaussian random variable.
As noted above, such a coupling gives Gaussian concentration, with the bound on its error as the loss shift and the probability that this bound fails as the tail shift.

The key concentration result we use is the following coupling of the sum of bounded random variables with a Gaussian random variable of the same mean and variance. We derive it in \cref{app:gaussapprox-proofs} from the bound on an exponential moment of the error of a Gaussian coupling in \citep[Theorem~1 and Corollary~1]{Sakhanenko96} (\cref{thm:sakhanenko}), and it is also a special case of a more general coupling \citep[Theorems~1 and~2]{Zaitsev25}.
\begin{lemma}[Gaussian coupling of bounded summands]\label{lem:coupling}
There is a universal constant $C>0$ such that the following holds. Let $t\in\naturals$ and $\rangeBnd\ge0$, and let $X_1,\ldots,X_t$ be independent mean-zero random variables with $|X_i|\le \rangeBnd$ almost surely, and let $S\coloneqq\sum_{i=1}^tX_i$. Then $S$ can be coupled with $G\sim\Nrm(0,\Var(S))$ so that $\prob{}{|S-G|>C\rangeBnd\ln(2/\tailShift)}\le\tailShift$ for every $0<\tailShift\le1/2$.
\end{lemma}

We apply \cref{lem:coupling} to $\lossRVfunc{P}{Q}$ by first truncating the $\tailShift/(2t)$ tail and then coupling the sum of the truncated privacy losses with a Gaussian random variable.

\begin{theorem}[Gaussian concentration for compositions]\label{thm:gauss-approx}
There is a universal constant $C>0$ such that for any $t\in\naturals$, $\rangeBnd>0$, $0<\tailShift\le1$ and a set $\truncSet$ such that $P(\truncSet^c)\le\tailShift/(2t)$ and the truncated privacy loss $\lossFunc{\omega}{P}{Q}\,\ind\{\omega\in \truncSet\}$ for $\omega\sim P$ deviates from its mean $\truncMean{P}{Q}$ by at most $\rangeBnd$ almost surely, the privacy loss $\lossRVfunc{P^t}{Q^t}$ is $(C\rangeBnd\ln(2/\tailShift),\tailShift)$-approximately $(t\cdot\truncMean{P}{Q},t\cdot\truncVar{P}{Q})$-Gaussian-concentrated, where $\truncVar{P}{Q}$ is the variance of this truncated privacy loss.
\end{theorem}
Note that the loss shift $C\rangeBnd\ln(2/\tailShift)$ depends neither on the variance nor on the number of steps. If $|\lossRVfunc{P}{Q}-\lossMean{P}{Q}|\le\rangeBnd$ almost surely, then \cref{thm:gauss-approx} applies for every $0<\tailShift\le1$ with $\truncSet$ the whole space, and thus with $\truncMean{P}{Q}=\lossMean{P}{Q}$ and $\truncVar{P}{Q}=\lossVar{P}{Q}$.

By \cref{thm:gauss-approx,lem:gc-profile}, the privacy profile of $(P^t,Q^t)$ is bounded from both sides by that of $\sqrt{t\cdot\truncVar{P}{Q}}$-GDP biased in $\eps$ by the excess $t(\truncMean{P}{Q}-\truncVar{P}{Q}/2)$, up to the loss shift $C\rangeBnd\ln(2/\tailShift)$ in $\eps$ and the tail shift $\tailShift$ in $\delta$.
As in the bound of \cref{thm:bernstein-bounded}, the excess is a bias of the privacy profile itself and moves both bounds by the same amount, whereas the loss shift is the error of the approximation and moves them apart.
In the upper bound the loss shift adds to the excess: $\HockeyStickDiv{e^{\eps}}{P^t}{Q^t}\le g_{t\cdot\truncVar{P}{Q}}(\eps-\lossShift)+\tailShift$ for every $\eps$, where $\lossShift\coloneqq t(\truncMean{P}{Q}-\truncVar{P}{Q}/2)+C\rangeBnd\ln(2/\tailShift)$.
At the central limit rate, $\eps$ is of order $\sqrt{t\cdot\truncVar{P}{Q}\ln(1/\delta)}$. The loss shift $C\rangeBnd\ln(2/\tailShift)$ is of lower order than $\eps$ when $\tailShift$ is of order $\delta$ and $\rangeBnd$ is well below the balanced cutoff $\sqrt{t\cdot\truncVar{P}{Q}/\ln(1/\delta)}$ (equivalently, $t\gg\rangeBnd^2\ln(1/\delta)/\truncVar{P}{Q}$).
\iflipics\else In \cref{app:cutoff} we show that $\truncMean{P}{Q}$ and $\truncVar{P}{Q}$ can be replaced by $\lossMean{P}{Q}$ and $\lossVar{P}{Q}$ if the deviations of the privacy losses from their means have exponentially decaying tails.\fi

\section{Subsampling via the \texorpdfstring{$\chi^2$}{chi-square} divergence}\label{sec:chi2-subsampling}
We now apply the results of \cref{sec:gaussapprox} to compositions of subsampled mechanisms. Unfortunately, for many subsampled mechanisms, most notably the subsampled Gaussian, the variance of the privacy loss has no explicit analytic expression. In this section we show that the $\chi^2$ divergence $\chiVar$ of the original pair $(P,Q)$ can be used to tightly approximate the means and the variances of the subsampled privacy losses in both directions, and thus the parameters of the Gaussian approximation of their composition. The value $\chiVar$ is known for many distributions and is equal to $e^{D_2(P\|Q)}-1$, where $D_2(P\|Q)$ is the order-$2$ R\'enyi divergence of $(P,Q)$. We defer the proofs to \cref{app:bernstein-proofs}.

The first property of $\chiVar$ that we use is that it upper bounds the mean and the variance parameters of the privacy loss. We prove this in \cref{app:chi2-tightness}, where we also show (\cref{lem:chi2-tight}) that the bound on the variance is tight up to lower-order terms when a fourth moment of the likelihood ratio is small.
\begin{lemma}[Mean and variance are upper bounded by $\chi^2$]\label{lem:chi2-moments}
Let $P\ll Q$ with $\chisq{P}{Q}<\infty$, let $\truncSet$ be a set, and let $\truncVar{P}{Q}$ be the variance of the truncated privacy loss $\lossFunc{\omega}{P}{Q}\,\ind\{\omega\in\truncSet\}$ for $\omega\sim P$. Then $0\le\lossMean{P}{Q}\le \chisq{P}{Q}$ and $\truncVar{P}{Q}\le\expect{}{\lossRVfunc{P}{Q}^2}\le \chisq{P}{Q}$.
\end{lemma}

More importantly, the $\chi^2$ divergence is easy to bound under subsampling. Specifically, subsampling at rate $\lambda$ scales it by $\lambda^2$ (up to a factor $1/(1-\lambda)$ in the add direction).
\begin{lemma}[Subsampling and the $\chi^2$ divergence]\label{lem:subsampling-chi2}
Let $P$ and $Q$ be mutually absolutely continuous and let $\lambda \in (0,1/2)$. Then $\chisq{P_\lambda}{Q}=\lambda^2\chisq{P}{Q}$ and $\chisq{Q}{P_\lambda}\le\lambda^2\chisq{P}{Q}/(1-\lambda)$.
\end{lemma}

Tail conditions on the subsampled privacy losses are also easy to translate into tail conditions on the original pair.
\begin{lemma}[Tails of subsampled privacy losses]\label{lem:subsampling-tails}
Let $P$ and $Q$ be mutually absolutely continuous and $\lambda \in (0,1/2)$.
Then for every $l>0$,
\[
 \prob{}{\lossRVfunc{P_\lambda}{Q}>l}\le\prob{}{\lossRVfunc{P}{Q}>\ln\left(1+\frac{e^l-1}{\lambda}\right)}
 \qquad\text{and}\qquad
 \prob{}{\lossRVfunc{Q}{P_\lambda}>l}\le\prob{}{\lossRVfunc{Q}{P}>\frac l\lambda}.
\]
Moreover, for every $\rangeBnd>0$ and every $\omega$ with $|\ratioFunc{\omega}{P}{Q}-1|\le\rangeBnd$, we have $|\lossFunc{\omega}{P_\lambda}{Q}|\le2\lambda\rangeBnd$.
\end{lemma}

Next we bound the parameters of \cref{thm:bernstein-bounded,thm:gauss-approx} for $t$ identical $\lambda$-subsampled steps by the $\chi^2$ divergence $\chiVar$ of the original pair $(P,Q)$, under tail conditions on $\lossRVfunc{P}{Q}$ and $\lossRVfunc{Q}{P}$, and then apply these results to the subsampled Gaussian mechanism.

\subsection{Subsampled compositions}\label{sec:subsampling}
We truncate the privacy losses of $(P_\lambda,Q)$ and $(Q,P_\lambda)$ at a cutoff $\lambda\rangeBnd$ proportional to $\lambda$, where $\rangeBnd$ bounds the deviation of the likelihood ratio of $(P,Q)$ from $1$ except with small probability. By \cref{lem:subsampling-tails}, the probability that this cutoff is exceeded for $(P_\lambda,Q)$ (resp.\ $(Q,P_\lambda)$) is at most the probability that the privacy loss of $(P,Q)$ exceeds $\ln(1+\rangeBnd)$ (resp.\ that of $(Q,P)$ exceeds $\rangeBnd$). These thresholds do not depend on $\lambda$. Applying Bernstein's inequality to the truncated privacy losses, we obtain the following bound.
\begin{theorem}[Subsampled composition]\label{thm:subsampled}
Let $P$ and $Q$ be mutually absolutely continuous with $\chiVar\coloneqq\chisq{P}{Q}$, and let $\lambda \in (0,1/2)$, $t\in\naturals$, $0<\delta\le1/2$ and $\rangeBnd\ge\sqrt{2\chiVar}$. If $\prob{}{\lossRVfunc{P}{Q}>\ln(1+\rangeBnd)}\le\delta/(2t)$ and $\prob{}{\lossRVfunc{Q}{P}>\rangeBnd}\le\delta/(2t)$, then $P_\lambda^t$ and $Q^t$ are $(\eps,\delta)$-close for
\[
 \eps\coloneqq\frac{\subsampVar}2+\sqrt{2\subsampVar\ln(2/\delta)}+\frac23\lambda\rangeBnd\ln(2/\delta),
 \qquad\text{where}\qquad
 \subsampVar\coloneqq\frac{t\lambda^2\chiVar}{1-\lambda} .
\]
\end{theorem}
The cutoff $\rangeBnd=\sqrt{2t\cdot\chiVar/\ln(2/\delta)}$, for which $\lambda\rangeBnd$ is of the order of the balanced cutoff (\cref{sec:bernstein}), satisfies $\rangeBnd\ge\sqrt{2\chiVar}$ when $\ln(2/\delta)\le t$. For this cutoff the last term of $\eps$ is at most $\frac23\sqrt{2\subsampVar\ln(2/\delta)}$. With a cutoff that satisfies the tail conditions and is of lower order than the balanced one, this term is of lower order than $\sqrt{2\subsampVar\ln(2/\delta)}$.

For the two-sided Gaussian concentration we instead set the privacy losses of both directions to zero wherever the privacy loss of $(P,Q)$ exceeds $\ln(1+\rangeBnd)$ or that of $(Q,P)$ exceeds $\rangeBnd$. Elsewhere the likelihood ratio of the original pair is within $\rangeBnd$ of $1$. Therefore the privacy loss of the subsampled pair is at most $2\lambda\rangeBnd$ in absolute value (\cref{lem:subsampling-tails}), and the truncated privacy loss deviates from its mean by at most $4\lambda\rangeBnd$. We bound both the means and the variances of the truncated privacy losses in both directions by the $\chi^2$ divergence. From \cref{thm:gauss-approx} we then obtain the following concentration result and, as its corollary, a bound by a Gaussian privacy profile with the same parameters in both directions.
\begin{corollary}[Gaussian concentration for subsampled compositions]\label{cor:subsampled-approx}
There is a universal constant $C>0$ such that the following holds. Let $P$ and $Q$ be mutually absolutely continuous with $\chiVar\coloneqq\chisq{P}{Q}<\infty$, and let $t\in\naturals$, $\lambda \in (0,1/2)$, $\rangeBnd>0$ and $0<\tailShift\le1$ be such that $\prob{}{\lossRVfunc{P}{Q}>\ln(1+\rangeBnd)}\le\tailShift/(4t)$ and $\prob{}{\lossRVfunc{Q}{P}>\rangeBnd}\le\tailShift/(4t)$, and let $\subsampVar\coloneqq t\lambda^2\chiVar/(1-\lambda)$. Then the privacy losses of $(P_\lambda^t,Q^t)$ and $(Q^t,P_\lambda^t)$ are $(4C\lambda\rangeBnd\ln(2/\tailShift),\tailShift)$-approximately $(\mu',V')$-Gaussian-concentrated for some $\mu'\le\subsampVar/2+\lambda\tailShift/2$ and $0\le V'\le\subsampVar$ (which may differ between the two directions). In particular, every algorithm dominated by $(P_\lambda^t,Q^t)$ is $(\lossShift,\tailShift)$-shifted $\sqrt{\subsampVar}$-GDP for $\lossShift\coloneqq4C\lambda\rangeBnd\ln(2/\tailShift)+\subsampVar/2+\lambda\tailShift/2$.
If $\lossRVfunc{P}{Q}$ and $\lossRVfunc{Q}{P}$ have the same distribution, then this holds with $\mu'\le\subsampVar/2$ and without the term $\lambda\tailShift/2$ of $\lossShift$.
\end{corollary}
By \cref{lem:gc-profile}, the shift $\lossShift$ is the loss shift of the concentration plus an upper bound on $\mu'-V'/2$, for which we use the upper bound on $\mu'$ and the nonnegativity of $V'$.
The resulting term $\subsampVar/2+\lambda\tailShift/2$ of $\lossShift$ is of lower order than $\sqrt{\subsampVar\ln(1/\tailShift)}$ when $\tailShift\ll\subsampVar\ll\ln(1/\tailShift)$.

It is natural to ask how tight the bounds based on $\subsampVar$ are, compared with the true variances of the subsampled privacy losses and of their truncations used in \cref{thm:subsampled,cor:subsampled-approx}. We show that the relative error of $\subsampVar$ is of order $\lambda$, up to a factor that depends on the fourth moment of the likelihood ratio. Truncating the privacy losses can further reduce the variance, making the bound looser. We show that under the same fourth-moment bound this reduction is at most of order $\sqrt{\tailShift/t}$.
\begin{lemma}[Lower bound on variance of subsampled privacy losses]\label{lem:chi2-tight-sub}
Let $P$ and $Q$ be mutually absolutely continuous with $\chiVar\coloneqq\chisq{P}{Q}>0$ and $m_4\coloneqq\expect{\omega\sim Q}{(\ratioFunc{\omega}{P}{Q}-1)^4}<\infty$, and let $\lambda\in(0,1/2)$. Then both $\lossVar{P_\lambda}{Q}$ and $\lossVar{Q}{P_\lambda}$ are at least $(1-x-5x^2)\,\lambda^2\chiVar$, where $x\coloneqq\lambda\sqrt{m_4/\chiVar}$.
In particular, if $|\ratioFunc{\omega}{P}{Q}-1|\le\rangeBnd$ for $Q$-almost every $\omega$ and $\lambda\rangeBnd\le1$, then both are at least $(1-6\lambda\rangeBnd)\,\lambda^2\chiVar$. Furthermore, for every set $\truncSet$, the variances $\truncVar{P_\lambda}{Q}$ and $\truncVar{Q}{P_\lambda}$ of the truncated privacy losses $\lossFunc{\omega}{P_\lambda}{Q}\,\ind\{\omega\in\truncSet\}$ for $\omega\sim P_\lambda$ and $\lossFunc{\omega}{Q}{P_\lambda}\,\ind\{\omega\in\truncSet\}$ for $\omega\sim Q$ are both at least $\bigl(1-x-5x^2-(3x/\lambda)\sqrt{Q(\truncSet^c)/\chiVar}\bigr)\lambda^2\chiVar$.
\end{lemma}
Together with the upper bounds $\lossVar{P_\lambda}{Q}\le\lambda^2\chiVar$ and $\lossVar{Q}{P_\lambda}\le\lambda^2\chiVar/(1-\lambda)$ of \cref{lem:chi2-moments,lem:subsampling-chi2}, this shows that $\subsampVar$ exceeds $t$ times either variance by a factor of at most $1+O\left(\lambda\left(1+\sqrt{m_4/\chiVar}\right)\right)$. In a typical application $\lambda=O(1/t)$, and then the factor is $1+O\left(\left(1+\sqrt{m_4/\chiVar}\right)/t\right)$, which is $1+O(1/t)$ for a fixed pair $(P,Q)$. The ratio $m_4/\chiVar$ can be seen as a measure of how well the privacy loss of $(P,Q)$ is concentrated. For the Gaussian pair it is $(3+o(1))/\sigma^2$ as $\sigma\to\infty$ and $(1+o(1))\,e^{5/\sigma^2}$ as $\sigma\to0$ (since the likelihood ratio has a log-normal distribution under $Q$).

For the truncated privacy losses used in \cref{cor:subsampled-approx} (and in \cref{thm:subsampled}, with $\delta$ in place of $\tailShift$), the truncation removes outcomes of probability at most $\tailShift/(2t)$. By the last part of \cref{lem:chi2-tight-sub}, it therefore adds $O\left(\sqrt{m_4\tailShift/t}/\chiVar\right)$ to the relative error. For well-concentrated losses the truncation effect is usually negligible. For example, for the Gaussian pair $m_4/\chiVar^2$ is $3+o(1)$ as $\sigma\to\infty$ and $(1+o(1))\,e^{4/\sigma^2}$ as $\sigma\to0$. Thus this term is $O\left(\sqrt{\tailShift/t}\right)$ for a fixed $\sigma$. With the cutoff \cref{eq:gauss-cutoff}, the truncation in fact adds only $(\tailShift/t)^{1-o(1)}$ to the relative error.

Combining the bounds in \cref{lem:chi2-tight-sub} with \cref{cor:subsampled-approx}, we get the following lower bounds on the parameters of the Gaussian concentration.
\begin{corollary}[Lower bounds on concentration parameters]\label{cor:subsampled-lower}
Under the conditions of \cref{cor:subsampled-approx}, let $\chiVar>0$, and let $m_4<\infty$ and $x$ be as in \cref{lem:chi2-tight-sub}. Then \cref{cor:subsampled-approx} holds with $\mu'$ and $V'$ that also satisfy
\begin{align*}
 \mu'&\ge(1-\lambda)\left(1-x-\frac1{\chiVar}\sqrt{\frac{m_4\tailShift}{2t}}\right)\frac{\subsampVar}2-\frac{\lambda\tailShift}2,\\
 V'&\ge(1-\lambda)\left(1-x-5x^2-\frac3{\chiVar}\sqrt{\frac{m_4\tailShift}{2t}}\right)\subsampVar .
\end{align*}
If $\lossRVfunc{P}{Q}$ and $\lossRVfunc{Q}{P}$ have the same distribution, then the bound on $\mu'$ holds without the term $\lambda\tailShift/2$.
\end{corollary}

\subsection{The subsampled Gaussian mechanism}\label{sec:gauss-two-sided}
In the rest of this section $(P,Q)$ is the Gaussian pair, which dominates the Gaussian mechanism \cref{eq:gauss-mech} (\cref{clm:gauss-dom}), and $\chiVar=e^{1/\sigma^2}-1$. Both privacy losses $\lossRVfunc{P}{Q}$ and $\lossRVfunc{Q}{P}$ are distributed as $\Nrm(1/(2\sigma^2),1/\sigma^2)$ (\cref{clm:gauss-pair}).

For the bound by a Gaussian privacy profile we use
\begin{equation}\label{eq:gauss-cutoff}
 \rangeBnd\coloneqq\exp\left(\frac1{2\sigma^2}+\frac{\sqrt{2\ln(2t/\tailShift)}}{\sigma}\right)-1,
\end{equation}
where $0<\tailShift\le1$. As the privacy losses are Gaussian, the tail conditions of \cref{cor:subsampled-approx} hold for this $\rangeBnd$ (see the proof of \cref{cor:gauss-poisson}). Thus \cref{cor:subsampled-approx} applies (without the term $\lambda\tailShift/2$ of $\lossShift$, as the two privacy losses have the same distribution). Using \cref{clm:gauss-dom,lem:dom-comp-sub}, we then obtain the following bound for DP-SGD, that is, for the $t$-step adaptive composition of the $\lambda$-subsampled Gaussian mechanism.

\begin{corollary}[Shifted GDP of the subsampled Gaussian mechanism]\label{cor:gauss-poisson}
There is a universal constant $C>0$ such that the following holds. Let $\sigma>0$, $t\in\naturals$, $\lambda \in (0,1/2)$ and $0<\tailShift\le1$, and let $\rangeBnd$ be given by \cref{eq:gauss-cutoff}. Then the $t$-step adaptive composition of the $\lambda$-subsampled Gaussian mechanism \cref{eq:gauss-mech} is $(\lossShift,\tailShift)$-shifted $\sqrt{\subsampVar}$-GDP, where $\subsampVar\coloneqq t\lambda^2\left(e^{1/\sigma^2}-1\right)/(1-\lambda)$ and $\lossShift\coloneqq4C\lambda\rangeBnd\ln(2/\tailShift)+\subsampVar/2$.
\end{corollary}

The $\eps$ that \cref{cor:gauss-poisson} gives exceeds that of $\sqrt{\subsampVar}$-GDP by the shift $\lossShift$, which has two sources.
The first is the term $\subsampVar/2$, which upper bounds the excess $\mu'-V'/2$ of the mean over half the variance (using $\mu'\le\subsampVar/2$ and dropping $-V'/2$, as in \cref{cor:subsampled-approx}; the actual excess is $O(\lambda\subsampVar)$ by \cref{cor:subsampled-lower}).
For $\sigma\ge1$ it is of order $t\lambda^2/\sigma^2$. Thus it is small compared with the central limit rate $\lambda\sqrt{t\ln(1/\delta)}/\sigma$ when $\lambda\sqrt t\ll\sigma\sqrt{\ln(1/\delta)}$ (for example, for $\lambda=1/t$).
The second is the loss shift of the Gaussian concentration, which is proportional to the cutoff $\rangeBnd$.
The ratio between this term and the central limit rate does not depend on $\lambda$. It is small when $\rangeBnd$ is well below $\sqrt t/\left(\sigma\sqrt{\ln(1/\delta)}\right)$.

Converting this result into an $(\eps,\delta)$-DP guarantee, we get the following for some universal constant $C>0$ and every $\lambda\in(0,1/2)$, $t\ge2$ and $0<\delta\le1/2$. If $\sigma\ge1$ and
\begin{equation}\label{eq:gauss-conditions}
 \sigma\ln\left(1+\frac{\sqrt t}{\sigma\sqrt{C\ln(1/\delta)}}\right)\ge C\sqrt{\ln(t/\delta)},
\end{equation}
then the $t$-step adaptive composition of the $\lambda$-subsampled Gaussian mechanism is $(\eps,\delta)$-DP for $\eps=O\left(t\lambda^2/\sigma^2+\lambda\sqrt{t\ln(1/\delta)}/\sigma\right)$ (\cref{clm:gauss-closed-form}).
By \cref{lem:gauss-conditions}, for all sufficiently large $t$ the condition \cref{eq:gauss-conditions} holds when $\ln(1/\delta)\le t^{1/2-\eta}$ for a fixed $\eta>0$ and $\sigma\gtrsim\max\left\{1,\sqrt{\ln(1/\delta)}/\ln(t)\right\}$, and, for a fixed $\sigma\ge1$, when $\ln(1/\delta)\lesssim\sigma^2\ln^2(t)$.
For $\lambda=1/t$ and sufficiently large $t$, the bound above is $\eps=O\left(\sqrt{\ln(1/\delta)}/\left(\sigma\sqrt t\right)\right)$, the central limit rate.
The following corollary gives its leading constant for a fixed $\sigma$ as $t\to\infty$, by using the sharper bound of \cref{lem:gdp-small-var} on $\eps_{\mathrm G}$ instead of the tail bound.
Its first claim covers a polynomial decay $\delta=t^{-\theta}$, and its second claim covers $\delta$ that decays faster than any polynomial in $t$.
We prove it as a special case of its general version, \cref{cor:gauss-leading-gen}, which allows a noise scale $\sigma$ that depends on $t$.
\begin{corollary}[Asymptotic upper bounds for the subsampled Gaussian mechanism]\label{cor:gauss-leading}
Fix $\sigma>0$ and let $\chiVar\coloneqq e^{1/\sigma^2}-1$. If $\delta=t^{-\theta}$ for a fixed $\theta>1/2$, then, as $t\to\infty$, the $t$-step adaptive composition of the $1/t$-subsampled Gaussian mechanism \cref{eq:gauss-mech} is $(\eps,\delta)$-DP for ${\eps=(1+o(1))\sqrt{(2\theta-1)\,\chiVar\ln(t)/t}}$.
If $\delta=\delta(t)$ satisfies $\ln(1/\delta)=\omega(\ln(t))$ and $\limsup_{t\to\infty}8\ln(1/\delta)/(\sigma\ln(t))^2<1$, then it is $(\eps,\delta)$-DP for ${\eps=(1+o(1))\sqrt{2\chiVar\ln(1/\delta)/t}}$.
\end{corollary}
For $\lambda=1/t$ both sources of the shift are of lower order than these bounds: the term $\subsampVar/2$ is of order $1/t$, and in both regimes the cutoff $\rangeBnd$ grows polynomially slower than $\sqrt t$.
The leading constant therefore comes from $\eps_{\mathrm G}(\subsampVar,\delta/2)$ with $\subsampVar=\chiVar/(t-1)$.
The privacy profile of $\sqrt{\subsampVar}$-GDP is at most a fraction of order $\sqrt{\subsampVar}$ of the tail of its privacy loss. Accounting for this factor through \cref{lem:gdp-small-var}, instead of using the tail bound, effectively replaces $\ln(1/\delta)$ by $\ln(1/\delta)-\frac12\ln(t)$. For $\delta=t^{-\theta}$ this gives the factor $2\theta-1$ in place of $2\theta$.

\section{Random allocation}\label{sec:alloc}
An alternative sampling scheme to Poisson subsampling is \emph{$1$-out-of-$t$ random allocation}. This scheme consists of a $t$-step adaptive composition of algorithms $\alg_1,\ldots,\alg_t$, where every element of the dataset is assigned to one of the $t$ steps, chosen independently and uniformly at random, and the $i$-th step runs $\alg_i$ on the dataset in which the elements not assigned to it are replaced by $\bot$.
We show that the privacy losses of the dominating pair of this scheme are approximately $(\chiVar/(2t),\chiVar/t)$-Gaussian-concentrated, where $\chiVar\coloneqq\chisq{P}{Q}$ for a pair $(P,Q)$ that dominates every step. We then apply this result to the Gaussian mechanism to show that it satisfies the same bounds as the $t$-step composition of Poisson subsampling at rate $1/t$. We defer the proofs to \cref{app:alloc-proofs}.

\citet[Theorem~3.1]{FS25ra} show that if $(P,Q)$ dominates $\alg_i(\cdot\,;\out_1,\ldots,\out_{i-1})$ for every $i\in[t]$ and all outputs $\out_1,\ldots,\out_{i-1}$, then random allocation is dominated by the pair $(\Pbar{t}, Q^{t})$ in the remove direction and $(Q^{t}, \Pbar{t})$ in the add direction, where $\Pbar{t}\coloneqq\frac1t\sum_{j=1}^{t}Q^{j-1}\times P\times Q^{t-j}$.
Since $\ratioFunc{\omega}{Q^{j-1}\times P\times Q^{t-j}}{Q^t}=\ratioFunc{\omega_j}{P}{Q}$ for every $\omega=(\omega_1,\ldots,\omega_t)$ and $j\in[t]$, the likelihood ratio of $(\Pbar{t},Q^t)$ is the average $\frac1t\sum_{j=1}^t\ratioFunc{\omega_j}{P}{Q}$. We denote by $\ratioRVfunc{\Pbar{t}}{Q^t}$ and by $1/\ratioRVfunc{Q^t}{\Pbar{t}}$ its values for $\omega\sim\Pbar{t}$ and for $\omega\sim Q^t$, respectively. The privacy losses in the remove and the add direction are $\lossRVfunc{\Pbar{t}}{Q^t}=\ln(\ratioRVfunc{\Pbar{t}}{Q^t})$ and $\lossRVfunc{Q^t}{\Pbar{t}}=-\ln(1/\ratioRVfunc{Q^t}{\Pbar{t}})$, and we bound them through the concentration of these averages.
The two averages differ only in the distribution of one of their $t$ draws.

\begin{lemma}[Likelihood ratio of random allocation]\label{lem:alloc-ratio}
Let $t\in\naturals$ and $P\ll Q$, let $\ratio\coloneqq\ratioFunc{\cdot}{P}{Q}$, and let $\omega_0\sim P$ and $\omega_1,\ldots,\omega_t\sim Q$ be mutually independent. Then $1/\ratioRVfunc{Q^t}{\Pbar{t}}$ and $\ratioRVfunc{\Pbar{t}}{Q^t}$ are distributed as $\frac1t\sum_{i=1}^t\ratio(\omega_i)$ and $\frac1t\left(\ratio(\omega_0)+\sum_{i=2}^t\ratio(\omega_i)\right)$, respectively.
\end{lemma}
The claim for $1/\ratioRVfunc{Q^t}{\Pbar{t}}$ holds by definition. The claim for $\ratioRVfunc{\Pbar{t}}{Q^t}$ follows from the fact that a sample from $\Pbar{t}$ is a sample from $P\times Q^{t-1}$ with its first coordinate moved to a uniformly random position. Since the average is invariant under permutations of the coordinates, it has the same distribution under $\Pbar{t}$ as under $P\times Q^{t-1}$.

If the likelihood ratio $\ratio\coloneqq\ratioFunc{\cdot}{P}{Q}$ satisfies $|\ratio-1|\le \rangeBnd$, then \cref{lem:coupling} couples the average $1/\ratioRVfunc{Q^t}{\Pbar{t}}$ with a Gaussian random variable up to an error of order $(\rangeBnd/t)\ln(1/\tailShift)$, except with probability $\tailShift$.
Replacing one of its draws from $Q$ by a draw from $P$ changes the average by at most $2\rangeBnd/t$. Therefore the same Gaussian random variable also couples $\ratioRVfunc{\Pbar{t}}{Q^t}$.
Taking logarithms where the averages are close to $1$, we couple both privacy losses with a $\Nrm(\chiVar/(2t),\chiVar/t)$ random variable, which is distributed as the privacy loss of $\mathcal G_{\chiVar/t}$.
\begin{theorem}[Gaussian concentration for random allocation]\label{thm:ra-gauss}
There are universal constants $c,C>0$ such that the following holds. Let $P\ll Q$, $t\in\naturals$ and $\rangeBnd>0$, assume that $|\ratioFunc{\omega}{P}{Q}-1|\le \rangeBnd$ for $Q$-almost every $\omega$, and let $\chiVar\coloneqq\chisq{P}{Q}$. Then for every $0<\tailShift\le1/2$ with $(\rangeBnd/t)\ln(C/\tailShift)\le c$, the privacy losses of $(\Pbar{t},Q^t)$ and of $(Q^t,\Pbar{t})$ are $((C\rangeBnd/t)\ln(C/\tailShift),\tailShift)$-approximately $(\chiVar/(2t),\chiVar/t)$-Gaussian-concentrated. In particular, every algorithm dominated by $(\Pbar{t},Q^t)$ is $((C\rangeBnd/t)\ln(C/\tailShift),\tailShift)$-shifted $\sqrt{\chiVar/t}$-GDP.
\end{theorem}
Here the excess of the mean over half the variance is zero. Therefore, by \cref{lem:gc-profile}, the privacy profiles of both pairs are bounded from both sides by $g_{\chiVar/t}$, up to the loss shift $(C\rangeBnd/t)\ln(C/\tailShift)$ in $\eps$ and the tail shift $\tailShift$ in $\delta$.

For an unbounded likelihood ratio we condition on an event $\truncSet$ with $P(\truncSet)=Q(\truncSet)\ge1-\tailShift/(2t)$ on which $|\ratio-1|\le \rangeBnd$. The conditional pair $(P_\truncSet,Q_\truncSet)\coloneqq(P(\cdot\mid \truncSet),Q(\cdot\mid \truncSet))$ has likelihood ratio $\ratio$ on $\truncSet$, and conditioning $\Pbar{t}$ and $Q^t$ on $\truncSet^t$ simply uses $(P_\truncSet,Q_\truncSet)$ in place of $(P,Q)$. As a result, if $0<\tailShift\le1$ and $(\rangeBnd/t)\ln(2C/\tailShift)\le c$ for the constants $c$ and $C$ of \cref{thm:ra-gauss}, then the privacy losses of $(\Pbar{t},Q^t)$ and of $(Q^t,\Pbar{t})$ are $((C\rangeBnd/t)\ln(2C/\tailShift),\tailShift)$-approximately $(\chiVar_\truncSet/(2t),\chiVar_\truncSet/t)$-Gaussian-concentrated for $\chiVar_\truncSet\coloneqq\chisq{P_\truncSet}{Q_\truncSet}$ (\cref{lem:ra-conditioning}).

\para{The Gaussian mechanism.}
The likelihood ratio of the Gaussian pair $(P,Q)$ is unbounded. For $\rangeBnd$ of \cref{eq:gauss-cutoff}, the event $\truncSet\coloneqq\{\omega:|\lossFunc{\omega}{P}{Q}|\le\ln(1+\rangeBnd)\}$ has the same probability under $P$ and $Q$, at least $1-\tailShift/(2t)$, and $|\ratio-1|\le\rangeBnd$ on $\truncSet$. Conditioning on this event, we obtain the following analogue of \cref{cor:gauss-poisson}.
\begin{corollary}[Gaussian random allocation]\label{cor:gauss-ra}
There are universal constants $c,C>0$ such that the following holds. Let $(P,Q)$ be the Gaussian pair with $\sigma>0$ and $\chiVar\coloneqq e^{1/\sigma^2}-1$, and let $t\ge2$. If $0<\tailShift\le1$ and $(\rangeBnd/t)\ln(C/\tailShift)\le c$ for $\rangeBnd$ given by \cref{eq:gauss-cutoff}, then $1$-out-of-$t$ random allocation of the Gaussian mechanism \cref{eq:gauss-mech} is $(\lossShift,\tailShift)$-shifted $\sqrt{\chiVar/t}$-GDP, where $\lossShift\coloneqq\frac{C\rangeBnd}{t}\ln(C/\tailShift)$.
\end{corollary}
As for Poisson subsampling, this gives a closed-form condition for the central limit rate.
\begin{corollary}[Closed-form bound for Gaussian random allocation]\label{cor:gauss-ra-closed-form}
There is a universal constant $C>0$ such that the following holds. Let $\sigma\ge1$, $t\ge2$ and $0<\delta\le1/2$. If $\sigma\ln\bigl(1+\sqrt t/\bigl(\sigma\sqrt{C\ln(1/\delta)}\bigr)\bigr)\ge C\sqrt{\ln(t/\delta)}$, then $1$-out-of-$t$ random allocation of the Gaussian mechanism \cref{eq:gauss-mech} is $(\eps,\delta)$-DP for ${\eps=C\sqrt{\ln(1/\delta)}/\bigl(\sigma\sqrt t\bigr)}$.
\end{corollary}
Up to the constants, this is the bound of \cref{clm:gauss-closed-form} for Poisson subsampling at rate $\lambda=1/t$, under a condition of the same form as \cref{eq:gauss-conditions}.
Since \cref{lem:gauss-conditions} proves \cref{eq:gauss-conditions} for every value of its constant, its two sufficient conditions also apply to random allocation.
In contrast, the reduction of random allocation to Poisson subsampling \citep[Corollary~4.5]{FS25ra} needs the assumption on $\sigma$ of the analysis of Poisson subsampling and, in addition, $\sigma\gtrsim\sqrt{\ln(t/\delta)}$.

Since the parameter $\chiVar/t$ is slightly smaller than the parameter $\chiVar/(t-1)$ of \cref{cor:gauss-poisson} for $\lambda=1/t$, inverting the Gaussian profile as in \cref{cor:gauss-leading} gives the same bounds.
\begin{corollary}[Asymptotic upper bounds for random allocation]\label{cor:gauss-leading-ra}
Fix $\sigma>0$ and let $\chiVar\coloneqq e^{1/\sigma^2}-1$. If $\delta=t^{-\theta}$ for a fixed $\theta>1/2$, then, as $t\to\infty$, $1$-out-of-$t$ random allocation of the Gaussian mechanism \cref{eq:gauss-mech} is $(\eps,\delta)$-DP for ${\eps=(1+o(1))\sqrt{(2\theta-1)\,\chiVar\ln(t)/t}}$.
If $\delta=\delta(t)$ satisfies $\ln(1/\delta)=\omega(\ln(t))$ and $\limsup_{t\to\infty}8\ln(1/\delta)/(\sigma\ln(t))^2<1$, then it is $(\eps,\delta)$-DP for ${\eps=(1+o(1))\sqrt{2\chiVar\ln(1/\delta)/t}}$.
\end{corollary}
We prove \cref{cor:gauss-leading-ra} as a special case of its general version, \cref{cor:gauss-leading-ra-gen}, which allows a noise scale $\sigma$ that depends on $t$.
\section{Application: connect-the-dots discretization}\label{sec:ctd}
In this section we use our approximation analysis for composition to explain the properties of connect-the-dots (CtD) discretization. Specifically, we show that while discretizing a single step can increase $\eps$ by almost the grid spacing $\gridSpacing$, after composition the increase per step is only of second order in the spacing (up to the accuracy of the Gaussian concentration). We defer the proofs to \cref{app:ctd-proofs}.

Numerical PLD accountants discretize the PLD of every step on a grid and convolve the resulting PLDs. The simplest conservative approach is to round the loss up to the nearest grid point, which is the approach used in the privacy buckets method \citep{MM18}. The CtD discretization \citep{DGKKM22} instead splits the probability mass of every privacy loss between the two adjacent grid points. It is dominated by every other dominating pair whose PLD is supported on the grid \citep[Lemma~4.1]{DGKKM22}. Experimentally, it keeps a similar error bound with a significantly larger discretization interval than that of privacy buckets \citep[Section~6]{DGKKM22}. \Cref{fig:ctd-gauss} shows the error of both discretizations for the Gaussian mechanism, for a single step and for $100$ steps. CtD is more accurate even for a single step, and after composition its error is close to quadratic in $\gridSpacing$, whereas that of privacy buckets is linear. We demonstrate this effect in \cref{fig:ctd-ratio} by examining the ratio of the two errors for $t=1,5,25,100$ and $1000$ steps, which decreases roughly linearly in $\gridSpacing$. In \cref{fig:ctd-ratio} we omit the point $(t=1000,\gridSpacing=10^{-4})$, where \texttt{dp\_accounting} loses numerical precision.
\begin{figure}[t]
\centering
\includegraphics[width=\linewidth]{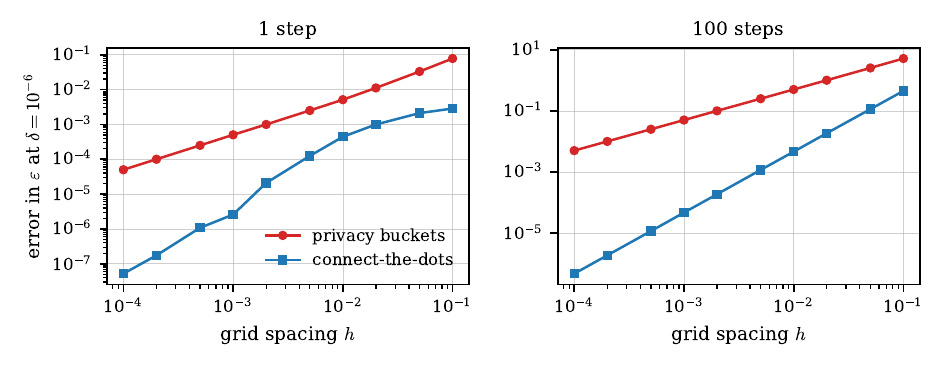}
\caption{The error of the upper bound on $\eps$ at $\delta=10^{-6}$ as a function of the grid spacing $\gridSpacing$, for a single step (left) and for the composition of $100$ steps (right). The privacy loss of the Gaussian mechanism with $\sigma=10$ (distributed as $\Nrm(1/(2\sigma^2),1/\sigma^2)$) is discretized on the grid $\gridSpacing\mathbb{Z}$ by privacy buckets and by CtD. The upper bounds are the pessimistic estimates of \texttt{dp\_accounting} \citep{GoogleDPAccounting} with its default truncation of the tails, and the error is computed relative to the exact value. For small $\gridSpacing$, the error of privacy buckets is close to $t\gridSpacing/2$, whereas that of CtD is of order $\gridSpacing^2$.}
\label{fig:ctd-gauss}
\end{figure}
\begin{figure}[t]
\centering
\includegraphics[width=0.55\linewidth]{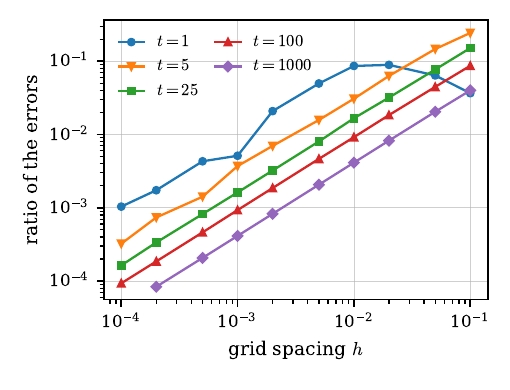}
\caption{The ratio of the error of CtD to that of privacy buckets, in the setting of \cref{fig:ctd-gauss}, as a function of the grid spacing $\gridSpacing$, for the composition of $t$ steps.}
\label{fig:ctd-ratio}
\end{figure}

\para{The discretization.}
Let $P$ and $Q$ be mutually absolutely continuous and $\lossRV\coloneqq\lossRVfunc{P}{Q}$. A \emph{grid} is a set $\Lambda\subset\reals$ of \emph{knots} with finitely many knots in every bounded interval and $\inf\Lambda\le \lossRV\le\sup\Lambda$ almost surely. The grid has \emph{spacing} at most $\gridSpacing$ if adjacent knots differ by at most $\gridSpacing$, as for $\gridSpacing\mathbb{Z}$. We analyze the effect of discretization separately from the truncation of the tails, which in standard implementations is done together with discretization.

\begin{definition}[Connect-the-dots discretization]\label{def:ctd}
The \emph{CtD discretization} of $\lossRV$ on a grid $\Lambda$ is the random variable $\tilde \lossRV$ that is equal to $\lossRV$ if $\lossRV\in\Lambda$, and otherwise is equal to one of the knots $\loss^-<\loss^+$ adjacent to $\lossRV$, with the conditional probabilities given, for $\loss^-<\loss<\loss^+$, by
\begin{equation}\label{eq:ctd-kernel}
 \prob{}{\tilde \lossRV=\loss^-\mid \lossRV=\loss}=\frac{e^{-\loss}-e^{-\loss^+}}{e^{-\loss^-}-e^{-\loss^+}}
 \qquad\text{and}\qquad
 \prob{}{\tilde \lossRV=\loss^+\mid \lossRV=\loss}=\frac{e^{-\loss^-}-e^{-\loss}}{e^{-\loss^-}-e^{-\loss^+}}.
\end{equation}
The \emph{CtD pair} of $(P,Q)$ on $\Lambda$ consists of the distributions $\tilde P(\loss)\coloneqq\prob{}{\tilde \lossRV=\loss}$ and $\tilde Q(\loss)\coloneqq e^{-\loss}\tilde P(\loss)$ over $\Lambda$.
\end{definition}
Note that this discretization preserves the conditional mean of $e^{-\lossRV}$, rather than that of $\lossRV$: $\expect{}{e^{-\tilde \lossRV}\mid \lossRV}=e^{-\lossRV}$.
This identity and $\expect{}{e^{-\lossRV}}=1$ imply that $\tilde Q$ is a probability distribution, that $\tilde P$ and $\tilde Q$ are mutually absolutely continuous, and that the privacy loss of $(\tilde P,\tilde Q)$ is distributed as $\tilde \lossRV$. Thus $\tilde \lossRV$ couples the privacy losses of the two pairs.

We state the results of this section for the privacy loss $\lossRV$ of $(P,Q)$. They also apply to the reversed pairs, since the CtD pair of $(Q,P)$ on the grid $-\Lambda$ is $(\tilde Q,\tilde P)$ with every knot $\loss$ replaced by $-\loss$. Indeed, by \cref{eq:ctd-kernel} for the privacy loss $-\lossFunc{\omega}{P}{Q}$ of $(Q,P)$ (where $\omega\sim Q$), the knot $-\loss$ receives the mass $e^{-\loss}\tilde P(\loss)=\tilde Q(\loss)$. In practice, the two directions are usually discretized separately, as in \texttt{dp\_accounting} \citep{GoogleDPAccounting} and \texttt{PLD\_accounting} \citep{PLDAccounting}.

\para{The resulting error.}
In the worst case, the discretization of a single step increases $\eps$ by almost the grid spacing. For example, for randomized response with privacy loss $\pm\eps_0$ (where $0<\eps_0<\gridSpacing$), the two distributions are $(\eps_0,0)$-close, whereas their CtD discretizations on the grid $\gridSpacing\mathbb{Z}$ are $(\eps,\delta)$-close only for $\eps$ above a threshold that tends to $\gridSpacing$ as $\delta\to0$ (\cref{clm:ctd-rr}). In contrast, the mean and the variance of the privacy loss change only by $O(\gridSpacing^2)$ (\cref{lem:ctd-facts}). Rounding up to the grid (as in privacy buckets) instead increases every privacy loss, and thus also the mean, by up to $\gridSpacing$. The privacy profile of a composition is determined by the mean and the variance of its privacy loss up to the shifts of the Gaussian concentration (\cref{thm:gauss-approx,lem:gc-profile}). Therefore the increase of $\eps$ per step tends to be quadratic in $\gridSpacing$ after composition (\cref{cor:ctd-error}). The bounds on the moments also suggest why CtD is more accurate in practice even for a single step (\cref{fig:ctd-gauss}). However, by themselves they cannot guarantee a smaller error for small $\delta$ and general algorithms, as the example of randomized response shows.
\begin{lemma}[CtD discretization of one step]\label{lem:ctd-facts}
Let $\tilde \lossRV$ and $(\tilde P,\tilde Q)$ be as in \cref{def:ctd} for a grid with spacing at most $\gridSpacing$, and assume that $\lossVar{P}{Q}<\infty$. Then
\begin{enumerate}
 \item $0\le\lossMean{\tilde P}{\tilde Q}-\lossMean{P}{Q}\le\gridSpacing^2/8$;
 \item $-(\gridSpacing^2/8)\sqrt{\lossVar{P}{Q}}\le\lossVar{\tilde P}{\tilde Q}-\lossVar{P}{Q}\le\gridSpacing^2/4+(\gridSpacing^2/8)\sqrt{\lossVar{P}{Q}}$.
\end{enumerate}
\end{lemma}

We now discretize every step of the composition $(P^t,Q^t)$ of \cref{sec:gaussapprox}. Let $(\tilde P,\tilde Q)$ be the CtD pair of $(P,Q)$ on a grid $\Lambda$. We write $\eps(\delta)$ and $\tilde\eps(\delta)$ for the smallest $\eps\ge0$ such that $\HockeyStickDiv{e^{\eps}}{P^t}{Q^t}\le\delta$ and $\HockeyStickDiv{e^{\eps}}{\tilde P^t}{\tilde Q^t}\le\delta$, respectively. We apply \cref{lem:ctd-facts} to $(P,Q)$ and \cref{thm:gauss-approx} to the discretized composition, and obtain the following comparison.
\begin{theorem}[Composition of CtD discretizations]\label{thm:ctd-composition}
Let $C$ be the constant of \cref{thm:gauss-approx} and let $t\in\naturals$ and $\rangeBnd,\gridSpacing>0$. Assume that the grid $\Lambda$ has spacing at most $\gridSpacing$ and that $|\lossRVfunc{P}{Q}-\lossMean{P}{Q}|\le\rangeBnd$ almost surely. Then $(\tilde P^t,\tilde Q^t)$ dominates $(P^t,Q^t)$, and for every $0<\tailShift\le1$, the privacy loss of $(\tilde P^t,\tilde Q^t)$ is $(C(\rangeBnd+2\gridSpacing)\ln(2/\tailShift),\tailShift)$-approximately $(t\cdot\lossMean{\tilde P}{\tilde Q},t\cdot\lossVar{\tilde P}{\tilde Q})$-Gaussian-concentrated.
\end{theorem}
We bound the cost of replacing the parameters $(t\cdot\lossMean{\tilde P}{\tilde Q},t\cdot\lossVar{\tilde P}{\tilde Q})$ of the Gaussian random variable by $(t\cdot\lossMean{P}{Q},t\cdot\lossVar{P}{Q})$ using \cref{eq:param-change}, and obtain the following bound on the error caused by the discretization.
\begin{corollary}[Error of the discretization]\label{cor:ctd-error}
Assume the conditions of \cref{thm:ctd-composition} and that $\lossVar{P}{Q}>0$. Then for every $0<\delta\le1$ and $0<\tailShift\le1$ with $3\tailShift<\delta$,
\[
 \eps(\delta)\le\tilde\eps(\delta)\le\eps(\delta-3\tailShift)+2C(\rangeBnd+\gridSpacing)\ln(2/\tailShift)+\frac{t\gridSpacing^2}8+\frac{\gridSpacing^2\sqrt t}8\left(\frac1{\sqrt{\lossVar{P}{Q}}}+1\right)\sqrt{2\ln(1/\tailShift)} .
\]
\end{corollary}
The term $2C(\rangeBnd+\gridSpacing)\ln(2/\tailShift)$ is the error of the two Gaussian concentrations. The term $t\gridSpacing^2/8$ bounds the increase of the mean, which adds to the bias of the privacy profile itself rather than to the error of the approximation. The last term bounds the effect of the change of the variance, and it is $O\left(\gridSpacing^2\sqrt{t\ln(1/\tailShift)/\lossVar{P}{Q}}\right)$ whenever $\lossVar{P}{Q}\le1$, which holds for $\rangeBnd\le1$ since $\lossVar{P}{Q}\le\rangeBnd^2$. This explains the finding in \citep{DGKKM22} for bounded privacy losses.

For comparison, we note that the error of a mean-preserving discretization can be bounded by coupling the original and the discretized privacy losses and applying Hoeffding's inequality \citep{Hoeffding63}, as in \citep[Lemma~5.3]{GLW21}. This argument also applies to CtD: for independent copies $\lossRV_1,\ldots,\lossRV_t$ of $\lossRVfunc{P}{Q}$ and their CtD discretizations $\tilde\lossRV_1,\ldots,\tilde\lossRV_t$ on $\Lambda$, drawn independently, the corrections $\tilde \lossRV_i-\lossRV_i$ are independent given $\lossRV_1,\ldots,\lossRV_t$, with means at most $\gridSpacing^2/8$ and ranges of length at most $\gridSpacing$ (see the proof of \cref{lem:ctd-facts}). Thus the privacy loss of $(P^t,Q^t)$ approximately stochastically dominates that of $(\tilde P^t,\tilde Q^t)$ with tail shift $\tailShift$ and loss shift $t\gridSpacing^2/8+\gridSpacing\sqrt{t\ln(1/\tailShift)/2}$. This loss shift is of order $\gridSpacing\sqrt{t\ln(1/\tailShift)}$ when $t\gridSpacing^2\ll1$. The last term of \cref{cor:ctd-error} is $O\left(t\gridSpacing^2\sqrt{\ln(1/\tailShift)}\right)$ when $t$ is of order $1/\lossVar{P}{Q}$. In this case our bound on the error caused by the discretization is quadratically better, up to logarithmic factors.

\subsection*{AI Usage}
The authors formulated the questions and the key approach to the problem. The development of these ideas and writing were performed through an iterative process in which AI (Claude Opus 5.5 and ChatGPT Astra 6) found the relevant technical results, generated details of proofs and wrote parts of the prose based on author-provided guidance and structure. The resulting text was also extensively edited and proofread by the authors. The implementation of the evaluation code was written by AI under close supervision and auditing by the authors. The authors take full responsibility for the content of this work.

\subsection*{Acknowledgments}
Shenfeld's work was supported in part by the Apple Scholars in AI/ML PhD Fellowship, ERC grant 101125913, Simons Foundation Collaboration 733792, and Israel Science Foundation (ISF) grant 2861/20. Views and opinions expressed are however those of the authors only and do not necessarily reflect those of the European Union or the European Research Council Executive Agency. Neither the European Union nor the granting authority can be held responsible for them.

\appendix
\crefalias{section}{appendix}
\crefalias{subsection}{appendix}

\section{Comparison with existing approaches}\label{sec:comparison}
We compare existing analyses of the Gaussian mechanism at sampling rate $\lambda=1/t$ (which is also the regime of random allocation) by the noise scale they need and by how their leading term relates to the central limit rate $\sqrt{\ln(1/\delta)}/\left(\sigma\sqrt t\right)$. \Cref{tab:comparison} summarizes the comparison.

\para{R\'enyi divergence.}
For the $\lambda$-subsampled Gaussian mechanism, the moments accountant \citep{ACGMMT16} shows that the R\'enyi divergence of integer order $\renyiOrder\le\sigma^2\ln(1/(\lambda\sigma))$ is $O(\lambda^2\renyiOrder/\sigma^2)$. This quadratic behavior fails for orders much larger than $\sigma^2\ln(1/\lambda)$. Indeed, by the exact binomial expression \citep{MTZ19}, the divergence is at least $\renyiOrder/(2\sigma^2)-\frac{\renyiOrder}{\renyiOrder-1}\ln(1/\lambda)$ (the contribution of the term in which the element is sampled in all $\renyiOrder$ draws; see also \citep[Section~6.6]{Steinke25}). At the same time, a quadratic bound for $t$ steps is converted to $(\eps,\delta)$-DP at the central limit rate using the order $\renyiOrder\asymp\sigma\sqrt{\ln(1/\delta)/t}/\lambda$. The two constraints are compatible only if $\lambda\sigma\ln(1/(\lambda\sigma))\gtrsim\sqrt{\ln(1/\delta)/t}$. At $\lambda=1/t$ this condition reads $\sigma\gtrsim\sqrt{t\ln(1/\delta)}/\ln(t/\ln(1/\delta))$ up to lower-order terms.

\para{Amplification and advanced composition.}
A black-box analysis starts from the $(\eps_0,\delta_0)$-DP guarantee of the Gaussian mechanism with $\eps_0=1/(2\sigma^2)+\sqrt{2\ln(1/\delta_0)}/\sigma$. It amplifies this guarantee by subsampling to $(\ln(1+\lambda(e^{\eps_0}-1)),\lambda\delta_0)$-DP and composes $t$ steps by advanced composition \citep{DRV10}. At $\lambda=1/t$ the $t$ steps fail with total probability $\delta_0$, which forces $\delta_0\lesssim\delta$. The amplified $\eps$ is of order $\eps_0/t$ only when $\eps_0=O(1)$, that is, when $\sigma\gtrsim\sqrt{\ln(1/\delta)}$. Advanced composition then gives $\eps\lesssim\ln(1/\delta)/\left(\sigma\sqrt t\right)$. The condition on $\sigma$ is mild, but the bound is a factor $\sqrt{\ln(1/\delta)}$ above the central limit rate.

\para{Central limit theorems.}
\citet[Section~3.1 and Theorem~6]{BDLS20} show that $t$ steps of the Poisson-subsampled Gaussian mechanism tend to $\lambda\sqrt{t\cdot\chiVar}$-GDP when $\lambda\sqrt t$ converges to a positive constant (here $\chiVar=e^{1/\sigma^2}-1$ is the $\chi^2$ divergence of the Gaussian pair). Through the Gaussian tail bound this limit yields $\eps\approx t\lambda^2\chiVar/2+\lambda\sqrt{2t\cdot\chiVar\ln(1/\delta)}$, the form of the bound of \cref{thm:subsampled}. It is a limit statement, however, and its hypotheses exclude $\lambda=1/t$, for which $\lambda\sqrt t\to0$.

For every finite $t$, there is a Berry--Esseen bound \citep{DRS22}. For batches of a fixed size $\lambda n$ out of $n$ examples, the trade-off function of $t$ steps of the Gaussian mechanism is at least the GDP trade-off function moved to the left and down by $\tailShift_{\mathrm{BE}}$ and truncated at $0$ \citep[Corollary~5.6]{DRS22}. Here $\tailShift_{\mathrm{BE}}$ equals $0.56/\sqrt t$ times a normalized third moment of the privacy loss of one step \citep[Lemma~5.5]{DRS22}. At type I error $0$ this lower bound is at most $1-\tailShift_{\mathrm{BE}}$, whereas the trade-off function of $(\eps,\delta)$-DP equals $1-\delta$ there. Therefore the guarantee yields $(\eps,\delta)$-DP only for $\delta\ge\tailShift_{\mathrm{BE}}$. For $\lambda=1/t$ and fixed $\sigma$, the error $\tailShift_{\mathrm{BE}}$ is of order $1/\sqrt t$.

Edgeworth expansions \citep{ZDLS20,WGZSSW26} and the saddle-point accountant \citep{AGACKS23} approximate the privacy profile more accurately. The first of these expansions applies only to type I errors bounded away from $0$ and $1$ \citep[Section~3.2]{ZDLS20}, while small values of $\delta$ depend on type I errors near $0$. The finite-sample bound on the error of the second expansion is additive and uniform in the argument of the distribution function, of order $1/\sqrt t$ with constants given by normalized moments of the privacy losses \citep[Lemma~4.3]{WGZSSW26}. This bound improves to order $1/t$ under a decay condition on the characteristic function and up to an unspecified lower-order remainder \citep[Remark~4.4 and Corollary~E.1]{WGZSSW26}. As for the Berry--Esseen bound, such an error gives $(\eps,\delta)$-DP only for $\delta$ above it, and we have not specialized these bounds to $\lambda=1/t$.

\citet[Theorem~5.6]{AGACKS23} prove their finite-composition bound by applying a normal approximation under an exponentially tilted distribution and bounding the resulting error with a large-deviation factor. A similar technical ingredient is used to prove \cref{lem:coupling} \citep[Theorem~1]{Zaitsev25}. The main difference is that we derive a uniform comparison to a Gaussian random variable in a simple form (\cref{def:gauss-conc}). The saddle-point bounds are evaluated numerically. The accompanying asymptotic composition theorem lets $\delta$ depend on the number of steps \citep[Theorem~4.1]{AGACKS23}, but it assumes that the standard deviation of the total privacy loss is much larger than $\sqrt{\ln(1/\delta)}$. In the regime of the central limit rate $\eps\lesssim\ln(1/\delta)$, this assumption fails for the subsampled Gaussian mechanism. Indeed, by \cref{lem:chi2-moments,lem:subsampling-chi2}, its total privacy loss has variance at most $t\lambda^2\chiVar/(1-\lambda)$ and therefore standard deviation at most of order $\eps/\sqrt{\ln(1/\delta)}$.

\begin{table}[t]
\centering
\small
\setlength{\tabcolsep}{4pt}
\begin{tabular}{@{}>{\raggedright\arraybackslash}p{4.9cm}l>{\raggedright\arraybackslash}p{4.3cm}l@{}}
\toprule
Method & Scheme & Sufficient noise scale & Leading term \\
\midrule
Quadratic R\'enyi bounds \citep{ACGMMT16,MTZ19} & Poisson & $\sigma\gtrsim\sqrt{t\ln(1/\delta)}/\ln(t/\ln(1/\delta))$ & $\sqrt{\ln(1/\delta)}/\left(\sigma\sqrt t\right)$ \\[2pt]
$(\eps_0,\delta_0)$ amplification with advanced composition & Poisson & $\sigma\gtrsim\sqrt{\ln(1/\delta)}$ & $\ln(1/\delta)/\left(\sigma\sqrt t\right)$ \\[2pt]
Reduction to Poisson subsampling \citep{FS25ra} & allocation & $\sigma\gtrsim\sqrt{\ln(t/\delta)}$ and the condition of the Poisson analysis & $\sqrt{\ln(t/\delta)}/\left(\sigma\sqrt t\right)$ \\[2pt]
This work (\cref{clm:gauss-closed-form,cor:gauss-ra-closed-form}) & both & $\sigma\ln\left(1+\frac{\sqrt t}{\sigma\sqrt{C\ln(1/\delta)}}\right)\gtrsim\sqrt{\ln(t/\delta)}$ & $\sqrt{\ln(1/\delta)}/\left(\sigma\sqrt t\right)$ \\
\bottomrule
\end{tabular}
\caption{Noise scales at which each analysis gives its leading privacy term for the Gaussian mechanism with Poisson subsampling at rate $1/t$ or with $1$-out-of-$t$ random allocation (lower-order terms omitted). For polynomially small $\delta$, the condition of this work holds for every fixed $\sigma\ge1$ once $t$ is sufficiently large.}
\label{tab:comparison}
\end{table}
\section{Comparison with other privacy notions}\label{app:notions}

The definitions below apply to the output distributions $P$ and $Q$ on every pair of adjacent datasets, in both directions. For $\renyiOrder>1$, we write $D_\renyiOrder(P\|Q)\coloneqq\frac1{\renyiOrder-1}\ln\left(\expect{}{e^{(\renyiOrder-1)\lossRVfunc{P}{Q}}}\right)$ for the R\'enyi divergence of order $\renyiOrder$.

\para{Zero-concentrated differential privacy.}
An algorithm is $(\xi,\rho)$-zCDP if $D_\renyiOrder(P\|Q)\le\xi+\rho\renyiOrder$ for every $\renyiOrder>1$, and it is $\rho$-zCDP if it is $(0,\rho)$-zCDP \citep[Definition~1.1]{BS16}. For a Gaussian privacy loss $\Nrm(\mu,V)$, equality holds for $(\xi,\rho)=(\mu-V/2,V/2)$. Every $\sqrt V$-GDP algorithm is $V/2$-zCDP \citep[Lemmas~1.8 and~2.4]{BS16}. Shifted GDP does not imply zCDP when $\tailShift>0$, since the R\'enyi divergences can be infinite. Subsampling does not amplify zCDP for the Gaussian mechanism \citep[Section~6.6]{Steinke25}.

\para{Truncated concentrated differential privacy.}
For $\rho>0$ and $\tcdpOrder>1$, an algorithm is $(\rho,\tcdpOrder)$-tCDP if $D_\renyiOrder(P\|Q)\le\rho\renyiOrder$ for every $\renyiOrder\in(1,\tcdpOrder)$ \citep[Definition~1]{BDRS18}. For a Gaussian privacy loss $\Nrm(\mu,V)$ with $\mu\ge V/2$, this holds exactly when $\rho\ge\mu$. \citet[Theorem~13]{BDRS18} prove privacy amplification by uniform subsampling without replacement for tCDP. We discuss the limitations of R\'enyi bounds for Poisson subsampling in \cref{sec:comparison}.

\para{Concentrated differential privacy.}
An algorithm is $(\mu,\tau)$-CDP if its privacy loss $\lossRV\coloneqq\lossRVfunc{P}{Q}$ has mean at most $\mu$ and $\lossRV-\expect{}{\lossRV}$ is $\tau$-subgaussian \citep[Definitions~2.5, 3.2 and~3.3]{DR16}. For a Gaussian privacy loss $\Nrm(\mu,V)$, both bounds hold with equality for $\left(\mu,\sqrt V\right)$. Subsampling does not reduce $\tau$ for the Gaussian mechanism \citep[Section~6.6]{Steinke25}, \citep[Lemma~4.2]{BS16}. Unlike zCDP and shifted GDP, CDP is not closed under post-processing \citep[Proposition~A.1]{BS16}.

\section{Proofs for \texorpdfstring{\cref{sec:prelim,sec:gaussapprox}}{Sections 2 and 3}}\label{app:gaussapprox-proofs}
We use the notation of \cref{sec:gaussapprox}. We first prove the composition of approximate domination (\cref{lem:approx-dom-comp}) and bound the privacy profile of GDP, which we use for the leading constants in \cref{cor:gauss-leading}. We then prove \cref{thm:gauss-approx} from a Gaussian coupling of sums of bounded independent random variables. Finally, we bound the cost of changing the parameters of a Gaussian random variable, which we also use in later sections.

\begin{lemma}[Composition of approximate domination]\label{lem:approx-dom-comp}
Given $\lossShift_1,\lossShift_2\in\reals$; $\tailShift_1,\tailShift_2\ge0$; datasets $\dataset$ and $\dataset'$; pairs $(P_1,Q_1)$ and $(P_2,Q_2)$; and algorithms $\alg_1$ and $\alg_2$, where $\alg_2$ also receives the output $\out_1$ of $\alg_1$, if $(P_1,Q_1)$ $(\lossShift_1,\tailShift_1)$-approximately dominates $(\alg_1(\dataset),\alg_1(\dataset'))$ and $(P_2,Q_2)$ $(\lossShift_2,\tailShift_2)$-approximately dominates $(\alg_2(\dataset;\out_1),\alg_2(\dataset';\out_1))$ for every $\out_1$, then $(P_1\times P_2,Q_1\times Q_2)$ $(\lossShift_1+\lossShift_2,\tailShift_1+\tailShift_2)$-approximately dominates the pair of output distributions of the adaptive composition of $\alg_1$ and $\alg_2$ on $\dataset$ and $\dataset'$.
\end{lemma}
\begin{proof}
Let $A$ and $B$ be distributions over $\Omega_1$ and, for every $\omega_1\in\Omega_1$, let $A_{\omega_1}$ and $B_{\omega_1}$ be distributions over $\Omega_2$.
Let $\bar A$ be the distribution of $(\omega_1,\omega_2)$ for $\omega_1\sim A$ and $\omega_2\sim A_{\omega_1}$, and define $\bar B$ similarly.
We first show that for every $\kappa\ge0$,
\begin{equation}\label{eq:hs-chain}
 \HockeyStickDiv{\kappa}{\bar A}{\bar B}=\expect{\omega_1\sim A}{\HockeyStickDiv{\kappa e^{-\lossFunc{\omega_1}{A}{B}}}{A_{\omega_1}}{B_{\omega_1}}},
\end{equation}
where $\kappa e^{-\lossFunc{\omega_1}{A}{B}}=0$ when $B(\omega_1)=0$.
Indeed, if $A(\omega_1)>0$, then $[A(\omega_1)A_{\omega_1}(\omega_2)-\kappa B(\omega_1)B_{\omega_1}(\omega_2)]_+=A(\omega_1)[A_{\omega_1}(\omega_2)-\kappa e^{-\lossFunc{\omega_1}{A}{B}}B_{\omega_1}(\omega_2)]_+$, and if $A(\omega_1)=0$, then the left-hand side is $0$.
Integrating over $\omega_2$ and then over $\omega_1$ gives \cref{eq:hs-chain}.
Also, since $\HockeyStickDiv{0}{P}{Q}=1$ for every pair, if $(P,Q)$ $(\lossShift,\tailShift)$-approximately dominates $(P',Q')$, then $\HockeyStickDiv{\kappa}{P'}{Q'}\le\HockeyStickDiv{\kappa e^{-\lossShift}}{P}{Q}+\tailShift$ for every $\kappa\ge0$.

Denote by $P'_1$ and $Q'_1$ the output distributions of $\alg_1$ on $\dataset$ and $\dataset'$, by $P'_{\out_1}$ and $Q'_{\out_1}$ those of $\alg_2(\cdot\,;\out_1)$, and by $P'$ and $Q'$ those of the adaptive composition, and let $\kappa\ge0$.
By \cref{eq:hs-chain}, the assumption on $\alg_2$, and \cref{eq:hs-chain} for $P'_1\times P_2$ and $Q'_1\times Q_2$,
\begin{align*}
 \HockeyStickDiv{\kappa}{P'}{Q'}
 &=\expect{\out_1\sim P'_1}{\HockeyStickDiv{\kappa e^{-\lossFunc{\out_1}{P'_1}{Q'_1}}}{P'_{\out_1}}{Q'_{\out_1}}}\\
 &\le\expect{\out_1\sim P'_1}{\HockeyStickDiv{\kappa e^{-\lossShift_2}e^{-\lossFunc{\out_1}{P'_1}{Q'_1}}}{P_2}{Q_2}}+\tailShift_2\\
 &=\HockeyStickDiv{\kappa e^{-\lossShift_2}}{P'_1\times P_2}{Q'_1\times Q_2}+\tailShift_2 .
\end{align*}
By \cref{eq:hs-chain} with the two coordinates exchanged, the assumption on $\alg_1$, and \cref{eq:hs-chain} for $P_1\times P_2$ and $Q_1\times Q_2$,
\begin{align*}
 \HockeyStickDiv{\kappa e^{-\lossShift_2}}{P'_1\times P_2}{Q'_1\times Q_2}
 &=\expect{\omega_2\sim P_2}{\HockeyStickDiv{\kappa e^{-\lossShift_2}e^{-\lossFunc{\omega_2}{P_2}{Q_2}}}{P'_1}{Q'_1}}\\
 &\le\expect{\omega_2\sim P_2}{\HockeyStickDiv{\kappa e^{-\lossShift_1-\lossShift_2}e^{-\lossFunc{\omega_2}{P_2}{Q_2}}}{P_1}{Q_1}}+\tailShift_1\\
 &=\HockeyStickDiv{\kappa e^{-\lossShift_1-\lossShift_2}}{P_1\times P_2}{Q_1\times Q_2}+\tailShift_1 .
\end{align*}
Combining the two bounds with $\kappa=e^{\eps}$ proves the lemma.
\end{proof}

\begin{lemma}[Privacy profile of GDP]\label{lem:gdp-small-var}
Let $V>0$, $\delta\in(0,1)$ and $k\ge1$ be such that $k\ge2\ln(1/\delta)-\ln(1/V)$. Then $\eps_{\mathrm G}(V,\delta)\le V/2+\sqrt{kV}$.
\end{lemma}
\begin{proof}
The privacy loss of $\mathcal G_V$ is distributed as $V/2+\sqrt V\,Z$ for $Z\sim\Nrm(0,1)$. Letting $\eps\coloneqq V/2+\sqrt{kV}$ and substituting $s=u-\sqrt k$ in \cref{def:hs}, we use $1-e^{-\sqrt V\,s}\le\sqrt V\,s$ and $e^{-(\sqrt k+s)^2/2}\le e^{-k/2-\sqrt k\,s}$ to obtain
\[
 g_V(\eps)=\int_{\sqrt k}^\infty\bigl(1-e^{-\sqrt V(u-\sqrt k)}\bigr)\frac{e^{-u^2/2}}{\sqrt{2\pi}}\,du
 \le\frac{\sqrt V}{\sqrt{2\pi}}\,e^{-k/2}\int_0^\infty s\,e^{-\sqrt k\,s}\,ds
 =\frac{\sqrt V}{\sqrt{2\pi}\,k}\,e^{-k/2} ,
\]
which is at most $\sqrt V\,e^{-k/2}\le\delta$ since $k\ge1$ and $k\ge2\ln(1/\delta)-\ln(1/V)$.
\end{proof}
We derive \cref{lem:coupling} from the following bound on an exponential moment of the error of a Gaussian coupling.
\begin{theorem}[{\citep[Theorem~1 and Corollary~1]{Sakhanenko96}}]\label{thm:sakhanenko}
There is a universal constant $C_0>0$ such that the following holds. Let $t\in\naturals$ and $\rangeBnd>0$, let $X_1,\ldots,X_t$ be independent mean-zero random variables with $|X_i|\le\rangeBnd$ almost surely, and let $S\coloneqq\sum_{i=1}^tX_i$ with $V_S\coloneqq\Var(S)>0$. Then $S$ can be coupled with $G\sim\Nrm(0,V_S)$ so that $\expect{}{|S-G|\exp\left(|S-G|/(512\rangeBnd)\right)}\le C_0\rangeBnd$.
\end{theorem}
Indeed, \citet[Theorem~1]{Sakhanenko96} bounds this exponential moment for the normalized summands $X_i/\sqrt{V_S}$, with $\alpha=1$ and the parameter $\lambda=\sqrt{V_S}/(64\rangeBnd)$ there. This choice of $\lambda$ is admissible by \citep[Corollary~1]{Sakhanenko96}, since $|X_i|/\sqrt{V_S}\le\rangeBnd/\sqrt{V_S}$. The right-hand side of the resulting bound is at most $(300e^{1/32}+1536)\rangeBnd/\sqrt{V_S}$. Thus the claim holds with $C_0=1846$. \citet[Theorem~1]{Zaitsev25} gives the same result up to the constants, with a proof based on exponential tilting.
\begin{proofof}{Proof of \cref{lem:coupling} (Gaussian coupling of bounded summands)}
We may assume that $V_S\coloneqq\Var(S)>0$, since otherwise $S=0$ almost surely and we can take $G=0$. Let $G$ and $C_0$ be the Gaussian random variable and the constant of \cref{thm:sakhanenko}, and let $C\coloneqq\max\{C_0,512\}$. For $0<\tailShift\le1/2$ and $u\coloneqq\ln(2/\tailShift)>1$, Markov's inequality for the nondecreasing function $x\mapsto xe^{x/(512\rangeBnd)}$ of $|S-G|$ gives that $\prob{}{|S-G|>C\rangeBnd u}\le\frac{C_0}{Cu}\,e^{-Cu/512}\le e^{-u}=\tailShift/2$, where we used that $C\ge C_0$, $C\ge512$ and $u>1$.
\end{proofof}

\begin{proofof}{Proof of \cref{thm:gauss-approx} (Gaussian concentration for compositions)}
We prove the claim with $C\coloneqq2C_1$, where $C_1$ is the constant of \cref{lem:coupling}. Let $\omega\sim P^t$ and $\lossRV\coloneqq\lossFunc{\omega}{P^t}{Q^t}=\sum_i\lossFunc{\omega_i}{P}{Q}$ (the privacy loss of a product pair is the sum of the privacy losses of its coordinates), and let $\lossRV'\coloneqq\sum_i\lossFunc{\omega_i}{P}{Q}\ind\{\omega_i\in \truncSet\}$ be the truncated privacy loss. By a union bound over the events $\omega_i\notin \truncSet$, we have that $\lossRV'=\lossRV$ except with probability at most $\tailShift/2$.

The centered summands $\lossFunc{\omega_i}{P}{Q}\ind\{\omega_i\in \truncSet\}-\truncMean{P}{Q}$ of $\lossRV'-t\cdot\truncMean{P}{Q}$ are independent and at most $\rangeBnd$ in absolute value almost surely. By \cref{lem:coupling} with $\tailShift/2$ in place of $\tailShift$, we can couple $\lossRV'-t\cdot\truncMean{P}{Q}$ with $G_0\sim\Nrm(0,t\cdot\truncVar{P}{Q})$ so that $|\lossRV'-t\cdot\truncMean{P}{Q}-G_0|\le C_1\rangeBnd\ln(4/\tailShift)$ except with probability $\tailShift/2$ (given $\omega$, we sample $G_0$ from the conditional distribution of $G_0$ given $\lossRV'$ in this coupling). Thus $G\coloneqq t\cdot\truncMean{P}{Q}+G_0\sim\Nrm(t\cdot\truncMean{P}{Q},t\cdot\truncVar{P}{Q})$ satisfies $|\lossRV'-G|\le C_1\rangeBnd\ln(4/\tailShift)$ except with probability $\tailShift/2$. Since $\ln(4/\tailShift)\le2\ln(2/\tailShift)$ for $\tailShift\le1$, we get that $|\lossRV-G|\le C\rangeBnd\ln(2/\tailShift)$ except with probability $\tailShift$. As noted after \cref{def:gauss-conc}, this coupling gives the claim.
\end{proofof}

Changing the parameters of a Gaussian random variable costs a loss shift. Let $\mu,\mu_0\in\reals$, $V,V_0\ge0$ and $0<\tailShift\le1$. Then random variables with distributions $\Nrm(\mu_0,V_0)$ and $\Nrm(\mu,V)$ each $(\lossShift,\tailShift)$-approximately stochastically dominate the other for $\lossShift$ given by \cref{eq:param-change} below. Indeed, for $G\sim\Nrm(0,1)$ the random variables $\mu+\sqrt V\,G$ and $\mu_0+\sqrt{V_0}\,G$ are distributed as $\Nrm(\mu,V)$ and $\Nrm(\mu_0,V_0)$. By the Gaussian tail bound, $|G|\le\sqrt{2\ln(1/\tailShift)}$ except with probability $\tailShift$, and on this event the two random variables differ by at most
\begin{equation}\label{eq:param-change}
 |\mu-\mu_0|+\left|\sqrt V-\sqrt{V_0}\right|\sqrt{2\ln(1/\tailShift)} .
\end{equation}
As noted after \cref{def:gauss-conc}, this coupling gives the claim.

\section{Proofs for \texorpdfstring{\cref{sec:chi2-subsampling}}{Section 4}}\label{app:bernstein-proofs}
In this appendix $\ratio$ is the likelihood ratio of $(P,Q)$, and $\ratio_\lambda\coloneqq1-\lambda+\lambda \ratio$ is the likelihood ratio of $(P_\lambda,Q)$.
\subsection{The \texorpdfstring{$\chi^2$}{chi-square} divergence as a proxy for the variance}\label{app:chi2-tightness}
We first prove \cref{lem:chi2-moments,lem:subsampling-chi2}. We then show that the bound of \cref{lem:chi2-moments} on the variance is tight up to lower-order terms, and use this to prove \cref{lem:chi2-tight-sub}.
\begin{proofof}{Proof of \cref{lem:chi2-moments} (Mean and variance are upper bounded by $\chi^2$)}
By Jensen's inequality,
\begin{align*}
 \lossMean{P}{Q}=\expect{\omega\sim P}{\ln(\ratioFunc{\omega}{P}{Q})}&\le\ln\left(\expect{\omega\sim P}{\ratioFunc{\omega}{P}{Q}}\right)\\
 &=\ln\left(\expect{\omega\sim Q}{\ratioFunc{\omega}{P}{Q}^2}\right)=\ln(1+\chisq{P}{Q})\le \chisq{P}{Q} ,
\end{align*}
and similarly $\expect{}{-\lossRVfunc{P}{Q}}\le\ln\left(\expect{\omega\sim P}{1/\ratioFunc{\omega}{P}{Q}}\right)\le0$. For the second moment we use that $s\ln^2(s)\le(s-1)^2$ for $s\ge0$, which for $s=e^{2a}$ is the inequality $|a|\le|\sinh a|$. Thus
\[
 \expect{}{\lossRVfunc{P}{Q}^2}=\expect{\omega\sim Q}{\ratioFunc{\omega}{P}{Q}\ln^2(\ratioFunc{\omega}{P}{Q})}\le\expect{\omega\sim Q}{(\ratioFunc{\omega}{P}{Q}-1)^2}=\chisq{P}{Q} .
\]
The bound on $\truncVar{P}{Q}$ follows since the variance of the truncated privacy loss is at most its second moment, which is at most $\expect{}{\lossRVfunc{P}{Q}^2}$.
\end{proofof}

\begin{proofof}{Proof of \cref{lem:subsampling-chi2} (Subsampling and the $\chi^2$ divergence)}
By the definition $P_\lambda=(1-\lambda)Q+\lambda P$, for $Q$-almost every $\omega$ we have $\ratio_\lambda(\omega)=\bigl((1-\lambda)Q(\omega)+\lambda P(\omega)\bigr)/Q(\omega)=1-\lambda+\lambda\ratio(\omega)$, so $\ratio_\lambda(\omega)-1=\lambda(\ratio(\omega)-1)$.
The definition of the $\chi^2$ divergence now gives $\chisq{P_\lambda}{Q}=\expect{\omega\sim Q}{(\ratio_\lambda(\omega)-1)^2}=\lambda^2\expect{\omega\sim Q}{(\ratio(\omega)-1)^2}=\lambda^2\chisq{P}{Q}$.

For the reverse direction, $\ratio_\lambda(\omega)\ge1-\lambda>0$, so $Q$ and $P_\lambda$ are mutually absolutely continuous and the likelihood ratio of $(Q,P_\lambda)$ is $1/\ratio_\lambda(\omega)$.
Changing the sampling distribution from $P_\lambda$ to $Q$ introduces the factor $\ratio_\lambda(\omega)$, giving
\begin{align*}
 \chisq{Q}{P_\lambda}
 &=\expect{\omega\sim P_\lambda}{\left(\frac{1}{\ratio_\lambda(\omega)}-1\right)^2}\\
 &=\expect{\omega\sim Q}{\ratio_\lambda(\omega)\left(\frac{1}{\ratio_\lambda(\omega)}-1\right)^2}\\
 &=\expect{\omega\sim Q}{\frac{(\ratio_\lambda(\omega)-1)^2}{\ratio_\lambda(\omega)}}
 =\lambda^2\expect{\omega\sim Q}{\frac{(\ratio(\omega)-1)^2}{\ratio_\lambda(\omega)}}.
\end{align*}
Finally, using $\ratio_\lambda(\omega)\ge1-\lambda$ and the nonnegativity of the numerator, we obtain $\chisq{Q}{P_\lambda}\le\frac{\lambda^2}{1-\lambda}\expect{\omega\sim Q}{(\ratio(\omega)-1)^2}=\frac{\lambda^2}{1-\lambda}\chisq{P}{Q}$.
\end{proofof}

We bound the gap between the variance and its upper bound in \cref{lem:chi2-moments} by a fourth moment of the likelihood ratio and, for a bounded privacy loss, by its range.
\begin{lemma}[Variance is lower bounded by $\chi^2$ up to a fourth moment]\label{lem:chi2-tight}
Let $P$ and $Q$ be mutually absolutely continuous with $\chisq{P}{Q}<\infty$, and let $\ratio\coloneqq\ratioFunc{\cdot}{P}{Q}$. If $\lossRVfunc{Q}{P}\le\ln2$ almost surely (that is, $\ratio\ge1/2$) and $m_4\coloneqq\expect{\omega\sim Q}{(\ratio(\omega)-1)^4}<\infty$, then $\lossVar{P}{Q}\ge\chisq{P}{Q}-2m_4$ and $\lossVar{Q}{P}\ge\chisq{Q}{P}-5m_4$.
Furthermore, if $|\lossRVfunc{P}{Q}|\le a$ almost surely for some $a\le1$, then $\lossVar{P}{Q}\ge(1-4a^2)\,\chisq{P}{Q}$ and $\lossVar{Q}{P}\ge(1-4a^2)\,\chisq{Q}{P}$.
\end{lemma}
\begin{proof}
Using the power series of $4\sinh^2(x/2)=2(\cosh x-1)$ and the fact that $2k(2k-1)\ge12$ for $k\ge2$, we have that, for every $x\in\reals$,
\[
 4\sinh^2(x/2)-x^2=\sum_{k\ge2}\frac{2x^{2k}}{(2k)!}\le\sum_{k\ge2}\frac{2x^{2k}}{12\,(2k-2)!}=\frac{x^2}{12}\cdot4\sinh^2(x/2) .
\]
For $s=e^x$ we have that $4\sinh^2(x/2)=(s-1)^2/s$, and therefore, for every $s>0$,
\[
 (s-1)^2-s\ln^2(s)\le\frac{(s-1)^2\ln^2(s)}{12}
 \qquad\text{and}\qquad
 \frac{(s-1)^2}s-\ln^2(s)\le\frac{(s-1)^2\ln^2(s)}{12s} ,
\]
where the second inequality is the first divided by $s$. We take the expectations over $s=\ratio(\omega)$ for $\omega\sim Q$. The expectations of $(s-1)^2$, $s\ln^2(s)$, $(s-1)^2/s$ and $\ln^2(s)$ are $\chisq{P}{Q}$, $\expect{}{\lossRVfunc{P}{Q}^2}$, $\chisq{Q}{P}$ and $\expect{}{\lossRVfunc{Q}{P}^2}$, respectively. If $\ratio\ge1/2$, then $\ln^2(s)\le(s-1)^2/s\le2(s-1)^2$, where the first inequality is the inequality $s\ln^2(s)\le(s-1)^2$ from the proof of \cref{lem:chi2-moments}. The right-hand sides above are then at most $(s-1)^4/6$ and $(s-1)^4/3$, and therefore $\chisq{P}{Q}-\expect{}{\lossRVfunc{P}{Q}^2}\le m_4/6$ and $\chisq{Q}{P}-\expect{}{\lossRVfunc{Q}{P}^2}\le m_4/3$.
By \cref{lem:chi2-moments}, the means satisfy $0\le\lossMean{P}{Q}\le\chisq{P}{Q}$ and $0\le\lossMean{Q}{P}\le\chisq{Q}{P}$. By Jensen's inequality, $\left(\chisq{P}{Q}\right)^2\le m_4$ and $\left(\chisq{Q}{P}\right)^2\le\expect{\omega\sim Q}{(\ratio(\omega)-1)^4/\ratio(\omega)^2}\le4m_4$. As a variance is the second moment minus the squared mean, this gives the first claim. For the second claim, $|\ln\ratio|\le a$ also $Q$-almost surely, since $P$ and $Q$ are mutually absolutely continuous. Therefore the expectations of the right-hand sides of the two inequalities are at most $a^2\chisq{P}{Q}/12$ and $a^2\chisq{Q}{P}/12$. Moreover, for $a\le1$ we have that $(\ratio-1)^2\le(e^a-1)^2\le3a^2$ and $(\ratio-1)^2/\ratio\le2(\cosh a-1)\le3a^2$, so the squared means are at most $3a^2\chisq{P}{Q}$ and $3a^2\chisq{Q}{P}$. The second claim follows in the same way.
\end{proof}

\begin{proofof}{Proof of \cref{lem:chi2-tight-sub} (Lower bound on variance of subsampled privacy losses)}
Let $Y\coloneqq\ratio(\omega)-1$ for $\omega\sim Q$, so that the likelihood ratio of $(P_\lambda,Q)$ is $\ratio_\lambda=1+\lambda Y\ge1-\lambda>1/2$. Thus \cref{lem:chi2-tight} applies to $(P_\lambda,Q)$, and its fourth moment for this pair is $\expect{}{(\ratio_\lambda-1)^4}=\lambda^4m_4$. By \cref{lem:subsampling-chi2} and its proof, $\chisq{P_\lambda}{Q}=\lambda^2\chiVar$ and $\chisq{Q}{P_\lambda}=\lambda^2\expect{}{Y^2/(1+\lambda Y)}\ge\lambda^2\chiVar-\lambda^3\expect{}{Y^3}$, where we used that $1/(1+z)\ge1-z$ for $z>-1$. Therefore
\[
 \lossVar{P_\lambda}{Q}\ge\lambda^2\chiVar-2\lambda^4m_4
 \qquad\text{and}\qquad
 \lossVar{Q}{P_\lambda}\ge\lambda^2\chiVar-\lambda^3\expect{}{Y^3}-5\lambda^4m_4 .
\]
Both bounds imply the first claim, since $\lambda^3\expect{}{Y^3}\le\lambda^3\sqrt{\chiVar m_4}=x\lambda^2\chiVar$ by the Cauchy--Schwarz inequality and $\lambda^4m_4=x^2\lambda^2\chiVar$. In the bounded case $m_4\le\rangeBnd^2\chiVar$, and therefore $x\le\lambda\rangeBnd\le1$ and $x+5x^2\le6\lambda\rangeBnd$.

For the last claim, let $\ind_\truncSet\coloneqq\ind\{\omega\in\truncSet\}$, $q\coloneqq Q(\truncSet^c)$ and $\tau\coloneqq\expect{}{Y^2(1-\ind_\truncSet)}/\chiVar$. In the remove and the add direction, the second moments of the truncated privacy losses are $\expect{}{\ratio_\lambda\ln^2(\ratio_\lambda)\ind_\truncSet}$ and $\expect{}{\ln^2(\ratio_\lambda)\ind_\truncSet}$, and their means are $\expect{}{\ratio_\lambda\ln(\ratio_\lambda)\ind_\truncSet}$ and $-\expect{}{\ln(\ratio_\lambda)\ind_\truncSet}$. The two pointwise inequalities in the proof of \cref{lem:chi2-tight} hold for every $s>0$. We take their expectations with the factor $\ind_\truncSet$ for $s=\ratio_\lambda\ge1/2$ and argue as in that proof and above. This implies that the two second moments are at least $\lambda^2(1-\tau)\chiVar-\lambda^4m_4/6$ and $\lambda^2(1-\tau-x)\chiVar-\lambda^4m_4/3$, respectively. For the means, note that $0\le s\ln s-(s-1)\le(s-1)^2$ and $0\le(s-1)-\ln s\le(s-1)^2$ for $s\ge1/2$. Moreover, $\expect{}{(\ratio_\lambda-1)\ind_\truncSet}=-\lambda\expect{}{Y(1-\ind_\truncSet)}$, since $\expect{}{Y}=0$, and this is at most $\lambda\sqrt{\chiVar q}$ in absolute value by the Cauchy--Schwarz inequality. Thus both means are at most $\lambda^2\chiVar+\lambda\sqrt{\chiVar q}$ in absolute value, and their squares are at most $2\lambda^4\chiVar^2+2q\lambda^2\chiVar\le(2x^2+2q)\lambda^2\chiVar$, where we used that $\chiVar^2\le m_4$. As a variance is the second moment minus the squared mean, both variances are at least $(1-x-5x^2-\tau-2q)\lambda^2\chiVar$. The last claim follows since $\tau\le\sqrt{m_4q}/\chiVar$ by the Cauchy--Schwarz inequality, $2q\le2\sqrt q\le2\sqrt{m_4q}/\chiVar$, and $\sqrt{m_4q}/\chiVar=(x/\lambda)\sqrt{q/\chiVar}$.
\end{proofof}

\subsection{Subsampled compositions}\label{app:subsampled-compositions}
\begin{proofof}{Proof of \cref{lem:subsampling-tails} (Tails of subsampled privacy losses)}
We have that $\lossFunc{\omega}{P_\lambda}{Q}=\ln(\ratio_\lambda(\omega))$ and $\lossFunc{\omega}{Q}{P_\lambda}=-\ln(\ratio_\lambda(\omega))$.

For the first bound, let $a\coloneqq1+(e^l-1)/\lambda>1$ and $A\coloneqq\{\ratio>a\}$. Then
\[
 \lossFunc{\omega}{P_\lambda}{Q}>l
 \iff\ratio_\lambda(\omega)>e^l
 \iff\ratio(\omega)>a
 \iff\lossFunc{\omega}{P}{Q}>\ln(a) ,
\]
and $P>aQ\ge Q$ on $A$, so
\[
 \prob{}{\lossRVfunc{P_\lambda}{Q}>l}=P_\lambda(A)=(1-\lambda)Q(A)+\lambda P(A)\le P(A)=\prob{}{\lossRVfunc{P}{Q}>\ln(a)}.
\]

For the second bound, concavity of the logarithm gives $\ln(\ratio_\lambda)\ge(1-\lambda)\ln(1)+\lambda\ln(\ratio)=\lambda\ln(\ratio)$, that is, $\lossFunc{\omega}{Q}{P_\lambda}\le\lambda\lossFunc{\omega}{Q}{P}$. Hence, for $\omega\sim Q$,
\[
 \prob{}{\lossRVfunc{Q}{P_\lambda}>l}=\prob{\omega\sim Q}{\lossFunc{\omega}{Q}{P_\lambda}>l}\le\prob{\omega\sim Q}{\lossFunc{\omega}{Q}{P}>\tfrac l\lambda}=\prob{}{\lossRVfunc{Q}{P}>\tfrac l\lambda}.
\]

For the last claim, let $|\ratio(\omega)-1|\le\rangeBnd$. Then $|\ratio_\lambda(\omega)-1|=\lambda|\ratio(\omega)-1|\le\lambda\rangeBnd$ and $\ratio_\lambda(\omega)\ge1-\lambda$, and since $\ln'(u)=1/u\le1/(1-\lambda)$ for $u\ge1-\lambda$, the mean value theorem gives
\[
 |\lossFunc{\omega}{P_\lambda}{Q}|=|\ln(\ratio_\lambda(\omega))-\ln(1)|\le\frac{|\ratio_\lambda(\omega)-1|}{1-\lambda}\le\frac{\lambda\rangeBnd}{1-\lambda}\le2\lambda\rangeBnd .
\]
\end{proofof}

\begin{proofof}{Proof of \cref{thm:subsampled} (Subsampled composition)}
We bound $\HockeyStickDiv{e^{\eps}}{P_\lambda^t}{Q^t}$ and may assume that $\chiVar>0$, since otherwise $P=Q$. The bound on $\HockeyStickDiv{e^{\eps}}{Q^t}{P_\lambda^t}$ follows in the same way from the pair $(Q,P_\lambda)$, the second tail bound of \cref{lem:subsampling-tails} and the second assumption.

Let $\lossRV\coloneqq\lossRVfunc{P_\lambda}{Q}$. By \cref{lem:chi2-moments,lem:subsampling-chi2}, $\expect{}{\lossRV^2}\le\chisq{P_\lambda}{Q}=\lambda^2\chiVar\le\subsampVar/t$, and in the add direction the second moment is at most $\chisq{Q}{P_\lambda}\le\subsampVar/t$. For the mean, let $x\coloneqq\lambda(\ratio-1)$, so that $\expect{}{\lossRV}=\expect{Q}{(1+x)\ln(1+x)}$. For every $x>-1$,
\begin{equation}\label{eq:log-quadratic}
 (1+x)\ln(1+x)\le x+\frac{x^2}{2(1+\min\{x,0\})}
 \qquad\text{and}\qquad
 -\ln(1+x)\le -x+\frac{x^2}{2(1+\min\{x,0\})} .
\end{equation}
Indeed, substituting $z=sx$ we have that
\begin{align*}
 (1+x)\ln(1+x)-x&=\int_0^x\frac{x-z}{1+z}\,dz=x^2\int_0^1\frac{1-s}{1+sx}\,ds,\\
 x-\ln(1+x)&=\int_0^x\frac{z}{1+z}\,dz=x^2\int_0^1\frac{s}{1+sx}\,ds ,
\end{align*}
and both claims follow since $1+sx\ge1+\min\{x,0\}$ and $\int_0^1(1-s)\,ds=\int_0^1s\,ds=1/2$. Since $x\ge-\lambda$ and $\expect{Q}{x}=0$, \cref{eq:log-quadratic} gives that $\expect{}{\lossRV}\le\expect{Q}{x^2}/(2(1-\lambda))=\subsampVar/(2t)$, and in the same way the mean $\expect{Q}{-\ln(1+x)}$ in the add direction is at most $\subsampVar/(2t)$. We truncate $\lossRV$ at $\lambda\rangeBnd$. As $(e^{\lambda \rangeBnd}-1)/\lambda\ge \rangeBnd$, the first tail bound of \cref{lem:subsampling-tails} and the first assumption give that $\prob{}{\lossRV>\lambda\rangeBnd}\le\prob{}{\lossRVfunc{P}{Q}>\ln(1+\rangeBnd)}\le\delta/(2t)$.

The truncated privacy loss $\lossRV'\coloneqq \lossRV\,\ind\{\lossRV\le \lambda\rangeBnd\}$ satisfies $\lossRV'\le \lossRV$ and $(\lossRV')^2\le \lossRV^2$, since $\lambda\rangeBnd>0$. Thus $\expect{}{\lossRV'}\le\subsampVar/(2t)$, $\expect{}{(\lossRV')^2}\le\subsampVar/t$ and $|\expect{}{\lossRV'}|\le\sqrt{\subsampVar/t}\le\lambda\sqrt{2\chiVar}\le \lambda\rangeBnd$, where we used that $\lambda<1/2$ and $\rangeBnd\ge\sqrt{2\chiVar}$. In particular, $\lossRV'-\expect{}{\lossRV'}\le2\lambda\rangeBnd$.

Let $\lossRV_1,\ldots,\lossRV_t$ be independent copies of $\lossRV$ and $\lossRV_i'\coloneqq \lossRV_i\ind\{\lossRV_i\le \lambda\rangeBnd\}$. Since the privacy loss of a product pair is the sum of the privacy losses of its coordinates, $\sum_i\lossRV_i$ is distributed as $\lossRVfunc{P_\lambda^t}{Q^t}$, and by a union bound, $\lossRV_i=\lossRV_i'$ for all $i$ except with probability $\delta/2$. By Bernstein's inequality \citep[Theorem~2.10]{BLM13} for the centered $\lossRV_i'$ with $2\lambda\rangeBnd$ in place of $\rangeBnd$, $V=\subsampVar$ and $\delta/2$ in place of $\delta$, except with probability $\delta$ we have $\sum_{i=1}^t\lossRV_i=\sum_{i=1}^t\lossRV_i'\le\subsampVar/2+\sqrt{2\subsampVar\ln(2/\delta)}+\frac23\lambda\rangeBnd\ln(2/\delta)=\eps$.
By the tail bound, this gives the claim.
\end{proofof}

\begin{proofof}{Proof of \cref{cor:subsampled-approx}}
Let $C$ be the constant of \cref{thm:gauss-approx} and $\truncSet\coloneqq\{\omega:-\rangeBnd\le\lossFunc{\omega}{P}{Q}\le\ln(1+\rangeBnd)\}$. On $\truncSet$ the likelihood ratio $\ratio=e^{\lossFunc{\omega}{P}{Q}}$ is between $e^{-\rangeBnd}\ge1-\rangeBnd$ and $1+\rangeBnd$, and therefore $|\ratio-1|\le\rangeBnd$. Both $P(\truncSet^c)$ and $Q(\truncSet^c)$ are at most $\prob{}{\lossRVfunc{P}{Q}>\ln(1+\rangeBnd)}+\prob{}{\lossRVfunc{Q}{P}>\rangeBnd}\le\tailShift/(2t)$, since $\lossFunc{\omega}{Q}{P}=-\lossFunc{\omega}{P}{Q}$ and since $P\le Q$ where $\lossFunc{\omega}{P}{Q}<0$ and $Q\le P$ where $\lossFunc{\omega}{P}{Q}>0$. Therefore $P_\lambda(\truncSet^c)=(1-\lambda)Q(\truncSet^c)+\lambda P(\truncSet^c)\le\tailShift/(2t)$, and since $|\ratio-1|\le\rangeBnd$ on $\truncSet$, \cref{lem:subsampling-tails} gives that $|\lossFunc{\omega}{Q}{P_\lambda}|=|\lossFunc{\omega}{P_\lambda}{Q}|\le 2\lambda\rangeBnd$ for every $\omega\in \truncSet$. Thus in both directions the truncated privacy loss of one step (which is zero outside $\truncSet$) and its mean are at most $2\lambda\rangeBnd$ in absolute value. In particular, the truncated privacy loss deviates from its mean by at most $4\lambda\rangeBnd$. Therefore, by \cref{thm:gauss-approx} for the pairs $(P_\lambda,Q)$ and $(Q,P_\lambda)$, with the set $\truncSet$ and $4\lambda\rangeBnd$ in place of $\rangeBnd$, we obtain the claimed concentration with $\mu'$ and $V'$ the mean and the variance of the truncated privacy loss. In the remove direction this loss is $\sum_{i=1}^t\lossFunc{\omega_i}{P_\lambda}{Q}\,\ind\{\omega_i\in\truncSet\}$ for $\omega\sim P_\lambda^t$, and in the add direction it is $\sum_{i=1}^t\lossFunc{\omega_i}{Q}{P_\lambda}\,\ind\{\omega_i\in\truncSet\}$ for $\omega\sim Q^t$.

We now bound the means and the variances. By \cref{lem:chi2-moments,lem:subsampling-chi2}, the variance of the truncated privacy loss of one step is at most $\chisq{P_\lambda}{Q}=\lambda^2\chiVar$ in the remove direction and at most $\chisq{Q}{P_\lambda}\le\lambda^2\chiVar/(1-\lambda)$ in the add direction, so both variances are at most $\subsampVar$. For the means, let $x\coloneqq\lambda(\ratio-1)\ge-\lambda$, so that the likelihood ratio of $(P_\lambda,Q)$ is $1+x$, and recall \cref{eq:log-quadratic} from the proof of \cref{thm:subsampled}. Since $1+\min\{x,0\}\ge1-\lambda$, the change of measure from $P_\lambda$ to $Q$ and \cref{eq:log-quadratic} give that the means of the truncated privacy losses of one step satisfy
\begin{align*}
 \expect{\omega\sim P_\lambda}{\lossFunc{\omega}{P_\lambda}{Q}\ind\{\omega\in\truncSet\}}&\le\expect{Q}{x\,\ind\{\omega\in\truncSet\}}+\frac{\expect{Q}{x^2}}{2(1-\lambda)},\\
 \expect{\omega\sim Q}{\lossFunc{\omega}{Q}{P_\lambda}\ind\{\omega\in\truncSet\}}&\le-\expect{Q}{x\,\ind\{\omega\in\truncSet\}}+\frac{\expect{Q}{x^2}}{2(1-\lambda)} .
\end{align*}
Here $\expect{Q}{x^2}=\lambda^2\chiVar$, and $\expect{Q}{x\,\ind\{\omega\in\truncSet\}}=\lambda(P(\truncSet)-Q(\truncSet))=\lambda(Q(\truncSet^c)-P(\truncSet^c))$ is at most $\lambda\tailShift/(2t)$ in absolute value. Summing over the $t$ steps, both means are at most $\subsampVar/2+\lambda\tailShift/2$. If $\lossRVfunc{P}{Q}$ and $\lossRVfunc{Q}{P}$ have the same distribution, then the assumptions also hold with $\ln(1+\rangeBnd)$ in place of $\rangeBnd$ in the second one. Thus the argument above also applies to the set $\truncSet_2\coloneqq\{\omega:|\lossFunc{\omega}{P}{Q}|\le\ln(1+\rangeBnd)\}$ in place of $\truncSet$, since $\ratio$ is between $1/(1+\rangeBnd)$ and $1+\rangeBnd$ on $\truncSet_2$. Moreover, $P(\truncSet_2^c)=Q(\truncSet_2^c)$, since $\lossRVfunc{Q}{P}$ is distributed as $-\lossFunc{\omega}{P}{Q}$ for $\omega\sim Q$. Therefore both means are at most $\subsampVar/2$.

For the shifted GDP bound, fix one of the two directions and let $\mu'$ and $V'$ be as above. By \cref{lem:gc-profile}, the hockey-stick divergence $\HockeyStick{e^{\eps}}$ of the pair in this direction is at most $g_{V'}(\eps-(\mu'-V'/2)-4C\lambda\rangeBnd\ln(2/\tailShift))+\tailShift$ for every $\eps\in\reals$. This bound is at most $g_{\subsampVar}(\eps-\lossShift)+\tailShift$, since $\mu'\le\subsampVar/2+\lambda\tailShift/2$, $0\le V'\le\subsampVar$, and $g_V(\eps)$ is nonincreasing in $\eps$ and nondecreasing in $V$. If $\lossRVfunc{P}{Q}$ and $\lossRVfunc{Q}{P}$ have the same distribution, then $\mu'\le\subsampVar/2$, and the same argument gives this bound without the term $\lambda\tailShift/2$ of $\lossShift$. If $(P_\lambda^t,Q^t)$ dominates an algorithm, then it dominates the pairs of output distributions of the algorithm in the remove direction, and $(Q^t,P_\lambda^t)$ dominates those in the add direction, since reversing both pairs preserves domination (\cref{lem:kov}). Hence the algorithm is $(\lossShift,\tailShift)$-shifted $\sqrt{\subsampVar}$-GDP.
\end{proofof}

\begin{proofof}{Proof of \cref{cor:subsampled-lower} (Lower bounds on concentration parameters)}
By the proof of \cref{cor:subsampled-approx}, $\mu'$ and $V'$ are $t$ times the mean and the variance of the truncated privacy loss of one step, for the set $\truncSet$ of that proof (or $\truncSet_2$ in the symmetric case), and $P(\truncSet^c)$ and $Q(\truncSet^c)$ are at most $\tailShift/(2t)$. Since $t\lambda^2\chiVar=(1-\lambda)\subsampVar$, the bound on $V'$ follows from the last part of \cref{lem:chi2-tight-sub}.

For the mean, let $Y$, $\ind_\truncSet$ and $\tau$ be as in the proof of \cref{lem:chi2-tight-sub}, and let $u\coloneqq\lambda Y$, so that the likelihood ratio of $(P_\lambda,Q)$ is $1+u$. By the integral identities in the proof of \cref{thm:subsampled} and $1/(1+z)\ge1-z$, we have that $(1+u)\ln(1+u)-u\ge u^2/2-u^3/6$ and $u-\ln(1+u)\ge u^2/2-u^3/3$. Therefore the means of the truncated privacy losses of one step in the remove and the add direction are at least
\[
 \pm\expect{}{u\,\ind_\truncSet}+\frac{\lambda^2}2\expect{}{Y^2\ind_\truncSet}-\frac{\lambda^3}3\expect{}{|Y|^3}\ge\frac{\lambda^2\chiVar}2(1-\tau-x)-\frac{\lambda\tailShift}{2t},
\]
where we used that $|\expect{}{u\,\ind_\truncSet}|=\lambda|Q(\truncSet^c)-P(\truncSet^c)|\le\lambda\tailShift/(2t)$ and $\lambda^3\expect{}{|Y|^3}\le\lambda^3\sqrt{\chiVar m_4}=x\lambda^2\chiVar$. In the symmetric case $P(\truncSet_2^c)=Q(\truncSet_2^c)$, and the term $\lambda\tailShift/(2t)$ is not needed. Summing over the $t$ steps and using that $\tau\le\sqrt{m_4Q(\truncSet^c)}/\chiVar$ by the Cauchy--Schwarz inequality, we obtain the claimed bound on $\mu'$.
\end{proofof}

\subsection{The subsampled Gaussian mechanism}\label{app:gauss-central}
In the rest of this appendix $(P,Q)$ is the Gaussian pair $(\Nrm(1,\sigma^2),\Nrm(0,\sigma^2))$ and $\chiVar=e^{1/\sigma^2}-1$. \begin{claim}[The Gaussian pair]\label{clm:gauss-pair}
Let $\sigma>0$ and let $(P,Q)$ be the Gaussian pair. Then $\lossRVfunc{P}{Q}$ and $\lossRVfunc{Q}{P}$ are both distributed as $\Nrm(1/(2\sigma^2),1/\sigma^2)$, and thus $\prob{}{\lossRVfunc{P}{Q}>\loss}=\prob{}{\lossRVfunc{Q}{P}>\loss}=\Phibar\left(\sigma \loss-1/(2\sigma)\right)$ for every $\loss\in\reals$.
Moreover, $\chisq{P}{Q}=e^{1/\sigma^2}-1$, and if $\sigma\ge1$, then $1/\sigma^2\le \chisq{P}{Q}\le(e-1)/\sigma^2$.
\end{claim}
\begin{proof}
The privacy loss $\lossFunc{\omega}{P}{Q}=(2\omega-1)/(2\sigma^2)$ is distributed as $\Nrm(1/(2\sigma^2),1/\sigma^2)$ for $\omega\sim P$ and as $\Nrm(-1/(2\sigma^2),1/\sigma^2)$ for $\omega\sim Q$, and $\lossRVfunc{Q}{P}$ is $-\lossFunc{\omega}{P}{Q}$ for $\omega\sim Q$. Completing the square gives that $\expect{\omega\sim Q}{\ratioFunc{\omega}{P}{Q}^2}=e^{1/\sigma^2}$, that is, $\chisq{P}{Q}=e^{1/\sigma^2}-1$. For $\sigma\ge1$ we have that $1+1/\sigma^2\le e^{1/\sigma^2}\le1+(e-1)/\sigma^2$ by the convexity of the exponential function and $1/\sigma^2\le1$.
\end{proof}

We prove that \cref{eq:gauss-conditions} holds, for every value of its constant, in the two cases stated informally after it, since we also use them for random allocation (\cref{sec:alloc}).
\begin{lemma}[Sufficient conditions for \cref{eq:gauss-conditions}]\label{lem:gauss-conditions}
Let $C>0$. For every $\sigma\ge1$, $t\ge2$ and $0<\delta\le1/2$, the inequality \cref{eq:gauss-conditions} holds in each of the following two cases:
\begin{enumerate}
 \item $\eta \in (0,1/2)$, $t\ge t_\eta$, $\ln(1/\delta)\le t^{1/2-\eta}$ and $\sigma\ge C_\eta\max\left\{1,\sqrt{\ln(1/\delta)}/\ln(t)\right\}$, where $C_\eta$ and $t_\eta$ depend only on $\eta$ and $C$;
 \item $t\ge\max\left\{t_\sigma,\exp\left(3\sqrt2\,C\sqrt{\ln(1/\delta)}/\sigma\right)\right\}$, where $t_\sigma$ depends only on $\sigma$ and $C$.
\end{enumerate}
\end{lemma}
\begin{proof}
Let $H\coloneqq\ln(t/\delta)$. We write \cref{eq:gauss-conditions} as $F(\sigma)\ge C\sqrt H$, where $F(s)\coloneqq s\ln(1+a/s)$ and $a\coloneqq\sqrt{t/(C\ln(1/\delta))}$. The function $F$ is increasing, since $F'(s)=\ln(1+b)-b/(1+b)>0$ for $b=a/s$. We take $t_\eta,t_\sigma\ge3$, so that $\ln(t)\ge1$.

In the first case let $C_\eta\coloneqq4C/\eta$. As $F$ is increasing, it suffices to verify the condition at $\sigma_*\coloneqq C_\eta\max\left\{1,\sqrt{\ln(1/\delta)}/\ln(t)\right\}$. Since $\ln(t)\ge1$, $\eta<1/2$ and $\ln(1/\delta)\le t^{1/2-\eta}$, we have that $\sqrt{\ln(1/\delta)}\,\sigma_*\le C_\eta\max\left\{\sqrt{\ln(1/\delta)},\ln(1/\delta)\right\}\le C_\eta t^{1/2-\eta}$, that is, $a/\sigma_*\ge t^{\eta}/\left(C_\eta\sqrt C\right)$. Therefore $\ln(1+a/\sigma_*)\ge\frac\eta2\ln(t)$ for all $t$ larger than a threshold that depends only on $\eta$ and $C$, and for such $t$,
\[
 F(\sigma_*)\ge\frac{\eta C_\eta}2\max\left\{\ln(t),\sqrt{\ln(1/\delta)}\right\}\ge C\left(\ln(t)+\sqrt{\ln(1/\delta)}\right)\ge C\sqrt{\ln(t)+\ln(1/\delta)}=C\sqrt H .
\]

In the second case let $c\coloneqq1/(18C^2)$. The assumption $t\ge\exp\left(3\sqrt2\,C\sqrt{\ln(1/\delta)}/\sigma\right)$, that is, $\ln(1/\delta)\le c\sigma^2\ln^2(t)$, gives that $a/\sigma\ge\sqrt t/\left(\sigma^2\sqrt{Cc}\,\ln(t)\right)$. Thus for all $t$ larger than a threshold that depends only on $\sigma$ and $C$ we have that $F(\sigma)\ge\frac\sigma3\ln(t)$ and $\ln(t)\le c\sigma^2\ln^2(t)$, and therefore
\[
 C\sqrt H\le C\sqrt{\ln(t)+c\sigma^2\ln^2(t)}\le C\sigma\sqrt{2c}\,\ln(t)=\frac\sigma3\ln(t)\le F(\sigma) .
\]
\end{proof}

\begin{proofof}{Proof of \cref{cor:gauss-poisson} (Shifted GDP of the subsampled Gaussian mechanism)}
Let $w\coloneqq\sqrt{2\ln(2t/\tailShift)}$, so that $\ln(1+\rangeBnd)=1/(2\sigma^2)+w/\sigma$. By \cref{clm:gauss-pair},
\[
 \prob{}{\lossRVfunc{P}{Q}>\ln(1+\rangeBnd)}=\prob{}{\lossRVfunc{Q}{P}>\ln(1+\rangeBnd)}=\Phibar(w)\le\frac{e^{-w^2/2}}2=\frac{\tailShift}{4t} .
\]
Thus, as $\ln(1+\rangeBnd)\le\rangeBnd$, \cref{cor:subsampled-approx} applies with $\chiVar=e^{1/\sigma^2}-1$, and since $\lossRVfunc{P}{Q}$ and $\lossRVfunc{Q}{P}$ have the same distribution (\cref{clm:gauss-pair}), it applies without the term $\lambda\tailShift/2$ of $\lossShift$. Since the Gaussian pair dominates the Gaussian mechanism (\cref{clm:gauss-dom}), the pair $(P_\lambda^t,Q^t)$ dominates the $t$-step adaptive composition of the $\lambda$-subsampled Gaussian mechanism by \cref{lem:dom-comp-sub}, which gives the claim. The proof of \cref{cor:subsampled-approx} gives the same bound on the hockey-stick divergences of the pairs $(P_\lambda^t,Q^t)$ and $(Q^t,P_\lambda^t)$.
\end{proofof}

We now state formally the bound for every sampling rate that is stated informally before \cref{cor:gauss-leading}.
\begin{claim}[Closed-form bound for the subsampled Gaussian mechanism]\label{clm:gauss-closed-form}
There is a universal constant $C>0$ such that the following holds. Let $\sigma\ge1$, $t\ge2$, $\lambda \in (0,1/2)$ and $0<\delta\le1/2$. If \cref{eq:gauss-conditions} holds, then the $t$-step adaptive composition of the $\lambda$-subsampled Gaussian mechanism \cref{eq:gauss-mech} is $(\eps,\delta)$-DP for $\eps\coloneqq C\left(t\lambda^2/\sigma^2+\lambda\sqrt{t\ln(1/\delta)}/\sigma\right)$.
\end{claim}
\begin{proof}
Let $C_1$ be the constant of \cref{cor:gauss-poisson}. By \cref{cor:gauss-poisson} with $\tailShift=\delta/2$, the $t$-step adaptive composition of the $\lambda$-subsampled Gaussian mechanism is $(\lossShift,\delta/2)$-shifted $\sqrt{\subsampVar}$-GDP for $\subsampVar=t\lambda^2\chiVar/(1-\lambda)$ and $\lossShift=4C_1\lambda\rangeBnd\ln(4/\delta)+\subsampVar/2$, where $\ln(1+\rangeBnd)=\frac1{2\sigma^2}+\frac{\sqrt{2\ln(4t/\delta)}}{\sigma}$. By \cref{def:shifted-gdp}, it is $(\eps_{\mathrm G}(\subsampVar,\delta/2)+\lossShift,\delta)$-DP.

Let $H\coloneqq\ln(t/\delta)$, so that $H\ge\ln(4)$ for $t\ge2$ and $\delta\le1/2$, and let $C\ge\max\{5/2,\,2(e-1),\,\allowbreak\sqrt{8(e-1)}+8C_1\}$. As $\ln(4t/\delta)=\ln(4)+H\le2H$ and $1/(2\sigma^2)\le\sqrt H/(2\sigma)$ for $\sigma\ge1$, we have that $\ln(1+\rangeBnd)\le\frac52\sqrt H/\sigma\le C\sqrt H/\sigma$. Therefore \cref{eq:gauss-conditions} gives that $\rangeBnd\le\sqrt t/\left(\sigma\sqrt{C\ln(1/\delta)}\right)$.
As $\ln(4/\delta)\le3\ln(1/\delta)$ for $\delta\le1/2$ and $C\ge5/2$, the first term of $\lossShift$ satisfies
\[
 4C_1\lambda\rangeBnd\ln(4/\delta)\le\frac{12C_1}{\sqrt C}\cdot\frac{\lambda\sqrt{t\ln(1/\delta)}}{\sigma}\le\frac{8C_1\lambda\sqrt{t\ln(1/\delta)}}{\sigma} .
\]
By \cref{clm:gauss-pair} and $\lambda<1/2$, we have that $\subsampVar\le2(e-1)t\lambda^2/\sigma^2$, and since $\eps_{\mathrm G}(\subsampVar,\delta/2)\le\subsampVar/2+\sqrt{2\subsampVar\ln(2/\delta)}$ and $\ln(2/\delta)\le2\ln(1/\delta)$,
\begin{align*}
 \eps_{\mathrm G}(\subsampVar,\delta/2)+\lossShift
 &\le\subsampVar+\sqrt{2\subsampVar\ln(2/\delta)}+\frac{8C_1\lambda\sqrt{t\ln(1/\delta)}}{\sigma}\\
 &\le\frac{2(e-1)t\lambda^2}{\sigma^2}+\left(\sqrt{8(e-1)}+8C_1\right)\frac{\lambda\sqrt{t\ln(1/\delta)}}{\sigma} ,
\end{align*}
which is at most $\eps$ by the choice of $C$.
\end{proof}

The following corollary is the general version of \cref{cor:gauss-leading}.
It allows a noise scale $\sigma$ that depends on $t$, and its second claim assumes a more general condition on $\delta$ and $\sigma$.
The condition of \cref{cor:gauss-leading} is the first of the two sufficient conditions for it that are listed at the end of the statement.
\begin{corollary}[Asymptotic upper bounds for the subsampled Gaussian mechanism, general version]\label{cor:gauss-leading-gen}
Let $\sigma=\sigma(t)>0$ be bounded away from zero and let $\chiVar\coloneqq e^{1/\sigma^2}-1$.
If $\delta=t^{-\theta}$ for a fixed $\theta>1/2$, then, as $t\to\infty$, the $t$-step adaptive composition of the $1/t$-subsampled Gaussian mechanism \cref{eq:gauss-mech} is $(\eps,\delta)$-DP for ${\eps=(1+o(1))\sqrt{(2\theta-1)\,\chiVar\ln(t)/t}}$.
If $\delta=\delta(t)$ satisfies $\ln(1/\delta)=\omega(\ln(t))$ and
\begin{equation}\label{eq:growing-sigma-cond}
 \ln(1/\delta)\exp\left(\frac{\sqrt{2\ln(4t/\delta)}}{\sigma}\right)=o\left(\sqrt t\right),
\end{equation}
then it is $(\eps,\delta)$-DP for ${\eps=(1+o(1))\sqrt{2\chiVar\ln(1/\delta)/t}}$.
In particular, this bound holds if $\ln(1/\delta)=\omega(\ln(t))$ and one of the following two conditions holds:
\begin{enumerate}
 \item $\sigma$ does not depend on $t$ and $\limsup_{t\to\infty}8\ln(1/\delta)/(\sigma\ln(t))^2<1$;
 \item $\ln(1/\delta)=o\left(\sqrt t\right)$ and $\sigma=\Omega\left(\sqrt{\ln(1/\delta)}\right)$.
\end{enumerate}
\end{corollary}

\begin{proof}
Let $\sigma_0>0$ be such that $\sigma\ge\sigma_0$ for every $t$, and let $C_1$ be the constant of \cref{cor:gauss-poisson}.
By \cref{cor:gauss-poisson} with $\lambda=1/t$ and $\tailShift=\delta/2$, the composition is $(\lossShift,\delta/2)$-shifted $\sqrt{\subsampVar}$-GDP for $\subsampVar=\chiVar/(t-1)$ and $\lossShift=(4C_1\rangeBnd/t)\ln(4/\delta)+\subsampVar/2$, where $\ln(1+\rangeBnd)=1/(2\sigma^2)+w/\sigma$ for $w\coloneqq\sqrt{2\ln(4t/\delta)}$.
By \cref{def:shifted-gdp}, the composition is $(\eps_{\mathrm G}(\subsampVar,\delta/2)+\lossShift,\delta)$-DP.
The claimed bound is $\sqrt{\chiVar K/t}$, where $K\coloneqq(2\theta-1)\ln(t)$ in the first claim and $K\coloneqq2\ln(1/\delta)$ in the second.
Thus it suffices to show that the Gaussian term $\eps_{\mathrm G}(\subsampVar,\delta/2)$ is at most $(1+o(1))\sqrt{\chiVar K/t}$ and that the shift $\lossShift$ is $o\left(\sqrt{\chiVar K/t}\right)$.
The Gaussian term determines the leading constant and needs only the assumptions on $\delta$, whereas \cref{eq:growing-sigma-cond} is needed only for the shift.
Both parts use that $\chiVar\le e^{1/\sigma_0^2}$, which is why $\sigma$ is bounded away from zero.

\para{The Gaussian term.} We bound $\eps_{\mathrm G}(U,\delta/2)$ for every $0<U\le\chiVar/(t-1)$, since for random allocation (\cref{cor:gauss-leading-ra-gen}) we use $U=\chiVar/t$ in place of $\subsampVar$.
The tail bound $\eps_{\mathrm G}(U,\delta/2)\le U/2+\sqrt{2U\ln(2/\delta)}$ gives, up to lower-order terms, the claimed bound with $2\ln(2/\delta)$ in place of $K$.
This suffices for the second claim, but in the first claim $K=2\ln(1/\delta)-\ln(t)$ is smaller by $\ln(t)$, which is of the same order.
\Cref{lem:gdp-small-var} recovers this term. It gives that $\eps_{\mathrm G}(U,\delta/2)\le U/2+\sqrt{kU}$ for every $k\ge1$ such that $k\ge2\ln(2/\delta)-\ln(1/U)$, and $\ln(1/U)$ is at least $\ln(t)$ up to an additive constant.
Indeed, $U\le2\chiVar/t\le2e^{1/\sigma_0^2}/t$, and therefore $\ln(1/U)\ge\ln(t)-1/\sigma_0^2-1$.
Thus we can take
\[
 k\coloneqq2\ln(2/\delta)-\ln(t)+\frac1{\sigma_0^2}+1=
 \begin{cases}
 (1+o(1))\,(2\theta-1)\ln(t) & \text{if }\delta=t^{-\theta},\\
 (1+o(1))\,2\ln(1/\delta) & \text{if }\ln(1/\delta)=\omega(\ln(t)),
 \end{cases}
\]
where the first case uses that $\theta>1/2$, and the second uses that $\ln(1/\delta)=\omega(\ln(t))$, which makes the term $\ln(t)$ negligible.
That is, $k=(1+o(1))K$ in both claims, and in particular $k\to\infty$ and $k\ge1$ for all sufficiently large $t$.
As $U\le2e^{1/\sigma_0^2}/t$ tends to zero, we have that $U/2=o\left(\sqrt{kU}\right)$, and therefore
\[
 \eps_{\mathrm G}(U,\delta/2)\le(1+o(1))\sqrt{kU}\le(1+o(1))\sqrt{\frac{\chiVar K}t} ,
\]
where the last step uses that $k=(1+o(1))K$ and $U\le\chiVar/(t-1)$.

\para{The shift.} The term $\subsampVar/2\le\chiVar/t$ of $\lossShift$ is $o\left(\sqrt{\chiVar K/t}\right)$, since $\chiVar\le e^{1/\sigma_0^2}$.
The other term is the loss shift $(4C_1\rangeBnd/t)\ln(4/\delta)$, which is proportional to the cutoff $\rangeBnd$.
For random allocation we need the same quantity with another constant in place of $4$, and therefore we bound $(\rangeBnd/t)\ln(C'/\delta)$ for every fixed $C'>0$.
Using that $e^z-1\le ze^z$ for $z\ge0$ and $\chiVar\ge1/\sigma^2$, we get that
\[
 \frac{\rangeBnd}{\sqrt{\chiVar}}\le(1+\rangeBnd)\,\sigma\ln(1+\rangeBnd)=(1+\rangeBnd)\left(\frac1{2\sigma}+w\right)\le(1+\rangeBnd)\left(\frac1{2\sigma_0}+w\right),
\]
and therefore, for all sufficiently large $t$ (as $\delta\to0$ in both claims),
\begin{equation}\label{eq:shift-ratio}
 \frac{(\rangeBnd/t)\ln(C'/\delta)}{\sqrt{\chiVar K/t}}\le\frac{(1+\rangeBnd)\left(w+1/(2\sigma_0)\right)\ln(C'/\delta)}{\sqrt{tK}} .
\end{equation}
If $\delta=t^{-\theta}$, then $w=O\left(\sqrt{\ln(t)}\right)$ and $\ln(C'/\delta)=O(\ln(t))$.
Moreover, $\ln(1+\rangeBnd)\le1/(2\sigma_0^2)+w/\sigma_0=O\left(\sqrt{\ln(t)}\right)$, so that $1+\rangeBnd=t^{o(1)}$.
Thus the right-hand side of \cref{eq:shift-ratio} is $t^{-1/2+o(1)}$.
No assumption on $\sigma$ beyond $\sigma\ge\sigma_0$ is needed here, since $\delta$ decays only polynomially and therefore the cutoff grows slower than every power of $t$.

In the second claim, as $\ln(1/\delta)=\omega(\ln(t))$, each of $\ln(4t/\delta)$ and $\ln(C'/\delta)$ is $(1+o(1))\ln(1/\delta)$, and therefore $w+1/(2\sigma_0)=(1+o(1))\sqrt K$.
Thus the right-hand side of \cref{eq:shift-ratio} is $(1+o(1))(1+\rangeBnd)\ln(1/\delta)/\sqrt t$.
The loss shift is therefore of lower order than the claimed bound if $(1+\rangeBnd)\ln(1/\delta)=o\left(\sqrt t\right)$, which is the actual requirement of the second claim.
Since $e^{w/\sigma}\le1+\rangeBnd\le e^{1/(2\sigma_0^2)}e^{w/\sigma}$, this is equivalent to \cref{eq:growing-sigma-cond}.

In both claims the right-hand side of \cref{eq:shift-ratio} therefore tends to zero.
For $C'=4$, this shows that the loss shift is $o\left(\sqrt{\chiVar K/t}\right)$, and therefore so is $\lossShift$.
Moreover, the claimed bound tends to zero, since $\chiVar\le e^{1/\sigma_0^2}$ and $K=o(t)$ (in the second claim because \cref{eq:growing-sigma-cond} implies that $\ln(1/\delta)=o\left(\sqrt t\right)$).
Thus $(\rangeBnd/t)\ln(C'/\delta)$ also tends to zero for every fixed $C'>0$.

\para{The two conditions.} We now show that each of the two conditions implies \cref{eq:growing-sigma-cond}.
They control the factor $e^{w/\sigma}$, which is the cutoff $1+\rangeBnd$ up to a constant factor, in two different ways.
As $\ln(1/\delta)=\omega(\ln(t))$, we have that $w=(1+o(1))\sqrt{2\ln(1/\delta)}$.
For a fixed $\sigma$, the cutoff can grow polynomially in $t$, and the first condition keeps it a fixed power below $\sqrt t$, which leaves room for the factor $\ln(1/\delta)$.
Indeed, there is a fixed $\gamma>0$ such that $8\ln(1/\delta)\le(\sigma\ln(t)/(1+\gamma))^2$ for all sufficiently large $t$.
Then $\sqrt{2\ln(1/\delta)}\le\sigma\ln(t)/(2(1+\gamma))$, so that $\ln(1/\delta)=O(\ln^2(t))$ and $w/\sigma\le(1+o(1))\ln(t)/(2(1+\gamma))$.
Thus the left-hand side of \cref{eq:growing-sigma-cond} is at most $t^{1/(2(1+\gamma))+o(1)}=o\left(\sqrt t\right)$.
Under the second condition, $w/\sigma=(1+o(1))\sqrt{2\ln(1/\delta)}/\sigma=O(1)$ and the cutoff $\rangeBnd$ is bounded.
Then the left-hand side of \cref{eq:growing-sigma-cond} is $O(\ln(1/\delta))=o\left(\sqrt t\right)$.
For a fixed $\sigma$, the first claim and the first condition give \cref{cor:gauss-leading}.
\end{proof}

\section{Proofs for \texorpdfstring{\cref{sec:alloc}}{Section 5}}\label{app:alloc-proofs}
In this appendix $\ratio$ is the likelihood ratio of $(P,Q)$.
\begin{proofof}{Proof of \cref{thm:ra-gauss} (Gaussian concentration for random allocation)}
Let $C_1$ be the constant of \cref{lem:coupling}, let $V\coloneqq\chiVar/t$, $u\coloneqq\ln(4/\tailShift)$ and $\rho\coloneqq(\rangeBnd/t)u$, and let $c\coloneqq1/(16(C_1+2)^2)$. We take $C\ge C_1+8$, so that $u\le\ln(C/\tailShift)$ and the assumption implies that $\rho\le c$. We can assume that $\chiVar>0$ (otherwise $P=Q$ and both privacy losses are $0$). Since $\ratio\le1+\rangeBnd$ and $\expect{Q}{\ratio}=1$, we have that $\chiVar=\expect{Q}{\ratio^2}-1\le\rangeBnd$. Therefore $Vu\le\rho$, and $V\le\rho$ as $u>1$.

Let $\omega_0\sim P$ and $\omega_1,\ldots,\omega_t\sim Q$ be mutually independent, and let $A\coloneqq\frac1t\sum_{i=1}^t\ratio(\omega_i)$ and $B\coloneqq\frac1t\left(\ratio(\omega_0)+\sum_{i=2}^t\ratio(\omega_i)\right)$.
By \cref{lem:alloc-ratio}, the privacy losses $\lossRVfunc{Q^t}{\Pbar{t}}$ and $\lossRVfunc{\Pbar{t}}{Q^t}$ are distributed as $-\ln A$ and $\ln B$, and the claim depends only on their distributions. Since $|\ratio-1|\le\rangeBnd$ almost surely under $Q$ and under $P$ (as $P\ll Q$), we have that $|B-A|=|\ratio(\omega_0)-\ratio(\omega_1)|/t\le2\rangeBnd/t\le2\rho$.

The summands $(\ratio(\omega_i)-1)/t$ of $A-1$ are independent, have mean $0$, are at most $\rangeBnd/t$ in absolute value, and their sum has variance $V$. By \cref{lem:coupling} with $\tailShift/2$ in place of $\tailShift$, we can couple $A$ with $G\sim\Nrm(0,1)$ so that $|A-1-\sqrt V\,G|\le C_1\rho$ except with probability $\tailShift/2$ (given $\omega_1,\ldots,\omega_t$, we sample $G$ from its conditional distribution in this coupling). Moreover, $\prob{}{|G|>\sqrt{2u}}\le e^{-u}<\tailShift/2$. Thus, except with probability $\tailShift$, for $D\coloneqq C_1+2$ we have that $|A-1-\sqrt V\,G|\le D\rho$ and $|B-1-\sqrt V\,G|\le D\rho$.
On this event both $|A-1|$ and $|B-1|$ are at most $\sqrt{2Vu}+D\rho\le\sqrt{2\rho}+D\rho\le1/2$, and their squares are at most $4\rho+2D^2\rho^2\le5\rho$, by the choice of $c$. As $|\ln(1+z)-z|\le z^2$ for $|z|\le1/2$, we get that
\[
 \left|-\ln A-\left(\frac V2-\sqrt V\,G\right)\right|\le D\rho+5\rho+\frac V2
 \qquad\text{and}\qquad
 \left|\ln B-\left(\frac V2+\sqrt V\,G\right)\right|\le D\rho+5\rho+\frac V2 .
\]
Both right-hand sides are at most $(C_1+8)\rho\le(C\rangeBnd/t)\ln(C/\tailShift)$. Since $V/2\pm\sqrt V\,G$ is distributed as $\Nrm(V/2,V)$, this coupling gives the first claim, as noted after \cref{def:gauss-conc}.

For the shifted GDP bound, note that the mean $V/2$ of this Gaussian distribution is half of its variance $V$. Thus, by \cref{lem:gc-profile}, the hockey-stick divergence $\HockeyStick{e^{\eps}}$ of each of the two pairs is at most $g_V(\eps-(C\rangeBnd/t)\ln(C/\tailShift))+\tailShift$ for every $\eps\in\reals$. If $(\Pbar{t},Q^t)$ dominates an algorithm, then it dominates the pairs of output distributions of the algorithm in the remove direction, and $(Q^t,\Pbar{t})$ dominates those in the add direction, since reversing both pairs preserves domination (\cref{lem:kov}). Hence the algorithm is $((C\rangeBnd/t)\ln(C/\tailShift),\tailShift)$-shifted $\sqrt V$-GDP.
\end{proofof}

We now prove the claim for unbounded likelihood ratios stated before \cref{cor:gauss-ra}.
\begin{lemma}[Conditioning for random allocation]\label{lem:ra-conditioning}
Let $P\ll Q$ be distributions, let $c$ and $C$ be the constants of \cref{thm:ra-gauss}, let $t\in\naturals$, $\rangeBnd>0$ and $0<\tailShift\le1$ with $(\rangeBnd/t)\ln(2C/\tailShift)\le c$, and let $\truncSet$ be an event such that $P(\truncSet)=Q(\truncSet)\ge1-\tailShift/(2t)$ and $|\ratioFunc{\omega}{P}{Q}-1|\le\rangeBnd$ for every $\omega\in\truncSet$. Then the privacy losses of $(\Pbar{t},Q^t)$ and of $(Q^t,\Pbar{t})$ are $((C\rangeBnd/t)\ln(2C/\tailShift),\tailShift)$-approximately $(\chiVar_\truncSet/(2t),\chiVar_\truncSet/t)$-Gaussian-concentrated, where $P_\truncSet\coloneqq P(\cdot\mid\truncSet)$, $Q_\truncSet\coloneqq Q(\cdot\mid\truncSet)$ and $\chiVar_\truncSet\coloneqq\chisq{P_\truncSet}{Q_\truncSet}$.
\end{lemma}
\begin{proof}
Let $\Pbar{t,\truncSet}\coloneqq\frac1t\sum_{j=1}^{t}Q_\truncSet^{j-1}\times P_\truncSet\times Q_\truncSet^{t-j}$ be the distribution $\Pbar{t}$ with $(P_\truncSet,Q_\truncSet)$ in place of $(P,Q)$. Since $P=P(\truncSet)\,P_\truncSet$ and $Q=P(\truncSet)\,Q_\truncSet$ on $\truncSet$, the restrictions of $\Pbar{t}$ and $Q^t$ to $\truncSet^t$ are $P(\truncSet)^t\,\Pbar{t,\truncSet}$ and $P(\truncSet)^t\,Q_\truncSet^t$, and the complement of $\truncSet^t$ has probability $1-P(\truncSet)^t\le t(1-P(\truncSet))\le\tailShift/2$ under both distributions. Thus on $\truncSet^t$ the privacy loss of $(\Pbar{t},Q^t)$ is equal to that of $(\Pbar{t,\truncSet},Q_\truncSet^t)$, and given $\truncSet^t$ a sample from $\Pbar{t}$ is a sample from $\Pbar{t,\truncSet}$. Therefore we can couple the two privacy losses so that they are equal except with probability at most $\tailShift/2$, and the same holds for the privacy losses of $(Q^t,\Pbar{t})$ and $(Q_\truncSet^t,\Pbar{t,\truncSet})$. By \cref{thm:ra-gauss} for $(P_\truncSet,Q_\truncSet)$, whose likelihood ratio is equal to $\ratio$ on $\truncSet$, and $\tailShift/2$ in place of $\tailShift$, the privacy losses of $(\Pbar{t,\truncSet},Q_\truncSet^t)$ and $(Q_\truncSet^t,\Pbar{t,\truncSet})$ are $((C\rangeBnd/t)\ln(2C/\tailShift),\tailShift/2)$-approximately $(\chiVar_\truncSet/(2t),\chiVar_\truncSet/t)$-Gaussian-concentrated. By the coupling and the transitivity noted after \cref{def:gauss-conc}, the same holds for the privacy losses of $(\Pbar{t},Q^t)$ and $(Q^t,\Pbar{t})$ with $\tailShift$ in place of $\tailShift/2$, as claimed.
\end{proof}

\begin{proofof}{Proof of \cref{cor:gauss-ra} (Gaussian random allocation)}
Let $\truncSet\coloneqq\{\omega:|\lossFunc{\omega}{P}{Q}|\le\ln(1+\rangeBnd)\}$ for $\rangeBnd$ of \cref{eq:gauss-cutoff}, and let $c_2$ and $C_2\ge1$ be the constants of \cref{thm:ra-gauss}. Let $w\coloneqq\sqrt{2\ln(2t/\tailShift)}$, so that $\ln(1+\rangeBnd)=1/(2\sigma^2)+w/\sigma$. The privacy loss $\lossFunc{\omega}{P}{Q}$ is $1/(2\sigma^2)+Z/\sigma$ for $\omega\sim P$ and $-1/(2\sigma^2)+Z/\sigma$ for $\omega\sim Q$, where $Z$ is a standard normal random variable. Thus $\truncSet$ is the event $\{-w-1/\sigma\le Z\le w\}$ under $P$ and the event $\{-w\le Z\le w+1/\sigma\}$ under $Q$. Therefore $P(\truncSet^c)=Q(\truncSet^c)=\Phibar(w)+\Phibar(w+1/\sigma)\le e^{-w^2/2}=\tailShift/(2t)$. As in the proof of \cref{cor:subsampled-approx}, $|\ratio-1|\le\rangeBnd$ on $\truncSet$. Thus, by \cref{lem:ra-conditioning,lem:gc-profile}, for $c=c_2$ and every $C\ge 2C_2$ the pairs $(\Pbar{t},Q^t)$ and $(Q^t,\Pbar{t})$ are $(\lossShift,\tailShift)$-shifted $\sqrt{\chiVar_\truncSet/t}$-GDP. Since $P(\truncSet)=Q(\truncSet)$, the likelihood ratio of $(P_\truncSet,Q_\truncSet)$ is equal to $\ratio$ on $\truncSet$. Therefore we have that $1+\chiVar_\truncSet=\expect{\omega\sim Q}{\ratioFunc{\omega}{P}{Q}^2\,\ind\{\omega\in\truncSet\}}/Q(\truncSet)$.
Let $q$ and $n_2$ be the densities of $Q=\Nrm(0,\sigma^2)$ and of $\Nrm(2,\sigma^2)$. As $\ratioFunc{\omega}{P}{Q}=e^{(\omega-1/2)/\sigma^2}$, completing the square gives $\ratioFunc{\omega}{P}{Q}^2q(\omega)=e^{1/\sigma^2}n_2(\omega)$. Thus $1+\chiVar_\truncSet$ is $e^{1/\sigma^2}$ times the ratio of the probabilities of $\truncSet$ under $\Nrm(2,\sigma^2)$ and under $\Nrm(0,\sigma^2)$. Here $\truncSet=\{\omega:|\omega-1/2|\le\sigma^2\ln(1+\rangeBnd)\}$ is an interval centered at $1/2$. A normal distribution with a fixed variance assigns it less probability when its mean is farther from $1/2$. Therefore this ratio is at most $1$, and $\chiVar_\truncSet\le e^{1/\sigma^2}-1=\chiVar$. As $g_V(\eps)$ is nondecreasing in $V$, the two pairs are $(\lossShift,\tailShift)$-shifted $\sqrt{\chiVar/t}$-GDP. By \cref{clm:gauss-dom} and \citep[Theorem~3.1]{FS25ra}, which holds for every $\eps\in\reals$, random allocation of the Gaussian mechanism is $(\eps,g_{\chiVar/t}(\eps-\lossShift)+\tailShift)$-DP for every $\eps\in\reals$. That is, it is $(\lossShift,\tailShift)$-shifted $\sqrt{\chiVar/t}$-GDP, which gives the claim.
\end{proofof}

\begin{proofof}{Proof of \cref{cor:gauss-ra-closed-form} (Closed-form bound for Gaussian random allocation)}
Let $c_1$ and $C_1$ be the constants of \cref{cor:gauss-ra}, let $\tilde C_1\coloneqq\max\{C_1,1\}$, $\tailShift\coloneqq\delta/2$, $H\coloneqq\ln(t/\delta)\ge\ln(4)$ and $K\coloneqq1+\log_2(2\tilde C_1)$, and let $C\ge\max\{3,\,\sqrt{K/c_1},\,\allowbreak 1+2\sqrt{e-1}+\tilde C_1K\}$. As $\ln(1+y)\le y$, \cref{eq:gauss-conditions} implies that $\sqrt{t/(C\ln(1/\delta))}\ge C\sqrt H$, that is, $t\ge C^3H\ln(1/\delta)$. Further, as in the proof of \cref{clm:gauss-closed-form} (here $\ln(2t/\tailShift)=\ln(4t/\delta)$ and $C\ge5/2$), condition \cref{eq:gauss-conditions} gives that $\rangeBnd\le\sqrt t/\left(\sigma\sqrt{C\ln(1/\delta)}\right)$.
For $\delta\le1/2$ we have that $\ln(C_1/\tailShift)\le\ln(2\tilde C_1)+\ln(1/\delta)\le K\ln(1/\delta)$, and since $\ln(1/\delta)\le t/C^3$ and $\sigma\ge1$, we have that $(\rangeBnd/t)\ln(C_1/\tailShift)\le K\sqrt{\ln(1/\delta)}/\left(\sigma\sqrt{Ct}\right)\le K/C^2\le c_1$.
Therefore, by \cref{cor:gauss-ra}, $1$-out-of-$t$ random allocation of the Gaussian mechanism is $(\lossShift_1,\delta/2)$-shifted $\sqrt{\chiVar/t}$-GDP, where $\lossShift_1\coloneqq C_1(\rangeBnd/t)\ln(C_1/\tailShift)\le\tilde C_1K\sqrt{\ln(1/\delta)}/\left(\sigma\sqrt t\right)$.
By \cref{def:shifted-gdp}, it is $(\eps_{\mathrm G}(\chiVar/t,\delta/2)+\lossShift_1,\delta)$-DP. Finally, $\chiVar/t\le(e-1)/(\sigma^2t)$ by \cref{clm:gauss-pair}. Since $g_V$ is nondecreasing in $V$ and $\eps_{\mathrm G}(V,\delta)\le V/2+\sqrt{2V\ln(1/\delta)}$, we have that $\eps_{\mathrm G}(\chiVar/t,\delta/2)\le\eps_{\mathrm G}((e-1)/(\sigma^2t),\delta/2)\le(e-1)/(2\sigma^2t)+\sqrt{2(e-1)\ln(2/\delta)/(\sigma^2t)}$. As $\ln(2/\delta)\le2\ln(1/\delta)$, this implies that $\eps_{\mathrm G}(\chiVar/t,\delta/2)\le(e-1)/(2\sigma^2t)+2\sqrt{(e-1)\ln(1/\delta)}/\left(\sigma\sqrt t\right)$.
As $\sigma\ge1$, $t\ge2$ and $\delta\le1/2$, we have that $\sigma\sqrt{t\ln(1/\delta)}\ge\sqrt{2\ln2}>(e-1)/2$. Therefore $(e-1)/(2\sigma^2t)\le\sqrt{\ln(1/\delta)}/\left(\sigma\sqrt t\right)$. Thus $\eps_{\mathrm G}(\chiVar/t,\delta/2)+\lossShift_1\le\eps$ by the choice of $C$, which gives the claim.
\end{proofof}

The following corollary is the general version of \cref{cor:gauss-leading-ra}, in the same way as \cref{cor:gauss-leading-gen} is the general version of \cref{cor:gauss-leading}.
\begin{corollary}[Asymptotic upper bounds for random allocation, general version]\label{cor:gauss-leading-ra-gen}
Let $\sigma=\sigma(t)>0$ be bounded away from zero and let $\chiVar\coloneqq e^{1/\sigma^2}-1$.
If $\delta=t^{-\theta}$ for a fixed $\theta>1/2$, then, as $t\to\infty$, $1$-out-of-$t$ random allocation of the Gaussian mechanism \cref{eq:gauss-mech} is $(\eps,\delta)$-DP for ${\eps=(1+o(1))\sqrt{(2\theta-1)\,\chiVar\ln(t)/t}}$.
If $\delta=\delta(t)$ satisfies $\ln(1/\delta)=\omega(\ln(t))$ and
\[
 \ln(1/\delta)\exp\left(\frac{\sqrt{2\ln(4t/\delta)}}{\sigma}\right)=o\left(\sqrt t\right),
\]
then it is $(\eps,\delta)$-DP for ${\eps=(1+o(1))\sqrt{2\chiVar\ln(1/\delta)/t}}$.
In particular, this bound holds if $\ln(1/\delta)=\omega(\ln(t))$ and one of the following two conditions holds:
\begin{enumerate}
 \item $\sigma$ does not depend on $t$ and $\limsup_{t\to\infty}8\ln(1/\delta)/(\sigma\ln(t))^2<1$;
 \item $\ln(1/\delta)=o\left(\sqrt t\right)$ and $\sigma=\Omega\left(\sqrt{\ln(1/\delta)}\right)$.
\end{enumerate}
\end{corollary}

\begin{proof}
The proof of \cref{cor:gauss-leading-gen} uses only that the parameter of the shifted GDP guarantee is at most $\chiVar/(t-1)$ and that its shift is $o\left(\sqrt{\chiVar K/t}\right)$ for $K$ as in that proof.
For random allocation, \cref{cor:gauss-ra} gives the parameter $\chiVar/t$, but it applies only under a condition that has no counterpart for Poisson subsampling.
Let $c$ and $C$ be the constants of \cref{cor:gauss-ra}, let $\tailShift\coloneqq\delta/2$, and let $\rangeBnd$ be given by \cref{eq:gauss-cutoff}, as in the proof of \cref{cor:gauss-leading-gen}.
Then both the condition $(\rangeBnd/t)\ln(2C/\delta)\le c$ of \cref{cor:gauss-ra} and its shift $\lossShift\coloneqq(C\rangeBnd/t)\ln(2C/\delta)$ are given by the quantity $(\rangeBnd/t)\ln(2C/\delta)$.
By the proof of \cref{cor:gauss-leading-gen} with $C'=2C$, under the assumptions of each claim this quantity is $o\left(\sqrt{\chiVar K/t}\right)$ and tends to zero.
Thus the assumptions play the same role as for Poisson subsampling. They make the shift of lower order than the claimed bound, and then the condition of \cref{cor:gauss-ra} holds for all sufficiently large $t$, since it only requires the same quantity to be small.
For such $t$, \cref{cor:gauss-ra} gives that random allocation of the Gaussian mechanism is $(\lossShift,\delta/2)$-shifted $\sqrt{\chiVar/t}$-GDP, and therefore it is $(\eps_{\mathrm G}(\chiVar/t,\delta/2)+\lossShift,\delta)$-DP.
By the proof of \cref{cor:gauss-leading-gen} with $U=\chiVar/t$, the Gaussian term $\eps_{\mathrm G}(\chiVar/t,\delta/2)$ is at most $(1+o(1))\sqrt{\chiVar K/t}$, which together with $\lossShift=o\left(\sqrt{\chiVar K/t}\right)$ gives the first two claims.
For the last claim, the proof of \cref{cor:gauss-leading-gen} shows that each of the two conditions implies \cref{eq:growing-sigma-cond}.
For a fixed $\sigma$, the first claim and the first condition give \cref{cor:gauss-leading-ra}.
\end{proof}
\section{Proofs for \texorpdfstring{\cref{sec:ctd}}{Section 6}}\label{app:ctd-proofs}
In this appendix, $(P,Q)$, $\lossRV$, $\Lambda$, $\tilde \lossRV$ and $(\tilde P,\tilde Q)$ are as in \cref{def:ctd}. If $\lossRV\notin\Lambda$, we denote by $\loss^-<\loss^+$ the two knots adjacent to $\lossRV$. The privacy losses of $(P,Q)$ and $(\tilde P,\tilde Q)$ are distributed as $\lossRV$ and $\tilde \lossRV$, and therefore, by \cref{def:hs}, for every $\kappa\ge0$,
\[
 \HockeyStickDiv{\kappa}{P}{Q}=\expect{}{[1-\kappa e^{-\lossRV}]_+}
 \qquad\text{and}\qquad
 \HockeyStickDiv{\kappa}{\tilde P}{\tilde Q}=\expect{}{[1-\kappa e^{-\tilde \lossRV}]_+}.
\]

\begin{claim}[CtD discretization of randomized response]\label{clm:ctd-rr}
Let $0<\eps_0<\gridSpacing$, let $(P,Q)$ be randomized response over $\{0,1\}$ with $P(1)=Q(0)=e^{\eps_0}/(1+e^{\eps_0})$, and let $(\tilde P,\tilde Q)$ be its CtD pair on the grid $\gridSpacing\mathbb{Z}$. Then $P$ and $Q$ are $(\eps_0,0)$-close, but for every $0\le\delta<\tanh(\eps_0/2)$ and $\eps\ge0$, the distributions $\tilde P$ and $\tilde Q$ are $(\eps,\delta)$-close exactly when $\eps\ge \gridSpacing+\ln(1-\delta/\tilde P(\gridSpacing))$, where $\tilde P(\gridSpacing)=\tanh(\eps_0/2)/(1-e^{-\gridSpacing})$. This threshold tends to $\gridSpacing$ as $\delta\to0$.
\end{claim}
\begin{proof}
The privacy loss of $(P,Q)$ is $\pm\eps_0$, and \cref{eq:ctd-kernel} splits the mass at $\pm\eps_0$ between $\pm \gridSpacing$ and $0$, with $\tilde P(\gridSpacing)=\tanh(\eps_0/2)/(1-e^{-\gridSpacing})$ and $\tilde P(-\gridSpacing)=e^{-\gridSpacing}\tilde P(\gridSpacing)$. Both pairs are symmetric ($Q(\omega)=P(1-\omega)$ and $\tilde Q(\loss)=\tilde P(-\loss)$), and for every $\eps\ge0$ their privacy profiles in both directions are
\[
 \HockeyStickDiv{e^{\eps}}{P}{Q}=\frac{[e^{\eps_0}-e^{\eps}]_+}{1+e^{\eps_0}}
 \qquad\text{and}\qquad
 \HockeyStickDiv{e^{\eps}}{\tilde P}{\tilde Q}=\tilde P(\gridSpacing)\,[1-e^{\eps-\gridSpacing}]_+ .
\]
The first is zero for $\eps\ge\eps_0$. Since $\delta<\tanh(\eps_0/2)=\tilde P(\gridSpacing)(1-e^{-\gridSpacing})$, the second is at most $\delta$ exactly when $e^{\eps-\gridSpacing}\ge1-\delta/\tilde P(\gridSpacing)$.
\end{proof}

\begin{proofof}{Proof of \cref{lem:ctd-facts} (CtD discretization of one step)}
We bound the correction $\zeta\coloneqq\tilde \lossRV-\lossRV$, which is $0$ when $\lossRV\in\Lambda$. Let $\lossRV=\loss\notin\Lambda$, $h_0\coloneqq \loss^+-\loss^-\le \gridSpacing$ and $p\coloneqq\prob{}{\tilde \lossRV=\loss^+\mid \lossRV=\loss}$. Given $\lossRV=\loss$, the correction is $\loss^--\loss\in(-h_0,0)$ with probability $1-p$ and $\loss^--\loss+h_0$ otherwise, and thus $|\zeta|\le h_0\le \gridSpacing$. Since the discretization preserves the conditional mean of $e^{-\lossRV}$, we have $\expect{}{e^{-\zeta}\mid \lossRV=\loss}=1$, which means that $\loss^--\loss=\ln(1-p+pe^{-h_0})$. Therefore, given $\lossRV=\loss$, the correction is $\zeta=\ln\left(1-p+pe^{-h_0}\right)+h_0\,\ind\{\tilde \lossRV=\loss^+\}$, where the indicator has mean $p$ given $\lossRV=\loss$. Thus $\expect{}{\zeta\mid \lossRV=\loss}\ge0$, since $e^{-\expect{}{\zeta\mid \lossRV=\loss}}\le\expect{}{e^{-\zeta}\mid \lossRV=\loss}=1$ by Jensen's inequality. At the same time, by Hoeffding's lemma \citep{Hoeffding63} for a random variable in an interval of length $h_0$, we have $0=\ln\left(\expect{}{e^{-\zeta}\mid \lossRV=\loss}\right)\le-\expect{}{\zeta\mid \lossRV=\loss}+h_0^2/8$, and thus $\expect{}{\zeta\mid \lossRV=\loss}\le h_0^2/8\le \gridSpacing^2/8$. Taking expectations, we get the bounds on $\lossMean{\tilde P}{\tilde Q}-\lossMean{P}{Q}=\expect{}{\zeta}$.

To bound $\Var(\zeta)$ we let $c_h\coloneqq\ln(\cosh(\gridSpacing/2))$ and show that $\expect{}{(\zeta-c_h)^2\mid \lossRV}\le \gridSpacing^2/4$ almost surely, which suffices since $\Var(\zeta)\le\expect{}{(\zeta-c_h)^2}$. When $\lossRV\in\Lambda$, the left-hand side is $c_h^2\le \gridSpacing^2/4$, as $0\le\ln(\cosh(\gridSpacing/2))\le\gridSpacing/2$. When $\lossRV=\loss\notin\Lambda$, by this expression for $\zeta$, the conditional variance of $\zeta$ is $h_0^2p(1-p)$ and its conditional mean is $m_{h_0}(p)\coloneqq h_0p+\ln(1-p+pe^{-h_0})$. The function $m_{h_0}$ is concave on $[0,1]$, with $m_{h_0}(0)=m_{h_0}(1)=0$, $m_{h_0}(1/2)=c_{h_0}\coloneqq\ln(\cosh(h_0/2))\in[0,h_0/2]$ and $m_{h_0}'(1/2)=h_0-2\tanh(h_0/2)\in[0,h_0]$. By concavity, the chord of $m_{h_0}$ between $p$ and $1/2$ has slope between $m_{h_0}'(1/2)$ and the slope $\pm2c_{h_0}$ of the chord from $1/2$ to the endpoint of $[0,1]$ on the side of $p$. Both slopes lie in $[-h_0,h_0]$, and thus $|m_{h_0}(p)-c_{h_0}|\le h_0|p-1/2|$. Therefore
\[
 \expect{}{(\zeta-c_{h_0})^2\mid \lossRV=\loss}=h_0^2p(1-p)+\left(m_{h_0}(p)-c_{h_0}\right)^2\le h_0^2p(1-p)+h_0^2\left(p-\frac12\right)^2=\frac{h_0^2}{4}.
\]
Now $c_h\ge c_{h_0}$ and $m_{h_0}(p)\ge0$, and therefore
\[
 \expect{}{(\zeta-c_h)^2\mid \lossRV=\loss}=\expect{}{(\zeta-c_{h_0})^2\mid \lossRV=\loss}-2(c_h-c_{h_0})\,m_{h_0}(p)+c_h^2-c_{h_0}^2\le\frac{h_0^2}{4}+c_h^2-c_{h_0}^2\le\frac{\gridSpacing^2}{4},
\]
where the last inequality uses that
\[
 c_h^2-c_{h_0}^2=\int_{h_0}^{\gridSpacing}\ln(\cosh(s/2))\tanh(s/2)\,ds\le\int_{h_0}^{\gridSpacing}\frac s2\,ds=\frac{\gridSpacing^2-h_0^2}{4}.
\]

For the covariance, let $\bar\zeta\coloneqq\expect{}{\zeta\mid \lossRV}$. As $\zeta$ is bounded and $\lossVar{P}{Q}<\infty$, we have that $\operatorname{Cov}(\lossRV,\zeta)=\expect{}{(\lossRV-\lossMean{P}{Q})\,\zeta}=\expect{}{(\lossRV-\lossMean{P}{Q})\,\bar\zeta}=\operatorname{Cov}(\lossRV,\bar\zeta)$.
We showed that $\bar\zeta\in[0,\gridSpacing^2/8]$, and therefore $\Var(\bar\zeta)\le(\gridSpacing^2/16)^2$ by Popoviciu's inequality. By the Cauchy--Schwarz inequality, $|\operatorname{Cov}(\lossRV,\zeta)|\le(\gridSpacing^2/16)\sqrt{\lossVar{P}{Q}}$, and the bounds on $\lossVar{\tilde P}{\tilde Q}-\lossVar{P}{Q}$ follow from $\lossVar{\tilde P}{\tilde Q}=\Var(\tilde \lossRV)=\lossVar{P}{Q}+\Var(\zeta)+2\operatorname{Cov}(\lossRV,\zeta)$ and $0\le\Var(\zeta)\le\gridSpacing^2/4$.
\end{proofof}

\begin{proofof}{Proof of \cref{thm:ctd-composition} (Composition of CtD discretizations)}
Since $\expect{}{e^{-\tilde \lossRV}\mid \lossRV}=e^{-\lossRV}$, the conditional Jensen inequality for the convex function $[1-\kappa\,\cdot\,]_+$ gives that $\expect{}{[1-\kappa e^{-\tilde \lossRV}]_+\mid \lossRV}\ge[1-\kappa e^{-\lossRV}]_+$. Taking expectations, we get that $(\tilde P,\tilde Q)$ dominates $(P,Q)$, and the domination of the products follows by applying the mappings of \cref{lem:kov} coordinatewise.

The pair $(\tilde P,\tilde Q)$ is mutually absolutely continuous, and its privacy loss is distributed as $\tilde \lossRV=\lossRV+(\tilde\lossRV-\lossRV)$. By \cref{def:ctd}, $|\tilde \lossRV-\lossRV|\le\gridSpacing$. Thus the correction $\tilde\lossRV-\lossRV$ deviates from its mean by at most $2\gridSpacing$, and therefore $|\tilde\lossRV-\lossMean{\tilde P}{\tilde Q}|\le|\lossRV-\lossMean{P}{Q}|+2\gridSpacing\le\rangeBnd+2\gridSpacing$. Thus \cref{thm:gauss-approx} with $\truncSet$ the whole space and $\rangeBnd+2\gridSpacing$ in place of $\rangeBnd$ gives the Gaussian concentration of the privacy loss of $(\tilde P^t,\tilde Q^t)$.
\end{proofof}

\begin{proofof}{Proof of \cref{cor:ctd-error} (Error of the discretization)}
The first inequality follows from the domination of \cref{thm:ctd-composition}. We denote the privacy losses of the original and of the discretized composition by $\lossRV$ and $\tilde \lossRV$, their means by $\mu=t\cdot\lossMean{P}{Q}$ and $\tilde\mu=t\cdot\lossMean{\tilde P}{\tilde Q}$ and their variances by $V=t\cdot\lossVar{P}{Q}$ and $\hat V=t\cdot\lossVar{\tilde P}{\tilde Q}$. Let $G\sim\Nrm(\mu,V)$ and $\tilde G\sim\Nrm(\tilde\mu,\hat V)$. By \cref{thm:gauss-approx} with $\truncSet$ the whole space and by \cref{thm:ctd-composition}, $G\preceq_{(C\rangeBnd\ln(2/\tailShift),\tailShift)}\lossRV$ and $\tilde \lossRV\preceq_{(C(\rangeBnd+2\gridSpacing)\ln(2/\tailShift),\tailShift)}\tilde G$. By \cref{eq:param-change}, $\tilde G\preceq_{(\lossShift_1,\tailShift)}G$ for $\lossShift_1\coloneqq\tilde\mu-\mu+|\sqrt{\hat V}-\sqrt V|\sqrt{2\ln(1/\tailShift)}$. By \cref{lem:ctd-facts}, $0\le\tilde\mu-\mu\le t\gridSpacing^2/8$ and $-(t\gridSpacing^2/8)\sqrt{\lossVar{P}{Q}}\le\hat V-V\le t\gridSpacing^2/4+(t\gridSpacing^2/8)\sqrt{\lossVar{P}{Q}}$. If $\hat V\ge V$, then $\sqrt{\hat V}-\sqrt V\le(\hat V-V)/\left(2\sqrt V\right)\le t\gridSpacing^2/\left(8\sqrt V\right)+\gridSpacing^2\sqrt t/16=\gridSpacing^2\sqrt t/\left(8\sqrt{\lossVar{P}{Q}}\right)+\gridSpacing^2\sqrt t/16$, and otherwise $\sqrt V-\sqrt{\hat V}\le(V-\hat V)/\sqrt V\le \gridSpacing^2\sqrt t/8$. Chaining the three relations and applying \cref{clm:stoch-hs}, we get that $\HockeyStick{e^{\eps}}(\tilde \lossRV)\le\HockeyStick{e^{\eps-\lossShift}}(\lossRV)+3\tailShift$ for every $\eps\in\reals$, where
\[
 \lossShift\coloneqq 2C(\rangeBnd+\gridSpacing)\ln(2/\tailShift)+\frac{t\gridSpacing^2}8+\frac{\gridSpacing^2\sqrt t}8\left(\frac1{\sqrt{\lossVar{P}{Q}}}+1\right)\sqrt{2\ln(1/\tailShift)} .
\]
At $\eps=\eps(\delta-3\tailShift)+\lossShift$ the right-hand side is at most $\delta$, and therefore $\tilde\eps(\delta)\le\eps(\delta-3\tailShift)+\lossShift$.
\end{proofof}

\iflipics\else
\section{Subexponential tails}\label{app:cutoff}
In this appendix we consider privacy losses whose deviations from their means have exponentially decaying tails. We show that for such losses we can replace the mean $\truncMean{P}{Q}$ and the variance $\truncVar{P}{Q}$ of the truncated privacy loss in \cref{thm:gauss-approx} by the mean $\lossMean{P}{Q}$ and the variance $\lossVar{P}{Q}$ of the privacy loss itself, which do not depend on the truncation set $\truncSet$. We use the notation of \cref{sec:gaussapprox}.
\begin{theorem}[Gaussian concentration under subexponential tails]\label{thm:gauss-cutoff}
There is a universal constant $C>0$ such that the following holds. Let $t\in\naturals$, $\rangeBnd>0$, $0<\tailShift\le1$ and $0\le\tailShift_0\le\tailShift/(4t)$, and assume that $\lossRVfunc{P}{Q}$ has a finite mean and that $\prob{}{|\lossRVfunc{P}{Q}-\lossMean{P}{Q}|>\loss}\le\tailShift_0e^{1-\loss/\rangeBnd}$ for every $\loss\ge\rangeBnd$.
Then $\lossRVfunc{P^t}{Q^t}$ is $(C\rangeBnd\ln(2/\tailShift),\tailShift)$-approximately $(t\cdot\lossMean{P}{Q},t\cdot\lossVar{P}{Q})$-Gaussian-concentrated.
\end{theorem}
If $\tailShift_0=0$, then the hypothesis says that $|\lossRVfunc{P}{Q}-\lossMean{P}{Q}|\le\rangeBnd$ almost surely, and the conclusion also follows from \cref{thm:gauss-approx} with $\truncSet$ the whole space. By \cref{lem:gc-profile}, we also get that $\HockeyStickDiv{e^{\eps}}{P^t}{Q^t}\le g_{t\cdot\lossVar{P}{Q}}(\eps-\lossShift)+\tailShift$ for every $\eps$, where $\lossShift\coloneqq t\cdot(\lossMean{P}{Q}-\lossVar{P}{Q}/2)+C\rangeBnd\ln(2/\tailShift)$. Compared with \cref{thm:gauss-approx}, the theorem assumes a tail bound at every scale beyond $\rangeBnd$ rather than at a single cutoff. Such bounds hold for subexponential privacy losses. Specifically, if the random variable $e^{|\lossRVfunc{P}{Q}-\lossMean{P}{Q}|/\rangeBnd}$ has expectation at most $2$, then by Markov's inequality the tail assumption holds with $\rangeBnd\ln(8t/\tailShift)$ in place of $\rangeBnd$ and with $\tailShift_0\coloneqq\tailShift/(4t)$. The cutoff grows only logarithmically in $t/\tailShift$.

\begin{proof}
We clip the centered privacy losses at $\rangeBnd$ and couple the sum of the clipped losses with a Gaussian random variable by \cref{lem:coupling}. Clipping changes the sum only with probability at most $t\tailShift_0$. It changes the mean and the standard deviation of the sum by $O(\rangeBnd)$, and therefore rescaling the Gaussian random variable to the mean and the variance of the privacy loss itself costs a loss shift of only $O(\rangeBnd)$.

Let $\lossRV_1,\ldots,\lossRV_t$ be independent copies of $\lossRVfunc{P}{Q}$, let $\mu\coloneqq\lossMean{P}{Q}$ and $V\coloneqq t\cdot\lossVar{P}{Q}$, and let $X_i\coloneqq\lossRV_i-\mu$. Then $\sum_i\lossRV_i$ is distributed as $\lossRVfunc{P^t}{Q^t}$, since the privacy loss of a product pair is the sum of the privacy losses of its coordinates, and $S\coloneqq\sum_iX_i$ has mean $0$ and variance $V$. We clip the centered privacy losses at $\rangeBnd$, that is, we let $Y_i\coloneqq\max\{-\rangeBnd,\min\{X_i,\rangeBnd\}\}$, $R_i\coloneqq X_i-Y_i$ and $T\coloneqq\sum_iY_i$. By the tail assumption with $\loss=\rangeBnd$ and a union bound, $S=T$ except with probability at most $t\tailShift_0\le\tailShift/4$.

\para{The parameters of the clipped sum.} The clipped parts satisfy $|R_i|=(|X_i|-\rangeBnd)_+$, and therefore integrating the tail assumption gives that
\[
 \expect{}{|R_i|}=\int_0^\infty\prob{}{|X_i|>\rangeBnd+s}\,ds\le\tailShift_0\rangeBnd
 \quad\text{and}\quad
 \expect{}{R_i^2}=\int_0^\infty2s\,\prob{}{|X_i|>\rangeBnd+s}\,ds\le2\tailShift_0\rangeBnd^2 .
\]
Let $m\coloneqq\expect{}{T}$ and $V_T\coloneqq\Var(T)$. The first bound gives that $|m|\le t\tailShift_0\rangeBnd\le\tailShift\rangeBnd/4$, since $\expect{}{S}=0$ and therefore $m=-\sum_i\expect{}{R_i}$. By the triangle inequality for the standard deviation and the independence of the $R_i$, the second bound gives that $\bigl|\sqrt V-\sqrt{V_T}\bigr|\le\sqrt{\Var\bigl(\sum_{i=1}^tR_i\bigr)}\le\rangeBnd\sqrt{2t\tailShift_0}\le\rangeBnd\sqrt{\tailShift/2}$.

\para{The coupling.} Each $Y_i$ deviates from its mean by at most $\frac54\rangeBnd$, since $|Y_i|\le\rangeBnd$ and $|\expect{}{Y_i}|=|\expect{}{R_i}|\le\tailShift_0\rangeBnd\le\rangeBnd/4$. Let $C_1$ be the constant of \cref{lem:coupling}. By \cref{lem:coupling} for $T-m$ with $\frac54\rangeBnd$ in place of $\rangeBnd$, we can couple $T$ with a standard normal random variable $Z$ so that $|T-m-\sqrt{V_T}Z|\le\frac54C_1\rangeBnd\ln(8/\tailShift)$ except with probability $\tailShift/4$ (given $\lossRV_1,\ldots,\lossRV_t$, we sample $Z$ from its conditional distribution given $T$ in this coupling). We let $G\coloneqq t\cdot\mu+\sqrt VZ$, so that $G\sim\Nrm(t\cdot\mu,V)$. On the event $S=T$ we have that
\[
 \Big|\sum_{i=1}^t\lossRV_i-G\Big|=\left|T-\sqrt VZ\right|\le\left|T-m-\sqrt{V_T}Z\right|+|m|+\left|\sqrt{V_T}-\sqrt V\right||Z| .
\]
By the Gaussian tail bound, $|Z|\le\sqrt{2\ln(4/\tailShift)}$ except with probability $\tailShift/4$. Combining these bounds, we get that, except with probability at most $3\tailShift/4$,
\[
 \Big|\sum_{i=1}^t\lossRV_i-G\Big|\le\frac54C_1\rangeBnd\ln(8/\tailShift)+\frac{\tailShift\rangeBnd}4+\rangeBnd\sqrt{\tailShift\ln(4/\tailShift)}\le C\rangeBnd\ln(2/\tailShift)
\]
for $C\coloneqq\frac{15}4C_1+3$, where we used that $\ln(8/\tailShift)\le3\ln(2/\tailShift)$ and $\tailShift/4+\sqrt{\tailShift\ln(4/\tailShift)}\le2\le3\ln(2/\tailShift)$ for $\tailShift\le1$. As noted after \cref{def:gauss-conc}, this coupling gives the claim.
\end{proof}
\fi

\newpage
\iflipics\lipicsbibliography{bibliography}\else\printbibliography\fi
\end{document}